\documentclass{article}
\usepackage{titletoc}

\usepackage{amsmath,amsfonts,bm}

\def\eqref#1{equation~\ref{#1}}

\def\1{\bm{1}}

\def\mA{{\bm{A}}}

\def\mI{{\bm{I}}}

\DeclareMathAlphabet{\mathsfit}{\encodingdefault}{\sfdefault}{m}{sl}
\SetMathAlphabet{\mathsfit}{bold}{\encodingdefault}{\sfdefault}{bx}{n}

\newcommand{\E}{\mathbb{E}}

\newcommand{\R}{\mathbb{R}}

\usepackage{microtype}
\usepackage{graphicx}
\usepackage{booktabs} %
\usepackage{blindtext}
\usepackage{titlesec}
\usepackage{dsfont}

\usepackage{hyperref}
\usepackage{algorithm2e}
\usepackage{multirow}

\usepackage{nicematrix}

\usepackage[accepted]{icml2024}

\usepackage{amsmath}
\usepackage{amssymb}
\usepackage{mathtools}
\usepackage{amsthm}

\usepackage{xfrac}

\usepackage{xr}
\usepackage{caption}
\usepackage{subcaption}
\usepackage{float}
\usepackage{enumitem}

\usepackage[export]{adjustbox}

\usepackage[capitalize,noabbrev]{cleveref}

\usepackage[none]{hyphenat}
\allowdisplaybreaks[2]

\newcommand{\Identity}{{\rm I\kern-.2em l}}
\newcommand{\Expect}{\mathbb{E}}
\newcommand{\Expectbracket}[1]{\mathbb{E}\left[ #1 \right]}

\newcommand{\T}{\mathrm{T}}
\newcommand{\A}{\mathbf{A}}
\newcommand{\bb}{\mathbf{b}}

\newcommand{\B}{\mathbf{B}}

\newcommand{\nn}{\mathbf{n}}
\newcommand{\x}{\mathbf{x}}
\newcommand{\y}{\mathbf{y}}
\newcommand{\z}{\mathbf{z}}
\newcommand{\g}{\mathbf{g}}

\newcommand{\bv}{\mathbf{v}}
\newcommand{\bu}{\mathbf{u}}

\newcommand{\norm}[1]{\left\Vert #1 \right\Vert}
\newcommand{\normsq}[1]{\left\Vert #1 \right\Vert^2}
\newcommand{\innerprod}[1]{\left\langle #1 \right\rangle}

\theoremstyle{plain}
\newtheorem{theorem}{Theorem}[section]
\newtheorem{proposition}[theorem]{Proposition}
\newtheorem{lemma}[theorem]{Lemma}
\newtheorem{corollary}[theorem]{Corollary}
\theoremstyle{definition}

\newtheorem{assumption}[theorem]{Assumption}
\theoremstyle{remark}

\usepackage[textsize=tiny]{todonotes}

\begin{document}

\twocolumn[
\icmltitle{{A New Theoretical Perspective on Data Heterogeneity in Federated Optimization}}

\icmlsetsymbol{equal}{*}

\begin{icmlauthorlist}
\icmlauthor{Jiayi Wang}{yyy}
\icmlauthor{Shiqiang Wang}{comp}
\icmlauthor{Rong-Rong Chen}{yyy}
\icmlauthor{Mingyue Ji}{yyy}
\end{icmlauthorlist}

\icmlaffiliation{yyy}{Department of Electrical and Computer Engineering, University of Utah, Salt Lake City, UT, USA.}
\icmlaffiliation{comp}{IBM T. J. Watson Research Center, Yorktown Heights, NY, USA}

\icmlcorrespondingauthor{Jiayi Wang}{jiayi.wang@utah.edu}
\icmlcorrespondingauthor{Shiqiang Wang}{wangshiq@us.ibm.com}
\icmlcorrespondingauthor{Rong-Rong Chen}{rchen@utah.edu}
\icmlcorrespondingauthor{Mingyue Ji}{mingyue.ji@utah.edu}

\icmlkeywords{Machine Learning, ICML}

\vskip 0.3in
]

\printAffiliationsAndNotice{}  %

\begin{abstract}
In federated learning (FL), data heterogeneity is the main reason that existing theoretical analyses are pessimistic about the convergence rate. In particular, for many FL algorithms, the convergence rate grows dramatically when the number of local updates becomes large, especially when the product of the gradient divergence and local Lipschitz constant is large. However, empirical studies can show that more local updates can improve the convergence rate even when these two parameters are large, which is inconsistent with the theoretical findings. This paper aims to bridge this gap between theoretical understanding and practical performance by providing a theoretical analysis from a new perspective on data heterogeneity. In particular, we propose a new and weaker assumption compared to the local Lipschitz gradient assumption, named the heterogeneity-driven pseudo-Lipschitz assumption. We show that this and the gradient divergence assumptions can jointly characterize the effect of data heterogeneity. By deriving a convergence upper bound for FedAvg and its extensions, we show that, compared to the existing works, local Lipschitz constant is replaced by the much smaller heterogeneity-driven pseudo-Lipschitz constant and the corresponding convergence upper bound can be significantly reduced for the same number of local updates, although its order stays the same. In addition, when the local objective function is quadratic, more insights on the impact of data heterogeneity can be obtained using the heterogeneity-driven pseudo-Lipschitz constant. For example, we can identify a region where FedAvg can outperform mini-batch SGD even when the gradient divergence can be arbitrarily large. Our findings are validated using experiments.

\end{abstract}

\section{Introduction}\label{sec:intro}

Federated learning (FL) has emerged as an important technique for locally training machine learning models over geographically distributed workers. 
It has advantages in improving training efficiency and preserving data privacy. 
In this paper, we consider the following optimization problem in FL:
\begin{align}
\label{eq: objective function}
\min_{\x} \left\{f(\x) := \frac{1}{N}\sum_{i=1}^N F_i(\x) \right\}, 
\end{align}
where $N$ is the number of workers; $F_i(\x)$ is the expected loss function of worker $i$ given by\footnote{The objective function can be extended to weighted average by multiplying each local objective function by a possibly distinct constant.}
\begin{align}
F_i(\x) := \Expect_{\nn_i \sim \mathcal{D}_i} [\ell (\x;\nn_i)],
\end{align}
where $\ell (\cdot)$ is the loss function, $\nn_i$ is the random data sample on worker $i$, and $\mathcal{D}_i$ is the data distribution on worker $i$. In addition, we use $\mathcal{D}$ to denote the global data distribution. 
In FL, each worker performs $I>1$ local iterations using its local dataset to reduce the communication cost, which is called \textit{local updates}. 
Federated averaging (FedAvg), also known as local stochastic gradient descent (local SGD), is one of the most popular algorithms to solve the above optimization problem~\citep{McMahan2017a}. 
In addition to FedAvg, a number of FL algorithms \citep{yu2019linear,karimireddy2020scaffold,reddi2020adaptive,li2020federated,wang2020tackling,wangslowmo} have been proposed, whereas the core mechanism, local updates, is still the foundation of %
FL. 
Nevertheless, existing theoretical analyses are pessimistic on the convergence error caused by local updates. It is unclear if performing a large number of local updates can improve the convergence rate when the gradient divergence is large or data are highly non-IID/heterogeneous. This will be explained in detail as follows.

\textbf{There is a gap between the theoretical understanding and the experimental results.} Unlike the centralized SGD running on a single machine, where the gradients are directly sampled from the global data distribution $\mathcal{D}$, the local gradients in FedAvg are sampled from the local data distributions $\{\mathcal{D}_i\}$, which are often highly heterogeneous \citep{kairouz2021advances}. This can deteriorate FL's performance since the local models could drift to different directions during local updates \citep{zhao2018federated,karimireddy2020scaffold}.
Therefore, a common understanding is that local SGD can have a larger convergence error than that of centralized SGD due to local updates. 
Existing theoretical analyses for non-convex objective functions \citep{yu2019linear, YuAAAI2019,CooperativeSGD,yang2020achieving} confirmed this intuition and showed that the convergence error caused by local updates grows fast when the number of local updates~$I$ becomes large. %
This limits the usefulness of local updates. %
However, in practice, a large number of local updates have been successfully applied \citep{LiT_2019,niknam2020federated,rieke2020future} and showed superior experimental performance compared to mini-batch SGD with each worker performing $I=1$ local iteration per round \citep{McMahan2017a, Lin2020Don't}.
This means that, empirically, a large $I$ can improve the convergence rate %
even when the data are highly non-IID. %
This inconsistency between the pessimistic theoretical results and the good experimental results for the local updates implies that the existing theoretical analyses may overestimate the error caused by local updates. In addition, it is indeed challenging to show theoretically when local SGD ($I>1$) can outperform mini-batch SGD ($I=1$) \citep{pmlr-v119-woodworth20a, woodworth2020minibatch}.

\begin{figure}
\centering
\includegraphics[width=5cm]{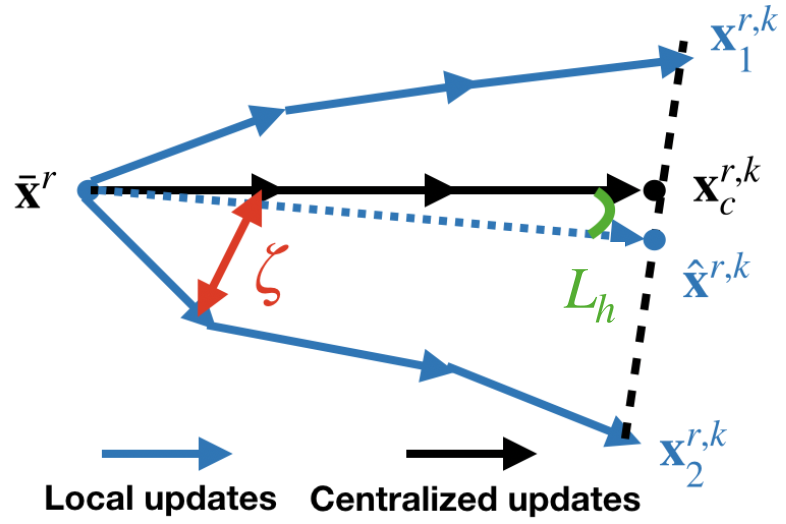} 
\caption{An illustrative comparison between local updates and centralized updates.
$\bar{\x}^r$ is the global model at $r$th round. 
The local models after $k$ local iterations at the $r$th round are denoted by $\x_1^{r,k}$ and $\x_2^{r,k}$. The average of $\x_1^{r,k}$ and $\x_2^{r,k}$ is $\hat{\x}^{r,k}$. The centralized model after $k$ centralized iterations is denoted by $\x_c^{r,k}$. It can be seen that $\zeta$ shows the difference between $\x_c^{r,k}$ and $\x_i^{r,k},i=1,2$ and $L_h$ shows the difference between $\x_c^{r,k}$ and $\hat{\x}^{r,k}$. 
}
\label{fig:D-varphi}
\end{figure}

\textbf{Although local models could drift to different directions, the average of local models can still be close to the centralized model.}
To the best of our knowledge, the only metric of data heterogeneity in existing works \citep{YuAAAI2019,CooperativeSGD,woodworth2020minibatch} is the gradient divergence ($\zeta$), or its more general version, called gradient dissimilarity \citep{karimireddy2020scaffold}, which characterizes the difference between the local gradient $\nabla F_i(\x)$ of worker~$i$ and the global gradient $\nabla f(\x)$. 
As shown in Figure~\ref{fig:D-varphi}, the intuition of the gradient divergence is that when $\zeta$ is large, the difference between local gradients and the global gradient is large. Then after multiple local iterations, the local models will drift to different directions. 
Previous theoretical results based on the gradient divergence show that when $\zeta$ is large, $I$ has to be small to avoid the divergence of the FL algorithms. However, in FL, the final output is the global model on the server, which is the average of local models after local updates. As shown in Figure~\ref{fig:D-varphi}, although $\zeta$ is large, the averaged model $\hat{\x}^{r,k}$ can still be close to the centralized model $\x_c^{r,k}$ that can be obtained if we had used centralized SGD. This means that the convergence error caused by local updates can be close to zero. While $\zeta$ successfully characterizes the variance among local gradients, it cannot capture the difference between the averaged model and the centralized model. Consequently, relying solely on the gradient divergence in convergence analysis can lead to an overestimation of the convergence error caused by local updates. To obtain a better convergence upper bound, \textit{it is necessary to introduce a new metric which can characterize the difference between the averaged model and the centralized model}.

To address the inconsistency between the theory and practice, we introduce a new metric $L_h$, referred to as the \textit{heterogeneity-driven pseudo-Lipschitz constant}. As shown in Figure~\ref{fig:D-varphi}, the proposed metric $L_h$ captures the difference between the averaged model and the centralized model, which cannot be characterized by $\zeta$. In our analysis, we use the heterogeneity-driven pseudo-Lipschitz constant $L_h$ and the global Lipschitz constant $L_g$ to substitute the widely used local Lipschitz constant $\tilde{L}$. This is based on our important observation that $\tilde{L}$ is affected by the data heterogeneity, which has not been pointed out in previous theoretical studies. In the literature \citep{YuAAAI2019,yang2020achieving, khaled2020tighter}, $\tilde{L}$ is used to characterize the smoothness of the gradients for all local objective functions under any degree of data heterogeneity. However, as shown in Table~\ref{tab:estimate-L} (Section~\ref{sec:experiments}), $\tilde{L}$ increases fast as the percentage of non-IID data increases.
We use $L_h$ to characterize the information on data heterogeneity %
and use $L_g$ to characterize the smoothness of the global objective function. %
It %
can be proved that the new assumptions used in this paper are weaker than the local Lipschitz gradient commonly used in the literature.

\textbf{Contribution of this paper.} In this paper, we reveal the fundamental effect of data heterogeneity on FedAvg %
and its extensions 
by introducing a new metric $L_h$, the \textit{heterogeneity-driven pseudo-Lipschitz constant} in Assumption~\ref{assumption:gradient-to-model}. In particular, our main contributions are as follows. 
\begin{enumerate}[noitemsep,topsep=0pt]
    \item Using the new assumptions, %
    which are proved to be weaker than those in the literature, we develop a novel analysis for FedAvg and its extensions, including FedAvg with momentum \citep{yu2019linear} and FedAdam \citep{reddi2020adaptive}, with general non-convex objective functions. We show that for the terms with the number of local updates ($I$), the local Lipschitz constant $\tilde{L}$ is replaced by the newly introduced heterogeneity-driven pseudo-Lipschitz constant $L_h$ and the global Lipschitz constant $L_g$. Since $L_h$ can be significantly smaller than $\tilde{L}$ in practice, a much larger number of local updates ($I$) can be used to achieve a small convergence upper bound even if the gradient divergence $\zeta$ is large. This bridges the gap between theory and practice.
    \item Our analysis can incorporate partial participation where only a subset of workers are sampled to perform local updates in each round. We show that with partial participation, increasing $I$ can still improve the convergence rate when the data are highly heterogeneous. 
    \item We discuss %
    a number of insights seen from the proposed $L_h$ metric. For example we identify a region where local SGD can outperform mini-batch SGD for some quadratic objective functions. %
    \item Our theoretical results are validated using experiments.
\end{enumerate}

\section{Related Work}

FedAvg, also known as local SGD, was first proposed by \citet{McMahan2017a}. Since then, there has been considerable work analyzing the convergence rate of local SGD \citep{NEURIPS2019_c17028c9,YuAAAI2019,kairouz2021advances} and its extensions such as FedAvg with momentum~\citep{yu2019linear}, SCAFFOLD~\citep{karimireddy2020scaffold} and adaptive methods~\citep{reddi2020adaptive}. There is also a line of work focusing on the partial participation~\citep{yang2020achieving}, compression and quantization~\citep{Peng2018, richtarik2021ef21} in local SGD. 
Despite the extensive analysis of local SGD and its extensions, it is hard to show that a large number of local updates can improve the convergence rate when data are highly heterogeneous~\citep{pmlr-v119-woodworth20a,woodworth2020minibatch}, while in practice, more local updates can improve the convergence. 
To address the gap between the theory and practice, there are two papers~\citep{wang2022unreasonable,das2022faster} trying to find new assumptions that can better characterize the effect of data heterogeneity in local SGD.
However, none of these works have noted that the local Lipschitz constant increases with the data heterogeneity. Consequently, they still rely on the local Lipschitz assumption for convergence analysis, whereas our work introduces the heterogeneity-driven pseudo-Lipschitz constant, yielding an improved convergence bound.
A detailed discussion on related work can be found in Appendix~\ref{sec:add-related-works}.

\section{Preliminaries} 

In FedAvg, each round is composed of the local update phase and the global update phase. 
The global model is initialized as $\bar{\x}^0$. 
At the start of round $r$, the server distributes the global model $\bar{\x}^r$ to all workers. During the local update phase, each worker  updates its local model with the local learning rate $\gamma$ and the stochastic gradients sampled from its own local data distribution~$\mathcal{D}_i$, 
\begin{align}
\x_i^{r,k+1} =  \x_i^{r,k} - \gamma \g(\x_i^{r,k};\nn_i),
\end{align}
where $\x_i^{r,k}$ is the local model at the $r$th round and $k$th iteration at worker $i$. 
For simplicity, we use $\g_i(\cdot)$ to denote the stochastic gradient $\g(\cdot;\nn_i)$. In addition, $\bar{\g}(\cdot)$ 
denotes the stochastic gradient sampled from the global dataset $\mathcal{D}$. We assume that the local stochastic gradient is an unbiased estimate of the full local gradient, i.e., $\Expect\big[\g_i(\x_i^{r,k}) \big| \x_i^{r,k}\big] = \nabla F_i(\x_i^{r,k})$. After $I$ local iterations at the $r$th round, worker~$i$ sends the local model update $\Delta_i^r:=\bar{\x}^r - \x_i^{r,I}$ 
to the server. During the global update phase, the server updates the global model
using the following equality:
\begin{align}
\textstyle \bar{\x}^{r+1} = \bar{\x}^r - \eta\cdot \frac{1}{N}\sum_{i=1}^N \Delta_i^r,
\end{align}
where $\eta$ is the global learning rate.

The following assumptions are widely used in the literature for the analysis of algorithms %
including FedAvg 
\citep{karimireddy2020scaffold,YuAAAI2019,khaled2020tighter,wang2020tackling}, FedAvg with momentum~\citep{yu2019linear} and adaptive methods~\citep{reddi2020adaptive}.%
\begin{assumption}[Local Lipschitz Gradient] \label{assumption:local-lipschitz-gradient}
 \begin{align}
\textstyle \norm{\nabla F_i(\x) - \nabla F_i(\y)} \leq \tilde{L}\norm{\x - \y}, \forall \x, \y, i.
 \end{align}
\end{assumption}
There are also some works \citep{khaled2020tighter} assuming that Lipschitz gradient condition holds for each data sample $\norm{\nabla \ell (\x;\xi) - \nabla \ell (\y;\xi)}\le L' \norm{\x-\y},\forall \x,\y \in \R^d,\xi\in\mathcal{D} $. 
Note that this is stronger and can imply local Lipschitz gradient condition.

\begin{assumption}[Bounded Stochastic Gradient Variance]\label{assumption:bounded-stochastic-variance}
\begin{align}
	\textstyle \Expectbracket{\normsq{\g_i(\x) -\nabla F_i(\x) }} \le \sigma^2,\forall i,\x.
\end{align}
\end{assumption}
\begin{assumption}[Bounded Gradient Divergence]\label{assumption:bounded-gradient-divergence}
\begin{align}
	\textstyle \normsq{\nabla F_i(\x) - \nabla f(\x)} \le \zeta^2,\forall i,\x.
\end{align}
\end{assumption}
Assumption~\ref{assumption:bounded-gradient-divergence} is often the only metric of data heterogeneity %
in the literature \citep{yu2019linear, CooperativeSGD}, where it was shown that there is a term $O(\gamma^2\tilde{L}^2I^2\zeta^2)$ in the convergence upper bound. 
This means that the gradient divergence ($\zeta$) and the number of local updates ($I$) are coupled, and the error caused by $\zeta$ grows fast as $I$ increases and the effect of $I^2\zeta^2$ is amplified by $\tilde{L}^2$.
In this paper, we find that this result can be pessimistic since it can be seen from Table~\ref{tab:estimate-L} (in Section~\ref{sec:experiments}) that $\tilde{L}$ can be very large, which means that the error caused by $I^2\zeta^2$ can become much larger due to the large $\tilde{L}^2$. In the next section, we will %
address this problem using Assumption~\ref{assumption:gradient-to-model} in the analysis. 

\section{Main Results}\label{sec:main-results}

In this section, we present the convergence upper bound for non-convex objective functions using the proposed new assumption for both full participation and partial participation. %
We summarize the technical novelty and provide proofs for all theorems and propositions in Appendix~\ref{sec:appendix-proofs}.

In the literature, three classes of assumptions on
\textit{stochastic gradient variance}, \textit{gradient divergence} and \textit{smoothness} are often made for theoretical analysis \citep{YuAAAI2019,wang2020tackling, khaled2020tighter}. 
We keep Assumption~\ref{assumption:bounded-stochastic-variance} for stochastic gradient variance and Assumption~\ref{assumption:bounded-gradient-divergence} for gradient divergence. Assumptions~\ref{assumption:global-lipschitz-gradient} and \ref{assumption:gradient-to-model} 
will replace  Assumption~\ref{assumption:local-lipschitz-gradient}. 
In Section~\ref{sec:discussions}, we will show that  Assumptions~\ref{assumption:global-lipschitz-gradient} and \ref{assumption:gradient-to-model} are weaker than Assumption~\ref{assumption:local-lipschitz-gradient}.
\begin{assumption}[Global Lipschitz Gradient] 
 \label{assumption:global-lipschitz-gradient}
The global objective function $f(\x)$ satisfies
 \begin{align}
\textstyle \norm{\nabla f(\x) - \nabla f(\y)} \leq L_g\norm{\x - \y}, \forall \x, \y.
 \end{align}
\end{assumption}
In our analysis, the Lipschitz gradient condition is only needed for the global objective function instead of for each local objective function as in Assumption~\ref{assumption:local-lipschitz-gradient} or for each data sample as in \citep{khaled2020tighter}.

\begin{assumption}[Heterogeneity-driven Pseudo-Lipschitz 
Condition on Averaged Gradients]\label{assumption:gradient-to-model}
There exists a constant $L_h\ge 0$ such that $\forall \x_i $,
\begin{align}\label{eq:gradient deviation}
\textstyle \normsq{\frac{1}{N}\sum_{i=1}^N \nabla F_i(\x_i) - \nabla f\left(\bar{\x}\right)} \leq \frac{L_h^2}{N}\sum_{i=1}^N \normsq{\x_i -\bar{\x}},
\end{align}
where $\bar{\x} = \frac{1}{N}\sum_{i=1}^N \x_i$ and $L_h$ is referred to as the {\em heterogeneity-driven pseudo-Lipschitz constant}. 
\end{assumption}
We consider Assumption~\ref{assumption:gradient-to-model} as a new perspective on data heterogeneity for the following reasons. First, $L_h$ can be used to characterize the convergence error caused by local updates. In particular, we will show that $\tilde{L}$ can be replaced by $L_h$ in the local-update related terms in existing convergence bounds in the literature. Second, unlike Assumption~\ref{assumption:bounded-gradient-divergence},  $L_h$ can characterize the difference between the \textit{averaged} model and the \textit{centralized} model. This difference captures the 
actual impact of data heterogeneity as discussed in Section~\ref{sec:intro} (see Figure~\ref{fig:D-varphi}). We will discuss these new perspectives of Assumption~\ref{assumption:gradient-to-model} and $L_h$ in detail in this section and in Section~\ref{sec:discussions}. 

Next, we present the %
convergence analysis for full participation. In the following, we define $\mathcal{F} := f(\x_0) - f^*$.

\begin{theorem}[General Non-convex Objective Functions]\label{thm:non-convex} 
    Assuming Assumptions~\ref{assumption:bounded-stochastic-variance}, \ref{assumption:bounded-gradient-divergence}, \ref{assumption:global-lipschitz-gradient}, \ref{assumption:gradient-to-model} hold,   
    when $\gamma\eta \le \frac{1}{2IL_g}$ and $\gamma\le \min\Big\{ \frac{1}{2\sqrt{30}IL_g}, \frac{1}{\sqrt{6(L_h^2+L_g^2)}I}\Big\}$, 
    after $R$ rounds of FedAvg, we have
    { 
    \begin{align}
    \label{eq: main theorem}
    &\min_{r\in[R]}\Expect\normsq{\nabla f(\bar{\x}^{r})} = \mathcal{O} \bigg( \underbrace{\frac{\mathcal{F}}{\gamma\eta IR} + \frac{\gamma\eta L_g \sigma^2}{N}}_{\text{error unrelated to local updates}} \notag\\
    & + \underbrace{\gamma^2\left(\frac{L_g^2}{N} + L_h^2\right)(I-1)\sigma^2  +  \gamma^2L_h^2(I-1)^2\zeta^2}_{\text{error caused by local updates}}\bigg), 
    \end{align}}%
    {where $[R]:=\{0,1,\ldots,R-1\}$ in this paper.}
\end{theorem}

\textbf{An improved bound by using Assumptions~\ref{assumption:global-lipschitz-gradient} and \ref{assumption:gradient-to-model}.} 
In (\ref{eq: main theorem}), the convergence error terms that are unrelated to local updates only depend on $L_g$, { while in the error caused by local updates, $\sigma^2$ is coupled with both $L_g$ and $L_h$, and $\zeta^2$ is coupled only with $L_h$. In \citet{yu2019linear,yang2020achieving}}, 
the error caused by the stochastic gradient noise is $\mathcal{O}(\frac{\gamma\eta\tilde{L}\sigma^2}{N})$, and the error caused by local updates is $\mathcal{O}(\gamma^2\tilde{L}^2(I-1)^2\zeta^2 + \gamma^2\tilde{L}^2(I-1)\sigma^2)$, where we observe that $\tilde{L}$ is substituted by $L_g$ and $L_h$, respectively, in (\ref{eq: main theorem}).
As shown by the experimental results 
in Table~\ref{tab:estimate-L} (Section~\ref{sec:experiments}), %
$L_g$ is smaller than $\tilde{L}$, and $L_h$ can be much smaller than $\tilde{L}$. {  In addition, our experimental results show that $L_g+L_h$ is not larger than $\tilde{L}$, which intuitively implies that $\frac{L_g^2}{N}+ L_h^2$ is not larger than $ \tilde{L}^2$ (a formal analysis of this relation is left for future work).} This means that the error caused by local updates can be significantly overestimated using the convergence results in existing works.
Moreover, in Section~\ref{sec:discussions}, we show mathematically that $L_h$ and $L_g$ are smaller than $\tilde{L}$.

\textbf{New insights about the effect of data heterogeneity.} It can be observed that in the error caused by local updates, both $\zeta^2$ and $\sigma^2$ are multiplied by $L_h$.
\emph{A key message is that when $\zeta^2$ is large, as long as $L_h^2$ is small enough, the error caused by local updates can still be small.} Since $L_h$ and $\zeta$ characterize the effect of data heterogeneity in different perspectives, we show that it is possible that $L_h=0$ while $\zeta$ can be arbitrarily large by providing an explicit example in Section~\ref{sec:discussions}. In that special case, no matter how large $\zeta$ is, the convergence error of local SGD is the same as that of centralized SGD, i.e., $I$ can be arbitrarily large and only one aggregation is sufficient.

It is worth noting that although %
$L_h$ increases with the percentage of heterogeneous data, %
it can still be small even if the percentage of heterogeneous data is large as shown by the experimental results %
in Table~\ref{tab:estimate-L} (Section~\ref{sec:experiments}). %
The following corollary shows that more local %
updates can improve the convergence. 

\begin{corollary}\label{cor:low-heterogeneity}
     {  With $\gamma\eta =  \min\Big\{\sqrt{\frac{\mathcal{F}N}{RIL_g\sigma^2}},\frac{1}{2IL_g} \Big\}$} and $\gamma = \frac{1}{\sqrt{R}I}$, for FedAvg, we have
    { 
    \begin{align}\label{eq: low degree of data heterogeneity}
        &\min_{r\in[R]}\mathbb{E}\|\nabla f(\bar{\mathbf{x}}^{r})\|^2 \notag \\
        &= \mathcal{O}\bigg(  \sqrt{\frac{\mathcal{F}L_g\sigma^2} {RIN}} + \frac{\mathcal{F}L_g+L_h^2\zeta^2 + (L_h^2+L_g^2/N)\sigma^2/I }{R}\bigg).
    \end{align}
    }
\end{corollary}
It can be seen that the order of the dominant term is $\mathcal{O}(\frac{1}{\sqrt{RI}})$, which is consistent with the results in the literature~\citep{yang2020achieving, karimireddy2020scaffold}. Similar to Theorem~\ref{thm:non-convex}, all $\tilde{L}$ in the existing works is replaced by $L_h$ and $L_g$. Hence, the insights discussed after Theorem~\ref{thm:non-convex} still hold here. {  In Appendix~\ref{sec:strong-convex}, we show that the new assumption can also be applied in the convergence analysis for strongly convex objective functions.}

\textbf{Analysis for Partial Participation.} We also use the new assumption to %
derive the convergence upper bound for partial participation. At each round, $M$ workers are uniformly sampled with replacement. 
The result provides insights into the relationship between local updates and partial participation. 
It is worth noting that the technique for partial participation in existing works cannot be directly applied in our analysis since the Lipschitz gradient  
(see Assumption~\ref{assumption:local-lipschitz-gradient}) is often assumed for each local objective function in the literature.
Therefore, we need to develop new techniques to incorporate the partial participation using $L_h$ and $L_g$, which can be found in Appendix~\ref{sec:appendix-proofs}.

\begin{theorem}[Partial Participation]\label{thm:partial-participation}
Consider uniformly sampling $M$ ($1\le M \le N$) workers in each round of FedAvg with replacement. Assuming Assumptions~\ref{assumption:bounded-stochastic-variance}, \ref{assumption:bounded-gradient-divergence}, \ref{assumption:global-lipschitz-gradient}, \ref{assumption:gradient-to-model} hold, %
{  when $\gamma\eta \le \frac{M}{16IL_g}$,$\gamma\le \min\Big\{\frac{1}{3\sqrt{10} L_g I}, \frac{1}{\sqrt{6(L_h^2+L_g^2)}I} \Big\}$} after $R$ rounds of FedAvg, we have
{ 
\begin{align}\label{eq:bound-partial-participation}
&\min_{r\in[R]}\Expect\normsq{\nabla f(\bar{\x}^{r})} \notag \\
&= \mathcal{O}\bigg(\underbrace{\frac{ \mathcal{F}}{\gamma\eta I R} + \frac{\gamma\eta L_g\sigma^2}{M}}_{\text{error unrelated to local updates}} + \underbrace{\frac{\gamma\eta L_gI\zeta^2}{M}}_{\text{error caused by p.p.}} \notag \\
&+\underbrace{\gamma^2\left(\frac{L_g^2}{N}+L_h^2\right)(I-1)\sigma^2 + \gamma^2L_h^2(I-1)^2\zeta^2}_{\text{error caused by local updates}} \bigg),
\end{align}
}
where ``p.p.'' means partial participation.
\end{theorem}

Compared to Theorem~\ref{thm:non-convex}, there are two differences in the convergence upper bound. First, the
error caused by the stochastic noise $\mathcal{O}\left(\frac{\gamma\eta L_g \sigma^2}{M} \right)$ depends on $M$.
This means that more workers sampled in each round can reduce the noise. Second, there is an additional term $ \mathcal{O}\left(\frac{\gamma\eta L_gI\zeta^2}{M}\right)$ in the convergence upper bound, %
which denotes the error caused by partial participation. In the literature~\citep{yang2020achieving}, this term is often multiplied by $\tilde{L}$. In (\ref{eq:bound-partial-participation}), 
this term depends on $L_g$ and not on $L_h$. 
This means that a small $L_h$ cannot reduce the error caused by partial participation, which can be shown explicitly by the following corollary.

\begin{corollary}\label{cor:partial-participation}
    Consider uniformly sampling $M$ workers at each round in FedAvg with replacement.  {  With $\gamma\eta = \min\Big\{ \sqrt{\frac{M\mathcal{F}}{L_gIR(\sigma^2+I\zeta^2)}}, \frac{1}{15L_gI}\Big\}$ and $\gamma = \frac{1}{\sqrt{R}I}$,} we have 
    { 
    \begin{align}\label{eq:corollary-partial-participation}
    &\min_{r\in[R]} \Expect\normsq{\nabla f(\bar{\x}^{r})} \notag\\
    &= \mathcal{O}\bigg(\sqrt{\frac{\mathcal{F}L_g\zeta^2}{RM}} + \sqrt{\frac{\mathcal{F}L_g\sigma^2}{RIM}} \notag \\
    &\quad\quad\quad+ \frac{\mathcal{F}L_g+L_h^2\zeta^2 + (L_h^2+L_g^2/N)\sigma^2/I }{R} \bigg).
    \end{align}
    }
\end{corollary}

Compared to the existing results, where the dominant term is $\mathcal{O}\left( \frac{\mathcal{F}\tilde{L}\zeta^2}{RM}\right)$, $\tilde{L}$ is substituted by $L_g$ in (\ref{eq:corollary-partial-participation}).
Since the effect of partial participation is shown by the dominant term, this implies that a small $L_h$ cannot reduce the error caused by partial participation.
This is because $L_h$ characterizes the difference between the averaged model over all workers and the centralized model (we will formally explain this property in Section~\ref{sec:discussions}). 
However, with partial participation, the global model on the server becomes a stochastic estimate of the average models over all workers since only a subset of workers are randomly sampled.

\textbf{Applying Assumption~\ref{assumption:gradient-to-model} to other FL algorithms.} Similar to Assumption~\ref{assumption:local-lipschitz-gradient}, the proposed Assumption~\ref{assumption:gradient-to-model} can be used to analyze the performance of other FL algorithms using our methodology. In particular, we provide the convergence analyses for two examples including the FedAvg with momentum~\citep{yu2019linear} in Appendix~\ref{sec:fedavg-momentum} and the FedAdam \citep{reddi2020adaptive} in Appendix~\ref{sec: FedAdam}. In these two examples, the same conclusions on the effect of the heterogeneity-driven pseudo-Lipschitz constant as that for FedAvg can be made.

\section{Discussions}\label{sec:discussions}
In this section, we discuss the properties and advantages of our proposed $L_h$ metric.
First, we show that Assumption~\ref{assumption:global-lipschitz-gradient} (global Lipschitz gradient) and Assumption~\ref{assumption:gradient-to-model} (heterogeneity-driven pseudo-Lipschitz gradient) used in this paper are weaker than the commonly used Assumption~\ref{assumption:local-lipschitz-gradient} (local Lipschitz gradient). Then, we explain the significance of $L_h$ by showing its ability to characterize the difference between the ``virtual'' averaged model (defined in (\ref{eq:virtual-averaged-model})) and the centralized model. Afterwards, we illustrate some nice properties of $L_h$ by considering an exemplar case of quadratic objective functions. By applying $L_h$ in the convergence analysis of quadratic objective functions, we identify a region where local SGD can be better than mini-batch SGD.

\subsection{Properties and Advantages of $L_h$}

\textbf{Additional definition.}
For the exposition of useful insights in our discussion, we define $\hat{\x}^{r,k}$ %
as the ``virtual'' averaged model during the local update phase and 
\begin{align}\label{eq:virtual-averaged-model}
\textstyle \hat{\x}^{r,k+1} := \frac{1}{N} \sum_{i=1}^N \x_i^{r,k+1} = \hat{\x}^{r,k} -\gamma\!\cdot\! \frac{1}{N}\!\sum_{i=1}^N \g_i(\x_i^{r,k}),\!
\end{align}
where $k \in \{0,1,2,\ldots,I-1 \}$.
Note that this virtual model $\hat{\x}^{r,k}$ may not be observed in the system, and is mainly used for the theoretical analysis. In addition, we define $\x_c^{r,k}$ as the model that would have been obtained by applying centralized updates at the $k$th iteration of the $r$th round given the averaged model $\hat{\x}^{r,k}$, which means that the gradient is sampled from the global data distribution $\mathcal{D}$.\footnote{Note that the model $\x_c^{r,k}$ is different from the model obtained by applying the centralized updates from the beginning of the algorithm. We use this for the purpose of illustration only while not affecting the convergence bound results.}  Specifically, 
\begin{align}\label{eq:centralized-update}
\x_c^{r,k+1} := \hat{\x}^{r,k} - \gamma \bar{\g} (\hat{\x}^{r,k}),  
\end{align}
where $\Expect\left[\bar{\g} (\hat{\x}^{r,k}) \right] = \nabla f(\hat{\x}^{r,k})$.

\textbf{Assumptions in this paper are weaker.} In the following proposition, we show that Assumptions~\ref{assumption:global-lipschitz-gradient} and \ref{assumption:gradient-to-model} are weaker than %
Assumption~\ref{assumption:local-lipschitz-gradient}.%

\begin{proposition}\label{lemma:local-assump}
If Assumption~\ref{assumption:local-lipschitz-gradient} holds, then Assumption~\ref{assumption:global-lipschitz-gradient} holds by choosing $L_g=\tilde{L}$ and Assumption~\ref{assumption:gradient-to-model} holds by choosing $L_h=\tilde{L}$.
\end{proposition}

\textbf{Explanation of $L_h$.}
Assumption~\ref{assumption:gradient-to-model} captures the difference between the averaged model and centralized model, which can be seen from the following proposition. Recall that the virtual averaged model $\hat{\x}^{r,k}$ is defined in (\ref{eq:virtual-averaged-model}) and the centralized model $\x_c^{r,k}$ is defined in (\ref{eq:centralized-update}).
\begin{proposition}\label{pro:explain-D}
    Given the virtual averaged model at the $r$th round and $k$th iteration $\hat{\x}^{r,k}$, we have
\begin{align}
&\normsq{\Expect[\hat{\x}^{r,k+1}|\hat{\x}^{r,k}]-\Expect[\x_c^{r,k+1}|\hat{\x}^{r,k}]} \notag \\
&\le \gamma^2\cdot \frac{L_h^2}{N}\sum_{i=1}^N \normsq{\x_i^{r,k} -\hat{\x}^{r,k}}.
\end{align}
\end{proposition}
Proposition~\ref{pro:explain-D} shows that although the difference among local models, captured by $\big\|\x_i^{r,k} -\hat{\x}^{r,k}\big\|^2$ (which depends on both $\zeta$ and $\sigma$ as shown in Lemma~\ref{lemma:model-divergence}), can be large after multiple local iterations, 
the difference between the averaged model and centralized model can still be small if $L_h$ is small.

\subsection{Analysis for Quadratic Objective Functions}
In order to obtain an explicit relationship among $L_h$, $\tilde{L}$ and $\zeta$, and demonstrate the benefit of using $L_h$, we consider the following quadratic objective function,\footnote{Here we do not assume the Hessian matrix is positive definite so that the quadratic objective function can be non-convex.}
\begin{align}\label{eq:quadratic-objective}
   \textstyle F_i(\x) = \frac{1}{2}\x^\T\A_i\x + \bb_i^\T\x + c_i.
\end{align}
Using (\ref{eq: objective function}), %
the global objective function is given by $f(\x) = \frac{1}{2}\x^\T\A\x + \bb^\T\x + c
$, where $\A := \frac{1}{N}\sum_{i=1}^N \A_i$, $\bb := \frac{1}{N}\sum_{i=1}^N \bb_i$ and $c := \frac{1}{N}\sum_{i=1}^N c_i$.

In Proposition~\ref{lemma:local-assump}, it is implied that $L_h\le \tilde{L}$. Further, as shown in Table~\ref{tab:estimate-L} (Section~\ref{sec:experiments}), $L_h$ can be %
much smaller than $\tilde{L}$. In general, the explicit relationship between $L_h$ and $\tilde{L}$ is challenging to derive. However, for quadratic objective functions, this relationship can be shown in the following proposition.

\begin{proposition}
    \label{lemma:DQuadratic}
    For quadratic objective functions defined in (\ref{eq:quadratic-objective}),
    Assumptions~\ref{assumption:local-lipschitz-gradient} and \ref{assumption:gradient-to-model} hold with $\tilde{L} = \max_{i \in [N]}\left\|\A_i\right\|_2$ and
    {  $L_h = \max_{i \in [N]} \left\|\A_i - \A\right\|_2
    $, }respectively, where $\|\cdot\|_2 $ is the spectral norm.
\end{proposition}

From Proposition~\ref{lemma:DQuadratic}, it can be seen that both $L_h$ and $\tilde{L}$ capture the properties of Hessian matrices for quadratic objective functions.
The heterogeneity-driven pseudo-Lipschitz constant $L_h$ characterizes the largest absolute eigenvalue of the ``deviation'' of $\{\A_i\}$ from the global Hessian matrix $\A$, while $\tilde{L}$ characterizes the largest absolute eigenvalue of $\{\A_i\}$. %
We observe that when $\A_i = \A, \forall i$, which means that the difference of local Hessian matrices is zero, Assumption~\ref{assumption:gradient-to-model} holds with $L_h=0$. 
Note that, at the same time, we can pick an $\A_i$ such that $\tilde{L} = \max_{i \in [N]}\left\| \A_i\right\|_2$ is much larger than zero. 
Hence, in this example, we explicitly show that $L_h$ can be arbitrarily smaller than $\tilde{L}$.

In Proposition~\ref{pro:explain-D}, it has been shown that even when $\zeta$ is large, the difference between the averaged model and the centralized model can still be small as long as $L_h$ is small. In the following proposition, we explicitly show that for quadratic functions, $L_h$ can be zero when $\zeta$ is large.

\begin{proposition}\label{proposition:quadratic-D-epsilon}
For quadratic objective functions defined in (\ref{eq:quadratic-objective}), when $\zeta=0$,  {Assumption~\ref{assumption:gradient-to-model} holds with} $L_h=0$, while when $L_h=0$, $\zeta$ can be arbitrarily large.
\end{proposition}
Proposition~\ref{proposition:quadratic-D-epsilon} shows that $L_h=0$ is not a sufficient condition for $\zeta=0$, which implies that only using $\zeta$ can overestimate the effect of the data heterogeneity.
This is because, as we have seen in Proposition~\ref{lemma:DQuadratic}, for quadratic objective functions, the key effect of heterogeneity on the local updates is due to the difference between $\A$ and $\A_i$, while $\zeta$ depends not only on the difference between $\A$ and $\A_i$ but also on the difference between $\mathbf{b}$ and $\mathbf{b}_i$. In addition, we notice that in multi-label learning~\citep{6471714}, when $\A=\A_i$, $\mathbf{b}$ can be very different from $\mathbf{b}_i$ since data examples sharing the same feature can have different labels. This means that $L_h=0$ but $\zeta>0$ is possible in practice.

\textbf{On local SGD v.s. mini-batch SGD.} In the following theorem, we consider the special case of $L_h=0$, by which we show that local SGD can 
outperform mini-batch SGD even when $\zeta$ is arbitrarily large analytically.
The extended discussion on the comparison between local SGD and mini-batch SGD for quadratic objective functions with $L_h>0 $ can be found in Appendix~\ref{sec:dis-quad-mini-local}, where similar conclusions still hold. 
Instead of directly applying $L_h=0$ to Theorem~\ref{thm:non-convex}, 
we develop a 
new proof technique for Theorem~\ref{thm:quadratic} below.
The difference in the techniques can be shown by the 
requirement of the local learning rate $\gamma$, which no longer depends on $I$ in Theorem~\ref{thm:quadratic} while Theorem~\ref{thm:non-convex} requires {$\gamma\le \min\Big\{ \frac{1}{2\sqrt{30}IL_g}, \frac{1}{\sqrt{6(L_h^2+L_g^2)}I}\Big\}$.}

In the following, we use $t$ to denote the index of the total number of iterations, where $t\in [RI] := \{0, 1, \ldots, RI-1\}$. For some given $r$ and $k$, we define $\hat{\x}^t$ as 
\begin{align}
\hat{\x}^t =\bigg\{
\begin{aligned}
&\hat{\x}^{r,k}, &&\textrm{if } t = rI+k \textrm{ and } k\neq 0, \\
&\bar{\x}^r, &&\textrm{if } t = rI.
\end{aligned}
\label{eq:t-index}    
\end{align}

\begin{theorem}[Special Case of $L_h=0$]\label{thm:quadratic}
For quadratic objective functions defined in (\ref{eq:quadratic-objective}), with a common Hessian $\A = \A_i, \forall i$, when $\gamma \le \frac{1}{L_g}$ and $\eta=1$, for local SGD with $I$ local iterations, we have
\begin{align}\label{eq:quadratic-local-bound}
&\min_{t\in [RI]} \E\normsq{\nabla f(\hat{\x}^t)}= \mathcal{O}\left( \frac{\mathcal{F}}{\gamma RI}+\frac{\gamma L_g}{N}\sigma^2 \right);
\end{align}
and for mini-batch SGD with batch size $I$ and learning rate $\gamma\le \frac{1}{L_g}$, we have
\begin{align}\label{eq:quadratic-mini-batch}
&\min_{t\in [RI]} \E\normsq{\nabla f(\hat{\x}^t)}= \mathcal{O}\left( \frac{\mathcal{F}}{\gamma R}+\frac{\gamma L_g}{NI}\sigma^2 \right).
\end{align}
\end{theorem}

In Theorem~\ref{thm:quadratic}, the cost of communication and computation is the same for both local SGD and mini-batch SGD when $R$ is fixed, since the number of aggregations is $R$ and the total number of gradients sampled is $NRI$ for both algorithms.
The upper bound for $\gamma$ is also the same. 
Comparing (\ref{eq:quadratic-local-bound}) with (\ref{eq:quadratic-mini-batch}), we see that %
for local SGD, $I$ is in the first term of (\ref{eq:quadratic-local-bound}), which means that local SGD uses more computation to reduce the error caused by initialization, since %
$\mathcal{F} = f(\bar{\x}^0) - f(\x^*)$. 
For mini-batch SGD, $I$ is in the second term of (\ref{eq:quadratic-mini-batch}), which means that mini-batch SGD uses more computation to reduce the error caused by the variance $\sigma^2$.
Based on the above insights, we identify a region where local SGD can be better than mini-batch SGD in the following corollary.

\begin{corollary}\label{cor:local-better-minibatch}
Consider the quadratic objective function in Theorem~\ref{thm:quadratic}. When $\sigma \le \sqrt{\frac{\mathcal{F}NL_g}{RI}}$ and with appropriately chosen learning rates, for local SGD, we have
\begin{align}
    \min_{t\in [RI]} \Expectbracket{\normsq{\nabla f(\hat{\x}^t)}} = \mathcal{O}\left( \frac{\mathcal{F}L_g}{RI}\right);
\end{align}
for mini-batch SGD, we have
{ 
\begin{align}
    \min_{t\in [RI]} \Expectbracket{\normsq{\nabla f(\hat{\x}^t)}} = \mathcal{O}\left( \frac{\mathcal{F}L_g}{R} \right).
\end{align}
}
\end{corollary}
First, it can be seen that the order of the dominant term for local SGD is $\mathcal{O}(\frac{1}{RI})$ while for mini-batch SGD, it is $\mathcal{O}(\frac{1}{R})$. This means that when $I$ is large, local SGD can be much faster than mini-batch SGD. Second, the condition confirms the intuition that when the error caused by initialization is large ($\sigma \le \sqrt{\frac{\mathcal{F}NL_g}{RI}}$), we should choose local SGD. %
Similar insights can also be shown in the results for $L_h>0$ in Theorem~\ref{thm:quad-localvsmini}.
It is worth noting that this result shows that the advantage of local SGD can be achieved even when $\zeta $ is arbitrarily large, while in the literature~\citep{woodworth2020minibatch}, local SGD has been proved to be better than mini-batch SGD only when $\zeta$ is small.

A limitation of the result in Theorem~\ref{thm:quadratic} and Corollary~\ref{cor:local-better-minibatch} is that the left-hand side (LHS) of the convergence bound includes $\hat{\x}^t$, which can be either the virtual (non-observable) average model $\hat{\x}^{r,k}$, when $k\neq 0$ in (\ref{eq:t-index}), or the observable average model $\bar{\x}^r$, when $k= 0$ in (\ref{eq:t-index}). An extension to considering only the errors related to $\bar{\x}^r$ is left for future work.

\begin{figure*}[h!]
  \centering
  \begin{subfigure}{0.49\columnwidth}
  \centering
  \includegraphics[width=3.2cm, height=2.5cm]{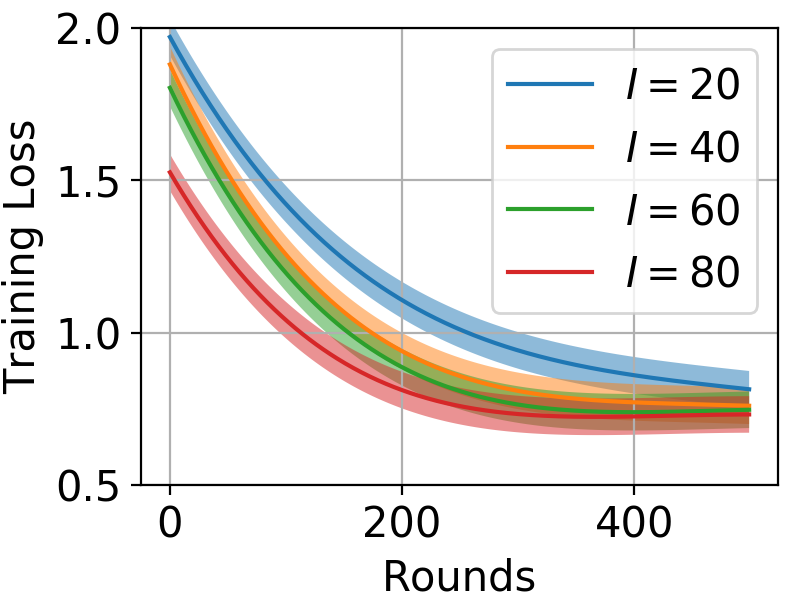}
  \caption{CNN ($50\%$ non-IID).}
  \label{fig:cifar-50-training}
  \end{subfigure}
  \begin{subfigure}{0.49\columnwidth}
  \centering
  \includegraphics[width=3.2cm, height=2.5cm]{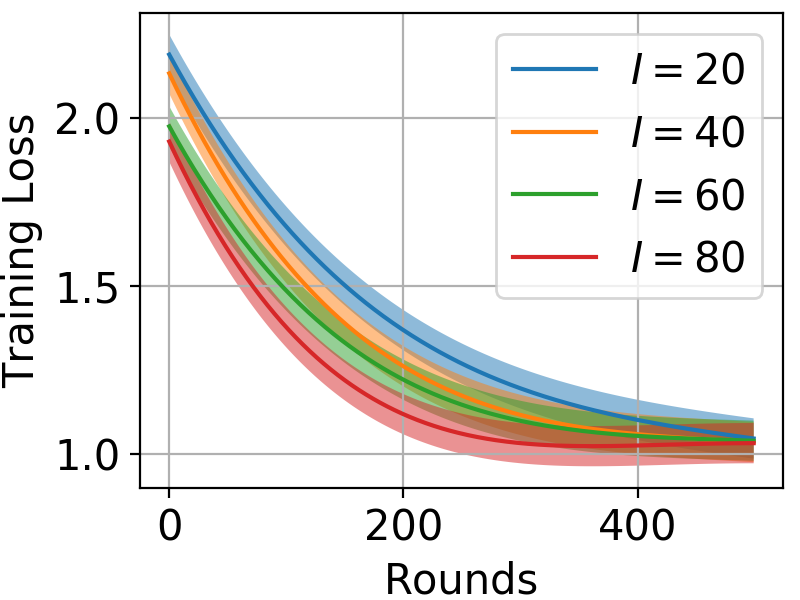}
  \caption{CNN ($75\%$ non-IID).}
  \label{fig:cifar-50-acc}
  \end{subfigure}
  \begin{subfigure}{0.49\columnwidth}
  \centering
  \includegraphics[width=3.2cm, height=2.5cm]{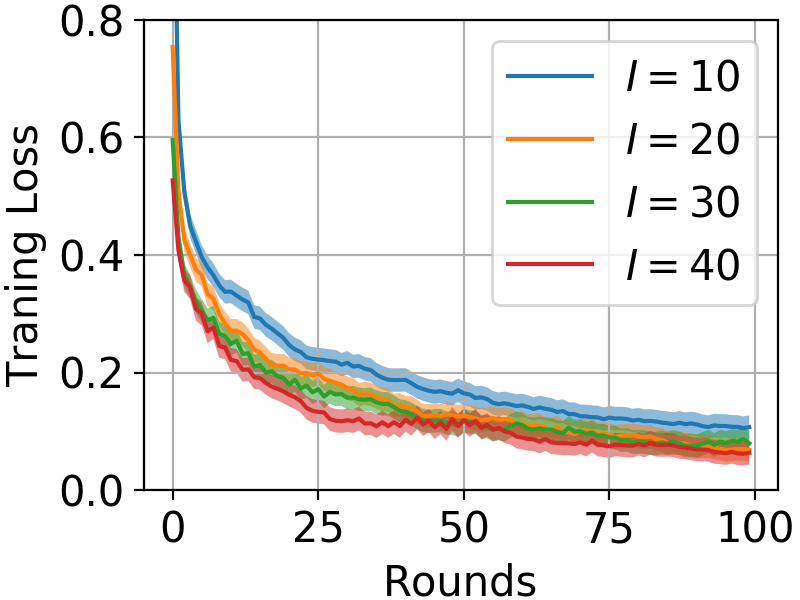}
  \caption{MLP ($50\%$ non-IID).}
  \label{fig:cifar-75-training}
  \end{subfigure}
  \begin{subfigure}{0.49\columnwidth}
  \centering
  \includegraphics[width=3.2cm, height=2.5cm]{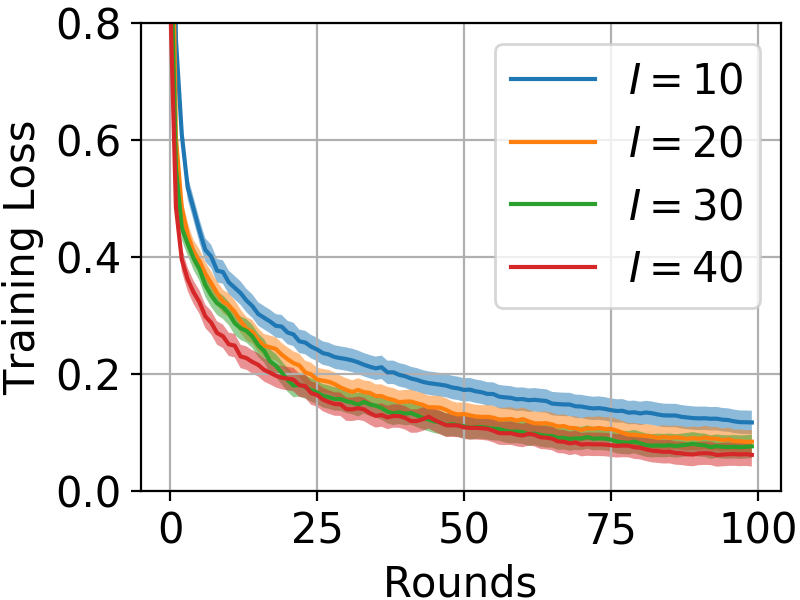}
  \caption{MLP ($75\%$ non-IID).}
  \label{fig:cifar-75-acc}
  \end{subfigure}
  \caption{Results for CNN with CIFAR-10 and MLP with MNIST. For CNN, the learning rates are chosen as $\eta=2$ and $\gamma=0.05$. For MLP, the learning rates are chosen as $\eta=2$ and $\gamma=0.1$. Results for CNN are shown in (a) and (b). Results for $75\%$ of MNIST are shown in (c) and (d).}\label{fig:cnn-mlp}
\end{figure*}

\section{Experiments}\label{sec:experiments}

In this section, we present experimental results obtained from various datasets and models to validate our theoretical findings. In particular, we estimate $\tilde{L}$, $L_h$ and $L_g$ on MNIST \citep{lecun1998gradient} with multilayer perceptron (MLP), CIFAR-10 \citep{CIFAR10} with CNN and VGG-11, CIFAR-100 with VGG-16. Then we provide the results for FedAvg on MNIST and CIFAR-10 to verify the theoretical results in Theorem~\ref{thm:non-convex} and Theorem~\ref{thm:partial-participation}. Results with synthetic data for quadratic objective functions are also provided to verify the insights shown by Theorem~\ref{thm:quadratic}.

The experimental setting is as follows.
For training CNN with CIFAR-10, we partition the training dataset into $100$ workers, and we uniformly sample $10$ workers in each round. For other datasets and models, we partition the training dataset into $10$ workers and use full participation.
For the non-IID setting, the data on each worker is sampled in two steps. First, $X\%$ of the data on one worker is sampled from a single label, and we say that the percentage of heterogeneous data on this worker is $X\%$. Then, we uniformly partition the remaining data into all workers. 
Additional experimental details and results can be found in Appendix~\ref{sec:appendix-experiments}. 

\begin{table}[tb]
\centering
\caption{Estimated $L_h$, $\tilde{L}$, $L_g$ for  MLP with MNIST, CNN and VGG-11 with CIFAR-10 and VGG-16 with CIFAR-100. Since $L_g$ only depends on the global dataset, $L_g$ does not change with the percentage of non-IID (NIID) data. }\label{tab:estimate-L}
\scriptsize
\begin{NiceTabular}{p{0.7cm}|p{0.5cm}|p{1.8cm}|p{1.8cm}|p{0.7cm}}[hlines]
Obj. & NIID & $\tilde{L}$ & $L_h$ & $L_g$ \\
\Block{4-1}{MLP} & $25\%$  & $130.97\pm 11.67$&$0.82 \pm 0.11$& \Block{4-1}{$122.23$\\ $\pm 9.75$} \\
&$50\%$ & $130.97\pm 11.67$&$0.82 \pm 0.11$& \\
&$75\%$ &  $134.24 \pm 12.23$ & $1.66\pm 0.23$& \\
&$100\%$ & $141.92 \pm 12.78$ & $2.36\pm 0.29$& \\
\Block{4-1}{CNN} & $25\%$ & $447.59 \pm 22.27$ &  $0.96\pm 0.13$& \Block{4-1}{$323.35$\\ $\pm 15.36$} \\
& $50\%$ & $898.49 \pm 38.57$ & $1.21\pm 0.19$& \\
& $75\%$ & $1131.36 \pm 47.82$ & $1.63\pm 0.26$& \\
& $100\%$ & $1662.24\pm 62.18$ & $2.15\pm 0.34$& \\
\Block{4-1}{VGG-\\11}&  $25\%$ & $161.07\pm 3.07$  & $9.14 \pm 0.32$ &  \Block{4-1}{$151.95$\\$\pm 5.42$} \\
& $50\%$ & $246.68.\pm 7.67$  & $10.47\pm 0.35$ & \\
&$75\%$ & $338.07\pm 8.43$ & $12.49 \pm 0.41$  &\\
& $100\%$ & $536.92\pm 10.43$  & $17.83\pm 1.12$ & \\
\Block{4-1}{VGG-\\16} & $25\%$ & $363.11\pm 3.87$ & $2.53\pm 0.12$ & \Block{4-1}{$360.85$\\ $\pm 1.86$} \\
& $50\%$ & $365.39\pm 1.1$ &  $3.26\pm 0.12$ & \\
& $75\%$ & $379.91\pm 8.82$ & $3.96 \pm 0.19 $ & \\
& $100\%$ &  $431.07\pm 25.10$ & $4.67\pm 0.27$ & 
\end{NiceTabular}
\end{table}

\textbf{The properties of $\tilde{L}$, $L_h$ and $L_g$ are verified.} The experimental results for estimating $\tilde{L}$, $L_h$ and $L_g$ are shown in Table~\ref{tab:estimate-L}. First, it can be seen that $\tilde{L}$ grows fast as the percentage of the data heterogeneity increases, which implies that $\tilde{L}$ %
is related to the data heterogeneity in addition to the smoothness of local objective function. Second, we observe that both $L_g$ and $L_h$ are smaller than $\tilde{L}$. This verifies the theoretical results in Proposition~\ref{lemma:local-assump}. Furthermore, the results in Table~\ref{tab:estimate-L} show that $L_h$ can be much smaller than $\tilde{L}$. This means that when characterizing the error caused by local updates, substituting $\tilde{L}$ by $L_h$ can reduce the convergence upper bound. %

\textbf{The theoretical results in Theorem~\ref{thm:non-convex} and Theorem~\ref{thm:partial-participation} are verified.} In Figure~\ref{fig:cnn-mlp}, the convergence results for CNN with partial participation and MLP with full participation are provided.
In Table~\ref{tab:estimate-L}, we see that $L_h$ is relatively small in these cases. According to Theorem~\ref{thm:non-convex} and Theorem~\ref{thm:partial-participation}, for both full participation and partial participation, when $L_h$ is small, the error caused by local updates is small, so a large $I$ can still improve convergence. The experimental results in Figure~\ref{fig:cnn-mlp} verify the theoretical results, because even when the percentage of heterogeneous data is more than $50\%$, the largest $I$ ($I=80$ for CNN and $I=40$ for MLP) can still achieve the smallest training loss when $R$ is fixed.

\begin{table}[t]
    \centering
    \caption{Special case of $L_h=0$ with the quadratic objective functions. $I=1$ is equivalent to mini-batch SGD. The number of rounds is the communication rounds needed to achieve the target of $f(\x) = 0.8$. 
    For varying $(\eta,\gamma)$, we fix $I=10$ and for varying $(I,s)$, we fix $\eta=1$, $\gamma=0.005$.}
    \label{tab:quadratic-bs}
    \scriptsize
    \begin{tabular}{c|c|c|c|c}
         \hline
         $(\eta, \gamma)$& ($1,0.005)$&$(2, 0.0025)$&$(5, 0.001)$ &$(10,0.0005)$ \\
         \hline
         \# Rounds& $86\pm 1.6$ & $86\pm 1.6$ &$86\pm 1.6$ & $86\pm 1.6$ \\
         \hline
         \hline
         $(I,s)$& $ (1,1) \& (1,5)$& $(1,10)$ & $(5,1)$ & $(10,1)$ \\
         \hline
         \# Rounds & $927\pm 3.4$ & $925 \pm 1.7$ & $187 \pm 2.3$ & $95\pm 2.4$\\
         \hline
    \end{tabular}
\end{table}

\textbf{The insights gained from the analysis for quadratic objective functions are verified.} %
We construct quadratic examples to verify Theorem~\ref{thm:quadratic}. %
We consider $F_i(\x) = \frac{1}{2}\normsq{\mathbf{U} \x - \mathbf{v}_i}$, where $\mathbf{U}\in \mathbb{R}^{100\times 100}$, $\mathbf{v}_i \in \mathbb{R}^{100}$. Each column of $\mathbf{U}$ 
and $\mathbf{v}_i$ is sampled from a normal distribution $\mathcal{N}(\mathbf{0},\mathbf{I})$. In this case, the gradient divergence is $\normsq{\mathbf{U}(\mathbf{v}_i-\mathbf{v})}>0$. 
We set the stochastic gradient variance as $\sigma^2 = 0.01$.
To distinguish the number of local updates from the mini-batch size in the experiments, we use a separate variable $s$ to indicate the mini-batch size. 
Theorem~\ref{thm:non-convex} shows that when $L_h=0$, using two-sided learning rates does not have advantage over a single learning rate. This is validated by the experiments shown in Table~\ref{tab:quadratic-bs}, where there is no difference among the results with different learning rates when keeping the product of $\gamma$ and $\eta$. %
When $s=1$, comparing the results with $I=1$, $I=5$, and $I=10$ in Table~\ref{tab:quadratic-bs}, we see that more local updates can reduce the number of rounds to achieve $f(\x) = 0.8$ (an arbitrarily chosen target value),
which validates the results in Theorem~\ref{thm:quadratic}. 
By the comparison between the results of $I=1,s=5$ and $I=5, s=1$ and the comparison between the results of $I=1,s=10$ and $I=10, s=1$, we can see that keeping the number of gradients sampled in one round the same, local SGD ($I>1$) converges faster than mini-batch SGD ($I=1$) when $\sigma^2$ is small (since %
$\sigma^2 = 0.01$ in this case),
which validates the discussion for Theorem~\ref{thm:quadratic}.

\section{Conclusion}%

In this paper, we bridged the gap between the pessimistic theoretical results and the good experimental performance for FL by introducing a new theoretical perspective of the data heterogeneity, %
named the heterogeneity-driven pseudo-Lipschitz assumption, which can characterize the difference between the averaged model and the centralized model. This is the key to explain the benefit of local updates, especially when the gradient divergence is large. %
Using this assumption, we developed new analytical approaches to derive convergence upper bounds for FedAvg and its extensions, and for both non-convex and quadratic functions. These bounds can be much smaller than those in the literature and can better explain the effect of data heterogeneity using the heterogeneity-driven pseudo-Lipschitz constant. As a by-product, our approach can identify a region where local SGD can outperform mini-batch SGD without any constraint on the gradient divergence. All theoretical findings were also validated using experiments.

\section*{Impact Statement}
This paper presents work whose goal is to advance the field of machine learning. There are many potential societal consequences of our work, none of which we feel must be specifically highlighted here.

\section*{Acknowledgment}
This work was supported in part by NSF CAREER Award
2145835 and NSF Award 2312227. 

\bibliography{references}

\begin{thebibliography}{30}
\providecommand{\natexlab}[1]{#1}
\providecommand{\url}[1]{\texttt{#1}}
\expandafter\ifx\csname urlstyle\endcsname\relax
  \providecommand{\doi}[1]{doi: #1}\else
  \providecommand{\doi}{doi: \begingroup \urlstyle{rm}\Url}\fi

\bibitem[Bottou et~al.(2018)Bottou, Curtis, and Nocedal]{bottou2018optimization}
Bottou, L., Curtis, F.~E., and Nocedal, J.
\newblock Optimization methods for large-scale machine learning.
\newblock \emph{{SAIM} Review}, 60\penalty0 (2):\penalty0 223--311, 2018.

\bibitem[Darlow et~al.(2018)Darlow, Crowley, Antoniou, and Storkey]{darlow2018cinic10}
Darlow, L.~N., Crowley, E.~J., Antoniou, A., and Storkey, A.~J.
\newblock {CINIC}-10 is not imagenet or {CIFAR}-10, 2018.

\bibitem[Das et~al.(2022)Das, Acharya, Hashemi, Sanghavi, Dhillon, and Topcu]{das2022faster}
Das, R., Acharya, A., Hashemi, A., Sanghavi, S., Dhillon, I.~S., and Topcu, U.
\newblock Faster non-convex federated learning via global and local momentum.
\newblock In \emph{Uncertainty in Artificial Intelligence}, pp.\  496--506. PMLR, 2022.

\bibitem[{\relax Haddadpour, Farzin} et~al.(2019)]{NEURIPS2019_c17028c9}
{\relax Haddadpour, Farzin} et~al.
\newblock Local {SGD} with periodic averaging: Tighter analysis and adaptive synchronization.
\newblock In \emph{Advances in Neural Information Processing Systems}, 2019.

\bibitem[Hubbard \& Hubbard(2015)Hubbard and Hubbard]{hubbard2015vector}
Hubbard, J.~H. and Hubbard, B.~B.
\newblock \emph{Vector calculus, linear algebra, and differential forms: a unified approach}.
\newblock Matrix Editions, 2015.

\bibitem[Jiang \& Agrawal(2018)Jiang and Agrawal]{Peng2018}
Jiang, P. and Agrawal, G.
\newblock A linear speedup analysis of distributed deep learning with sparse and quantized communication.
\newblock In \emph{NeurIPS}, 2018.

\bibitem[Kairouz et~al.(2021)Kairouz, McMahan, Avent, Bellet, Bennis, Bhagoji, Bonawitz, Charles, Cormode, Cummings, et~al.]{kairouz2021advances}
Kairouz, P., McMahan, H.~B., Avent, B., Bellet, A., Bennis, M., Bhagoji, A.~N., Bonawitz, K., Charles, Z., Cormode, G., Cummings, R., et~al.
\newblock Advances and open problems in federated learning.
\newblock \emph{Foundations and Trends{\textregistered} in Machine Learning}, 14\penalty0 (1--2):\penalty0 1--210, 2021.

\bibitem[Karimireddy et~al.(2020)Karimireddy, Kale, Mohri, Reddi, Stich, and Suresh]{karimireddy2020scaffold}
Karimireddy, S.~P., Kale, S., Mohri, M., Reddi, S., Stich, S., and Suresh, A.~T.
\newblock Scaffold: Stochastic controlled averaging for federated learning.
\newblock In \emph{International Conference on Machine Learning}, pp.\  5132--5143. PMLR, 2020.

\bibitem[Khaled et~al.(2020)Khaled, Mishchenko, and Richt{\'a}rik]{khaled2020tighter}
Khaled, A., Mishchenko, K., and Richt{\'a}rik, P.
\newblock Tighter theory for local {SGD} on identical and heterogeneous data.
\newblock In \emph{International Conference on Artificial Intelligence and Statistics}, pp.\  4519--4529. PMLR, 2020.

\bibitem[Krizhevsky \& Hinton(2009)Krizhevsky and Hinton]{CIFAR10}
Krizhevsky, A. and Hinton, G.
\newblock Learning multiple layers of features from tiny images.
\newblock Technical report, University of Toronto, 2009.

\bibitem[LeCun et~al.(1998)LeCun, Bottou, Bengio, and Haffner]{lecun1998gradient}
LeCun, Y., Bottou, L., Bengio, Y., and Haffner, P.
\newblock Gradient-based learning applied to document recognition.
\newblock \emph{Proceedings of the IEEE}, 86\penalty0 (11):\penalty0 2278--2324, 1998.

\bibitem[Li et~al.(2020{\natexlab{a}})Li, Sahu, Talwalkar, and Smith]{LiT_2019}
Li, T., Sahu, A.~K., Talwalkar, A., and Smith, V.
\newblock Federated learning: Challenges, methods, and future directions.
\newblock \emph{IEEE Signal Processing Magazine}, 37\penalty0 (3):\penalty0 50--60, 2020{\natexlab{a}}.

\bibitem[Li et~al.(2020{\natexlab{b}})Li, Sahu, Zaheer, Sanjabi, Talwalkar, and Smith]{li2020federated}
Li, T., Sahu, A.~K., Zaheer, M., Sanjabi, M., Talwalkar, A., and Smith, V.
\newblock Federated optimization in heterogeneous networks, 2020{\natexlab{b}}.

\bibitem[Lin et~al.(2020)Lin, Stich, Patel, and Jaggi]{Lin2020Don't}
Lin, T., Stich, S.~U., Patel, K.~K., and Jaggi, M.
\newblock Don't use large mini-batches, use local {SGD}.
\newblock In \emph{International Conference on Learning Representations}, 2020.
\newblock URL \url{https://openreview.net/forum?id=B1eyO1BFPr}.

\bibitem[McMahan et~al.(2017)McMahan, Moore, Ramage, Hampson, and y~Arcas]{McMahan2017a}
McMahan, B., Moore, E., Ramage, D., Hampson, S., and y~Arcas, B.~A.
\newblock {Communication-Efficient Learning of Deep Networks from Decentralized Data}.
\newblock In \emph{Proc. International Conference on Artificial Intelligence and Statistics (AISTATS)}, pp.\  1273--1282, 2017.

\bibitem[Niknam et~al.(2020)Niknam, Dhillon, and Reed]{niknam2020federated}
Niknam, S., Dhillon, H.~S., and Reed, J.~H.
\newblock Federated learning for wireless communications: Motivation, opportunities, and challenges.
\newblock \emph{IEEE Communications Magazine}, 58\penalty0 (6):\penalty0 46--51, 2020.

\bibitem[Reddi et~al.(2020)Reddi, Charles, Zaheer, Garrett, Rush, Konečný, Kumar, and McMahan]{reddi2020adaptive}
Reddi, S., Charles, Z., Zaheer, M., Garrett, Z., Rush, K., Konečný, J., Kumar, S., and McMahan, H.~B.
\newblock Adaptive federated optimization, 2020.

\bibitem[Richt{\'a}rik et~al.(2021)Richt{\'a}rik, Sokolov, and Fatkhullin]{richtarik2021ef21}
Richt{\'a}rik, P., Sokolov, I., and Fatkhullin, I.
\newblock Ef21: A new, simpler, theoretically better, and practically faster error feedback.
\newblock \emph{Advances in Neural Information Processing Systems}, 34:\penalty0 4384--4396, 2021.

\bibitem[Rieke et~al.(2020)Rieke, Hancox, Li, Milletari, Roth, Albarqouni, Bakas, Galtier, Landman, Maier-Hein, et~al.]{rieke2020future}
Rieke, N., Hancox, J., Li, W., Milletari, F., Roth, H.~R., Albarqouni, S., Bakas, S., Galtier, M.~N., Landman, B.~A., Maier-Hein, K., et~al.
\newblock The future of digital health with federated learning.
\newblock \emph{NPJ digital medicine}, 3\penalty0 (1):\penalty0 1--7, 2020.

\bibitem[Wang \& Joshi(2019)Wang and Joshi]{CooperativeSGD}
Wang, J. and Joshi, G.
\newblock Cooperative {SGD:} {A} unified framework for the design and analysis of communication-efficient {SGD} algorithms.
\newblock In \emph{ICML}, 2019.

\bibitem[Wang et~al.(2020{\natexlab{a}})Wang, Liu, Liang, Joshi, and Poor]{wang2020tackling}
Wang, J., Liu, Q., Liang, H., Joshi, G., and Poor, H.~V.
\newblock Tackling the objective inconsistency problem in heterogeneous federated optimization.
\newblock \emph{Advances in Neural Information Processing Systems}, 33:\penalty0 7611--7623, 2020{\natexlab{a}}.

\bibitem[Wang et~al.(2020{\natexlab{b}})Wang, Tantia, Ballas, and Rabbat]{wangslowmo}
Wang, J., Tantia, V., Ballas, N., and Rabbat, M.
\newblock Slowmo: Improving communication-efficient distributed {SGD} with slow momentum.
\newblock In \emph{International Conference on Learning Representations}, 2020{\natexlab{b}}.

\bibitem[Wang et~al.(2022)Wang, Das, Joshi, Kale, Xu, and Zhang]{wang2022unreasonable}
Wang, J., Das, R., Joshi, G., Kale, S., Xu, Z., and Zhang, T.
\newblock On the unreasonable effectiveness of federated averaging with heterogeneous data, 2022.

\bibitem[Woodworth et~al.(2020{\natexlab{a}})Woodworth, Patel, Stich, Dai, Bullins, Mcmahan, Shamir, and Srebro]{pmlr-v119-woodworth20a}
Woodworth, B., Patel, K.~K., Stich, S., Dai, Z., Bullins, B., Mcmahan, B., Shamir, O., and Srebro, N.
\newblock Is local {SGD} better than minibatch {SGD}?
\newblock In \emph{Proceedings of the 37th International Conference on Machine Learning}, pp.\  10334--10343, 2020{\natexlab{a}}.

\bibitem[Woodworth et~al.(2020{\natexlab{b}})Woodworth, Patel, and Srebro]{woodworth2020minibatch}
Woodworth, B.~E., Patel, K.~K., and Srebro, N.
\newblock Minibatch vs local {SGD} for heterogeneous distributed learning.
\newblock \emph{Advances in Neural Information Processing Systems}, 33:\penalty0 6281--6292, 2020{\natexlab{b}}.

\bibitem[Yang et~al.(2020)Yang, Fang, and Liu]{yang2020achieving}
Yang, H., Fang, M., and Liu, J.
\newblock Achieving linear speedup with partial worker participation in non-iid federated learning.
\newblock In \emph{International Conference on Learning Representations}, 2020.

\bibitem[Yu et~al.(2019{\natexlab{a}})Yu, Jin, and Yang]{yu2019linear}
Yu, H., Jin, R., and Yang, S.
\newblock On the linear speedup analysis of communication efficient momentum {SGD} for distributed non-convex optimization.
\newblock In \emph{ICML}, pp.\  7184--7193, Jun. 2019{\natexlab{a}}.

\bibitem[Yu et~al.(2019{\natexlab{b}})Yu, Yang, and Zhu]{YuAAAI2019}
Yu, H., Yang, S., and Zhu, S.
\newblock Parallel restarted {SGD} with faster convergence and less communication: Demystifying why model averaging works for deep learning.
\newblock In \emph{AAAI}, Jan.--Feb. 2019{\natexlab{b}}.

\bibitem[Zhang \& Zhou(2014)Zhang and Zhou]{6471714}
Zhang, M.-L. and Zhou, Z.-H.
\newblock A review on multi-label learning algorithms.
\newblock \emph{IEEE Transactions on Knowledge and Data Engineering}, 26\penalty0 (8):\penalty0 1819--1837, 2014.
\newblock \doi{10.1109/TKDE.2013.39}.

\bibitem[Zhao et~al.(2018)Zhao, Li, Lai, Suda, Civin, and Chandra]{zhao2018federated}
Zhao, Y., Li, M., Lai, L., Suda, N., Civin, D., and Chandra, V.
\newblock Federated learning with non-iid data.
\newblock \emph{arXiv preprint arXiv:1806.00582}, 2018.

\end{thebibliography}
\bibliographystyle{icml2024}

\newpage
\appendix
\onecolumn

\numberwithin{equation}{section}
\counterwithin{figure}{section}
\counterwithin{theorem}{section}

\counterwithin{table}{section}

\section*{Appendix}

\startcontents[sections]
\printcontents[sections]{l}{1}{\setcounter{tocdepth}{2}}

\clearpage

This Appendix is composed of three sections. 
Appendix~\ref{sec:add-related-works} provides more details of related work and additional theoretical results are presented. 
Appendix~\ref{sec:appendix-proofs} provides all the proofs for theorems, corollaries and propositions in this paper. Appendix~\ref{sec:appendix-experiments} provides additional details and results of experiments. 

\section{Additional Discussions}\label{sec:add-related-works}
In this section, first we provide more discussions on related work. Second, we present the convergence analysis for the quadratic objective functions with $L_h>0$, where we can also identify a parameter region where local SGD can be better than mini-batch SGD. Then we show that $L_h$ and $L_g$ can be successfully applied in the analysis for FedAvg with momentum~\citep{yu2019linear} and FedAdam~\citep{reddi2020adaptive}.

\subsection{Additional Details of Related Work}

There has been considerable work analyzing the convergence rate of federated learning algorithms (not limited to FedAvg), with non-convex objective functions~\citep{NEURIPS2019_c17028c9,YuAAAI2019,CooperativeSGD,karimireddy2020scaffold,reddi2020adaptive}. A key step shared by these analyses is to relate the difference of gradients, 
\begin{align*}
\left\| \frac{1}{N}\sum_{i=1}^N \nabla F_i(\x_i) - \nabla f(\bar{\x})\right\|^2,
\end{align*}
to the model divergence, 
\begin{align}
\label{eq: model divergence}
\frac{1}{N}\sum_{i=1}^N\left\|  \x_i - \bar{\x} \right\|^2,
\end{align}
which can be found, for example, in inequality (10) in the supplementary of \citet{YuAAAI2019}, the inequality (6) in the supplementary of \citet{reddi2020adaptive}, and the proof of Lemma~19 in~\citet{karimireddy2020scaffold}. In this step, the local Lipschitz gradient assumption (Assumption~\ref{assumption:local-lipschitz-gradient}) is often applied, which amplifies the effect of data heterogeneity. In this paper, the pseudo-Lipschitz constant $L_h$ is applied in this step so that the convergence error can be much smaller than that based on the local Lipshictz constant $\tilde{L}$, since it can be seen in Table~\ref{tab:estimate-L} that $L_h$ is often much smaller than $\tilde{L}$. Therefore, we believe that our techniques can also be applied to %
federated learning algorithms to improve the convergence analysis.

There are two papers \citep{wang2022unreasonable,das2022faster} closely related %
to our work. Both works assume the Lipschitz gradient for each local objective function while we only assume it for the global objective function. 
\citet{wang2022unreasonable} %
aim to re-characterize the data heterogeneity by extending the single gradient divergence assumption ((4) in \citet{wang2022unreasonable}) to the averaged gradient divergence assumption ((15) in \citet{wang2022unreasonable}).  
In addition, \citet{wang2022unreasonable} consider the convex objective function and their analysis cannot guarantee convergence to a stationary point while we consider general non-convex objective function and our results can guarantee convergence to a stationary point. 

In the following, we provide more details about the difference between \citet{wang2022unreasonable} and our paper. In \citet{wang2022unreasonable}, a new metric for data heterogeneity, $\rho$, the average drift at optimum, is proposed. The definition of $\rho$ is
\begin{align}
\rho=\left\|\frac{1}{\gamma I}\left(\frac{1}{N}\sum_{i=1}^N\mathbf{x}_i^{r,I} - \bar{\mathbf{x}}^r \right)\right\|.
\end{align}

We discuss the difference between \citet{wang2022unreasonable} and our paper in the following three aspects.

First, the new metric $\rho$ in \citet{wang2022unreasonable} focuses on the difference between models while in our paper, we still focus on the difference between the gradients. The key insight in \citet{wang2022unreasonable} is that %
if $\rho$ is small, when the global model is $\mathbf{x}^*$, after multiple local updates, the averaged model does not change significantly. 
In our paper, the key insight is that since $L_h$ can be small, the difference between the current global gradient and the current averaged local gradients can be small.
In \citet{wang2022unreasonable}, the gradient divergence (Assumption~\ref{assumption:bounded-gradient-divergence}) is not used in the analysis.
In our analysis, we still use the gradient divergence jointly with the proposed $L_h$ to characterize the data heterogeneity.

Second, in \citet{wang2022unreasonable}, it is only empirically shown that $\rho$ can be small. In our paper, we not only empirically %
demonstrate that $L_h$ can be small, but also mathematically proved that $L_h$ is smaller than or equal to $\tilde{L}$ and provide an analytical example to show the exact values of $L_h$ and $\tilde{L}$. Our quadratic example can be non-convex, a case which $\rho$ cannot cover.

Third, one weakness of using $\rho$ is that in the convergence upper bound in \citet{wang2022unreasonable}, the convergence error shown by $\rho$ may not vanish. This means that by choosing $\gamma = \frac{1}{\sqrt{R}}$, when $R$ goes to infinity, the convergence upper bound cannot guarantee that FedAvg can converge to the local minima of the global objective function. On the contrary, the convergence upper bound proved in this paper can guarantee the convergence to the local minima of the global objective function, which is shown by Corollary~\ref{cor:low-heterogeneity}.

\citet{das2022faster} introduce a parameter $\alpha$ to characterize the relationship between the difference of gradients to the model divergence shown in (\ref{eq: model divergence}).
This can also be done using the Assumption~\ref{assumption:gradient-to-model} and $L_h$. 
However, $\alpha$ cannot %
characterize the impact of $L_h$ %
and \citet{das2022faster} still assume Lipschitz gradient for each local objective function.
They only use $\alpha$ as an intermediate step instead of theoretically analyzing the effect of data heterogeneity.
In their theoretical results, the convergence error increases with $I$ even when $\alpha=0$.

\subsection{Additional Results for Quadratic Objective Functions}\label{sec:dis-quad-mini-local}
In this section, we provide a comparison between mini-batch SGD and local SGD for quadratic objective functions with $L_h>0$, i.e., $\mA_i \neq \mA, \forall i$.

Recall that we use $t$ to denote the index of the total number of iterations, where $0\le t \le RI-1$, and the averaged model $\hat{\x}^t$ is defined in (\ref{eq:t-index}).
In the appendix, we also define the local model on worker $i$ at the $t$th iteration as
\begin{align}
\x_i^t = \x_i^{r,k}, t=rI+k.
\end{align}

In the following, we introduce $\kappa$ to characterize the difference between eigenvalues of the Hessian matrices $\{\mA_i\}$. That is,
\begin{align}
\kappa := \max_{i,j} 1 - \frac{\lambda_j(\mA_i)}{\left\| \A_i\right\|_2},
\end{align}
where $\lambda_j(\mA_i)$ is the $j$th eigenvalue of $\mA_i$ and $0< \kappa\le 2$.

It can be seen that $\kappa$ is determined by the smallest eigenvalues of $\{\mA_i\}$. { 
Only when $\lambda_j(\mA_i)<0$ and $|\lambda_j(\mA_i)| = \left\| \A_i\right\|_2 $, }$\kappa$ is maximized and then we have $\kappa=2$.

When analyzing the quadratic objective functions, it is worth noting that the gradient divergence is given by
\begin{align}
\normsq{\nabla F_i(\x) - \nabla f(\x)} = \normsq{(\mA - \mA_i)\x+ \bb - \bb_i},\forall i.
\end{align}
It can be seen that in this case, the gradient divergence cannot be bounded for all $\x\in\R^d$. Therefore, we apply the following assumption in the analysis for the quadratic objective functions.
\begin{assumption}[Weak Gradient Divergence]\label{assumption:weaker-gradient-divergence}
For FedAvg, with quadratic objective functions, for the global model $\bar{\x}\in \{\bar{\x}^0, \bar{\x}^1,\ldots, \bar{\x}^R \}$, we have
\begin{align}
\normsq{\nabla F_i(\bar{\x}) - \nabla f(\bar{\x})} \le \zeta_q^2.
\end{align}
\end{assumption}

\begin{theorem}[Quadratic Objective Functions with $L_h>0$]\label{thm:quad-localvsmini}
With $\gamma\le \min\left\{\frac{1}{\lambda_{\max}},\frac{1}{2L_h}\cdot \min\left\{\frac{1}{I}, \frac{([\varphi(\kappa)]^2-1)^3}{[\varphi(\kappa)]^{2(I+2)}} \right\} \right\}$, for local SGD with quadratic objective functions that satisfy Assumptions~\ref{assumption:bounded-stochastic-variance} and \ref{assumption:weaker-gradient-divergence}, we have 
\begin{align}
\min_{t\in [T]} \E\normsq{\nabla f(\hat{\x}^{t})}\le  \frac{4\mathcal{F}}{\gamma T}+ \frac{2\gamma L_g \sigma^2}{N} + 16\gamma^2L_h^2 I \cdot\phi(\kappa,I)\cdot\zeta_q^2 + 4\gamma^2L_h^2 \cdot\phi(\kappa,I)\cdot\sigma^2,
\end{align}
where $\lambda_{\max} = \max_i \left\| \A_i\right\|_2$,
\begin{align}
\phi(\kappa,k) =
\begin{cases}
      k & 0\le \kappa < 1\\
      \frac{\kappa^{2k}-1}{\kappa^2-1} & 1\le \kappa \le 2.
    \end{cases}
\end{align}
and 
\begin{align}
\varphi(\kappa)=
\begin{cases}
1 & 0\le \kappa <1\\
\kappa & 1\le \kappa \le 2.
\end{cases}
\end{align}
\end{theorem}
The proof can be found in Appendix~\ref{sec:proof-quad-vs}.
Compared to the theoretical results for general non-convex objective functions, the main improvement is on the choice of learning rate. We develop new techniques in the proof to achieve the improvement on the learning rate, which %
takes the advantage of the properties of quadratic objective functions. For the ease of comparison, the convergence bound for mini-batch SGD is %
as follows. With learning rate $\gamma \le \frac{1}{L_g}$,  { for mini-batch SGD  \citep{bottou2018optimization}}, we have 
\begin{align}
    \min_{t\in [T]} \E\normsq{\nabla f(\hat{\x}^{t})}=\mathcal{O}\left(\frac{\mathcal{F}}{\gamma R}+ \frac{\gamma L_g\sigma^2}{2NI} \right),
\end{align}
where $L_g = \|\A \|_2$.

Next, we will present the convergence rate for different values of $\sigma$. The following corollaries can be obtained directly from Theorem~\ref{thm:quad-localvsmini} by plugging in the corresponding learning rate $\gamma$. 

First, we consider the simplest case, that is $\sigma=0$.
\begin{corollary}[$\sigma=0$ for Quadratic Objective Functions] When $\sigma=0$, with $\gamma = \frac{1}{(RI)^{\frac{1}{3}}} $, for local SGD,  we have 
\begin{align}
    \min_{t\in [T]} \E\normsq{\nabla f(\hat{\x}^{t})}=\mathcal{O}\left( \frac{\mathcal{F}+L_h^2 I \cdot\phi(\kappa,I)\cdot\zeta_q^2}{(RI)^{\frac{2}{3}}} \right),
\end{align}
while for mini-batch SGD, we have
\begin{align}
\min_{t\in [T]} \E\normsq{\nabla f(\hat{\x}^{t})}= \mathcal{O}\left( \frac{\mathcal{F}L_g}{R}\right).
\end{align}
\end{corollary}
In this case, when $\frac{R^{\frac{1}{3}}}{I^{\frac{2}{3}}} < \frac{\mathcal{F}L_g}{\mathcal{F}+L_h^2I\phi(\kappa,I)\zeta_q^2}$, %
the convergence rate of local SGD is better than that of mini-batch SGD.

Second, we consider the case %
when $\sigma^2\le \frac{2N\mathcal{F}}{ L_g RI}$.
\begin{corollary}%
When $\sigma^2\le \frac{N\mathcal{F}}{\gamma^2 L_g RI}$, with $\gamma = \frac{1}{(RI)^\frac{1}{3}}$, for local SGD, we have 
\begin{align}
    \min_{t\in [T]} \E\normsq{\nabla f(\hat{\x}^{t})}=\mathcal{O}\left(\frac{\mathcal{F}+L_h^2 I \cdot\phi(\kappa,I)\cdot\zeta_q^2}{(RI)^{\frac{2}{3}}} + \frac{\mathcal{F}NL_h^2\phi(\kappa,I)}{L_gR^{\frac{5}{3}}I^{\frac{5}{3}}}\right),
\end{align}
while for mini-batch SGD, we have
\begin{align}
    \min_{t\in [T]} \E\normsq{\nabla f(\hat{\x}^{t})}=\mathcal{O}\left( \frac{\mathcal{F}L_g}{R}+\frac{\mathcal{F}L_g}{RI^2} \right).
\end{align}
\end{corollary}
Similarly, when $\frac{R^{\frac{1}{3}}}{I^{\frac{2}{3}}} < \frac{\mathcal{F}L_g}{\mathcal{F}+L_h^2I\phi(\kappa,I)\zeta_q^2}$, dominant term of local SGD is better than that of mini-batch SGD.

\subsection{Applying $L_h$ and $L_g$ in the Analysis for FedAvg with Momentum}\label{sec:fedavg-momentum}
In this section, we apply Assumptions~\ref{assumption:global-lipschitz-gradient} and \ref{assumption:gradient-to-model} in the analysis for FedAvg with Momentum in \citet{yu2019linear}. First, we introduce the notations and the algorithm for clarification. We summarize the FedAvg with momentum in Algorithm~\ref{algorithm:fed-momentum}. 

\begin{algorithm}[htb]
\caption{FedAvg with momentum (Algorithm 1 in \citet{yu2019linear})}
\label{algorithm:fed-momentum}
\SetNoFillComment
\KwIn{$\gamma$, $\hat{\x}^0$, $I$, $\beta$, $\hat{\bu}^0 = \mathbf{0}, \forall i$}
\KwOut{Global averaged model $\hat{\x}^{RI}$ }
\For{$t=0$ to $RI-1$}{
\For{Each worker $i$, in parallel}{
\If{$t = aI, 0\le a \le R-1$ }{
$\bu_i^t \leftarrow \hat{\bu}^t $;

$\x_i^t \leftarrow  \hat{\x}^t$;
}
Sample the stochastic gradient $\g_i(\x_i^t)$;

$\bu_i^{t+1} \leftarrow \beta \bu_i^t +\g_i(\x_i^t)$;

$\x_i^{t+1} \leftarrow \x_i^t - \gamma \bu_i^{t+1}$;
}
\If{$t = aI+I-1, 0\le a \le R-1$}{
$\hat{\bu}^{t+1} \leftarrow \frac{1}{N}\sum_{i=1}^N \bu_i^{t+1}$;

$\hat{\x}^{t+1} \leftarrow  \frac{1}{N}\sum_{i=1}^N \x_i^{t+1}$;
}
}
\end{algorithm}

The momentum of worker $i$ at $t$th iteration is denoted by $\bu_i^t\in \R^d$, where $t$ is the index of the total number of iterations. That is, $t = aI + b, a,b\in \mathbb{N}$ and $0\le a \le R-1, 0\le b\le I-1$. At the start of the algorithm, the momentum is initialized as zero. That is, $\bu_i^0 = \mathbf{0}, \forall i$. During local updates, we have 
\begin{align}
    \x_i^{t+1} = \x_i^t - \gamma \bu_i^{t+1},
\end{align}
and
\begin{align}
    \bu_i^{t+1} = \beta \bu_i^t +\g_i(\x_i^t),
\end{align}
where $\beta\in[0,1)$. After local updates, the momentum is reset as 
{ 
\begin{align}
    \bu_i^t = \frac{1}{N}\sum_{j=1}^N \bu_j^t, \forall t = (a+1)I.
\end{align}
Also, the local models are aggregated then updated as 
\begin{align}
    \x_i^t = \frac{1}{N}\sum_{j=1}^N \x_j^t, \forall t = (a+1)I.
\end{align}
}
Now we provide the theoretical results for the new analysis of FedAvg with momentum.
\begin{theorem}[FedAvg with Momentum]\label{thm:fedavg-momentum}
By $\gamma\le \min\{ \frac{(1-\beta)^2}{L_g(1+\beta)}, \frac{1-\beta}{\sqrt{18(L_g^2 + L_h^2)}I}
\}$ , for FedAvg with momentum in Algorithm~\ref{algorithm:fed-momentum}, with Assumptions~\ref{assumption:bounded-stochastic-variance},\ref{assumption:bounded-gradient-divergence}, \ref{assumption:global-lipschitz-gradient} and \ref{assumption:gradient-to-model}, we have
\begin{align}
		\frac{1}{T}\sum_{t=0}^{T-1}\E\normsq{\nabla f(\hat{\x}^t)}\le  \frac{2(1-\beta)(f_0-f_*)}{\gamma T}+\frac{\gamma L_g\sigma^2}{N(1-\beta)^2} + \frac{3\gamma^2L_h^2I\sigma^2}{(1-\beta)^2} + \frac{9\gamma^2L_h^2 I^2\zeta^2}{(1-\beta)^2}.
\end{align}
\end{theorem}
{  The proof can be found in Section~\ref{sec:proof-fed-momentum}.}
For the ease of comparison, we present the convergence bound in Theorem 1 by \citet{yu2019linear} as follows.
\begin{align}
    \frac{1}{T}\sum_{t=0}^{T-1}\E\normsq{\nabla f(\hat{\x}^t)}\le  \frac{2(1-\beta)(f_0-f_*)}{\gamma T}+\frac{\gamma \tilde{L}\sigma^2}{N(1-\beta)^2} + \frac{3\gamma^2\tilde{L}^2I\sigma^2}{(1-\beta)^2} + \frac{9\gamma^2\tilde{L}^2 I^2\zeta^2}{(1-\beta)^2}.
\end{align}
It can be seen that the difference is that in Theorem~\ref{thm:fedavg-momentum}, $\tilde{L}$ is  substituted by $L_g$ and $L_h$. It has been shown in Proposition~\ref{lemma:local-assump} that $L_g\le \tilde{L}$ and $L_h \le \tilde{L}$. Therefore, by applying $L_h$ and $L_g$ in the analysis for FedAvg with momentum, we obtain a tighter convergence upper bound. Similar to Theorem~\ref{thm:non-convex}, the insights of $L_h$ can also be applied to FedAvg with momentum. That is, when $L_h$ is small, the error caused by local updates can still be small.

\subsection{Applying $L_h$ and $L_g$ in the Analysis for FedAdam}
\label{sec: FedAdam}

\begin{algorithm}[htb]
\caption{FedAdam in \cite{reddi2020adaptive}}
\label{alg:fedadam}
\SetNoFillComment
\KwIn{$\gamma$, $\eta$, $\bar{\x}^0$, $I$, $\tau$, $\beta_1$, $\beta_2$ }
\KwOut{Global averaged model $\bar{\x}^R$}
\For{$r=0$ to $R-1$}{
\For{Each worker $i$ in parallel}{
$\x_i^{r,0} \leftarrow \bar{\x}^r$;

\For{$k=0$ to $I-1$}{
Sample a gradient $\g_i(\x_i^{r,k})$;

$\x_i^{r,k+1}\leftarrow \x_i^{r,k}-\gamma \g_i(\x_i^{r,k})$;
}
$\Delta_r^i \leftarrow \x_i^{r,0} - \x_i^{r,I} $;
}
$\Delta_r = \frac{1}{N}\sum_{i=1}^N\Delta^i_r  $;

$\mathbf{m}^r\leftarrow \beta_1\mathbf{m}^{r-1} + (1-\beta_1)\Delta_r   $;

$\bv^r \leftarrow \beta_2\bv^{r-1} + (1-\beta_2)\Delta_r^2 $;

$\bar{\x}^{r+1}\leftarrow \bar{\x}^r - \eta \frac{\mathbf{m}^r}{\sqrt{\bv^r}+\tau}$;
}
\end{algorithm}

In this section, we present the theoretical results for the convergence analysis of FedAdam \citep{reddi2020adaptive}. 
We use $x_{i,j}$ to denote the $j$th element of the model from the $i$th worker, and $[\g_i(\x_i)]_j$ %
to denote the $j$th element of the gradient. In addition to Assumptions~\ref{assumption:global-lipschitz-gradient} and \ref{assumption:gradient-to-model}, we use the following assumption in \citet{reddi2020adaptive} for the analysis. { The algorithm can be found in Algorithm~\ref{alg:fedadam}.}

\begin{assumption}[Bounded Gradients (Assumption 3 in \citet{reddi2020adaptive})]\label{assumption:bounded-gradient}
\begin{align}
\left|[\g_i(\x)]_j\right| \le G, \forall \x\in \R^d, \forall i,j,
\end{align}
where $G$ is a positive constant.
\end{assumption}
Now we present the theoretical results. The proof can be found in Appendix~\ref{sec:proof-fedadam}.
\begin{theorem}\label{thm:fedadam}
Assuming Assumptions~\ref{assumption:bounded-stochastic-variance},\ref{assumption:bounded-gradient-divergence},\ref{assumption:global-lipschitz-gradient}, \ref{assumption:gradient-to-model} and \ref{assumption:bounded-gradient} hold, for FedAdam, with 
\begin{align*}
\gamma\le \min  \left\{\frac{1}{16L_gI}, \frac{1}{\sqrt{6(L_h^2+L_g^2)}I}, \frac{\tau^{\frac{1}{3}}}{16K(120L_g^2G)^{\frac{1}{3}}},\frac{\tau}{6(2G+\eta L_g)} \right\},
\end{align*}
we have
\begin{align}
\min_r\E\normsq{\nabla f(\bar{\x}^r)} \le &\left( \sqrt{\beta_2}\gamma IG +\tau \right) \left(  \frac{8(f_0-f_*)}{\gamma\eta I R} + \frac{\gamma L_g \sigma^2}{\tau N} + \frac{96\gamma^2I^2L_h^2\zeta^2}{\tau} + \frac{32\gamma^2L_hI\sigma^2}{\tau} \right) \notag \\
&+ \left( \sqrt{\beta_2}\gamma IG +\tau \right) \left(\sqrt{1-\beta_2}G + \frac{\eta L_g}{2}\right)\left(\frac{32\gamma }{N\tau^2}
    \sigma^2+ \frac{768\gamma^3 L_h^2 I^3\zeta^2}{\tau^2} + \frac{256\gamma^3L_h^2 I^2\sigma^2}{\tau^2}   \right).
\end{align}
\end{theorem}
From Theorem~\ref{thm:fedadam}, it can be seen that after applying $L_h$ and $L_g$ in the analysis for FedAdam, the error caused by local updates such as $\mathcal{O}(\gamma^2L_h^2I^2\zeta^2)$ and $\mathcal{O}(\gamma^2L_h^2I\sigma^2)$ are related to $L_h$. 
In Theorem~\ref{thm:non-convex}, it has been shown that when $L_h$ is small, the error caused by local updates can be small.
We see here that the insights shown by $L_h$ can also be applied to FedAdam.

{ 
\subsection{Applying $L_h$ and $L_g$ in the Analysis for Strongly Convex Objective Functions}\label{sec:strong-convex}
In this section, we provide the theoretical results for the convergence analysis of FedAvg with strongly convex objective functions. For strongly convex objective functions, we have the following assumption.

\begin{assumption}\label{assump:strong-convex}
The local objective function $F_i(\x)$ is $\mu$-convex for $\mu>0$ and satisfies
\begin{align}
    \left\langle \nabla F_i(\x), \y-\x \right\rangle \le -\left( F_i(\x) - F_i(\y) + \frac{\mu}{2} \normsq{\x-\y} \right), \forall i, \x,\y. 
\end{align}
\end{assumption}
Assumption~\ref{assump:strong-convex} implies that the global objective function $f(\x)$ is also $\mu$-convex, since we have
\begin{align}
    &\left\langle \nabla f(\x), \y-\x \right\rangle \notag \\
    &= \frac{1}{N}\sum_{i=1}^N \left\langle \nabla F_i(\x), \y-\x \right\rangle \notag \\
    &\le -\frac{1}{N}\sum_{i=1}^N \left( F_i(\x) - F_i(\y) + \frac{\mu}{2} \normsq{\x-\y} \right) \notag \\
    &= - \left( f(\x) - f(\y) + \frac{\mu}{2} \normsq{\x-\y} \right). 
\end{align}
The theoretical results for strongly convex objective functions are as follows.

\begin{theorem}[$\mu$-convex]\label{thm:mu-convex}
For $\mu$-strongly convex objective functions, which satisfy Assumption~\ref{assump:strong-convex},  
with Assumptions~\ref{assumption:bounded-stochastic-variance}, \ref{assumption:bounded-gradient-divergence}, \ref{assumption:global-lipschitz-gradient}, \ref{assumption:gradient-to-model},  $\gamma\eta \le \min\{\frac{1}{16L_gI}, \frac{1}{4L_hI} \}$ and $\gamma \le \min \{\frac{1}{24L_gI}\sqrt{\frac{\mu}{L_g}}, \frac{1}{\sqrt{6(L_h^2+L_g^2)}I} \}$,  we have the following convergence upper bound.
\begin{align}
\mathbb{E}[f(\bar{\x}^R)] - f^* = \mathcal{O}\left( \mu\|\bar{\mathbf{x}}^0 - \mathbf{x}^* \|^2 \exp(-\mu\gamma\eta IR) +  \frac{\gamma\eta\sigma^2 }{N} +  \frac{\gamma^2(L_g^2/N+L_h^2 ) I\sigma^2}{\mu} + \frac{\gamma^2L_h^2I^2\zeta^2}{\mu}\right).
\end{align}
\end{theorem}
From this convergence upper bound, it can be seen that the error caused by local updates is $\mathcal{O}\left(\frac{\gamma^2(L_g^2/N+L_h^2 ) I\sigma^2}{\mu} + \frac{\gamma^2L_h^2I^2\zeta^2}{\mu}\right)$, where the impact of data heterogeneity can be characterized by $L_h^2\zeta^2$. This characterization is the same as that shown in Theorem~\ref{thm:non-convex} in the main paper. Using Lemma 1 in \citet{karimireddy2020scaffold}, by carefully choosing the learning rates, we have the following corollary.
\begin{corollary}
By choosing $\gamma = \frac{1}{L_gIR}\sqrt{\frac{\mu}{L_g}}$ and $\gamma\eta = \min \{\frac{\log (\max (1, \mu^2 RIN\|\bar{\mathbf{x}}^0 - \mathbf{x}^* \|^2/\sigma^2 )) }{\mu RI}, \frac{1}{\max\{16L_g,L_h\}I}\}$, we have
\begin{align}
\mathbb{E}[f(\bar{\mathbf{x}}^R)] - f^* = \tilde{\mathcal{O}}\left(
\mu \|\bar{\mathbf{x}}^0 - \mathbf{x}^* \|^2\exp \left(-\frac{\mu R}{\max\{16L_g,L_h\}}\right) + \frac{\sigma^2}{\mu NIR} + \frac{(L_g^2/N+L_h^2)\sigma^2}{L_g^3 I R } + \frac{L_h^2\zeta^2}{L_g^3 R^2} \right),
\end{align}
where $\tilde{\mathcal{O}}(\cdot)$ means $\mathcal{O}(\cdot)$ ignoring logarithmic terms.
\end{corollary}
It can be seen that the dominant term $\mathcal{O}\left( \frac{\sigma^2}{\mu NIR}\right)$ is the same as that in \citet{karimireddy2020scaffold} and the insight of the error caused by local updates still holds.

}

\section{Proofs}\label{sec:appendix-proofs}
The description of FedAvg with two-sided learning rates can be found in Algorithm~\ref{algorithm:fedavg}. For full participation, we have $\mathcal{S}_r=\{1, 2, \ldots, N \}, \forall r$ and $M=N$. For partial participation, we have $M<N$.

\begin{algorithm}[h]
\caption{FedAvg with two-sided learning rates}
\label{algorithm:fedavg}
\SetNoFillComment
\KwIn{$\gamma$, $\eta$, $\bar{\x}^0$, $I$}
\KwOut{Global averaged model $\bar{\x}^{R}$}
\For{$r=0$ to $R-1$}{
Sample a subset of workers $\mathcal{S}_r$, $|\mathcal{S}_r|=M$;\\
Distribute the current global model $\bar{\x}^r$ to workers in $\mathcal{S}_r$;\\
\For{Each worker $i$ in $\mathcal{S}_r$, in parallel}{
\tcc{Local Update Phase} 
$k = 0$;\\
\While{$k<I$}{
Sample the stochastic gradient $\g_i(\x_i^{r,k})$;\\
Update the local model \\
$\x_i^{r,k+1} \leftarrow \x_i^{r,k} - \gamma \g_i(\x_i^{r,k})$;\\
$k \leftarrow k + 1$;}
Send $\Delta_i^r \leftarrow  \bar{\x}^r - \x_i^{r,I}$ to the server;}
\tcc{Global Update Phase}
Update the global model \\
$\bar{\x}^{r+1} \leftarrow \bar{\x}^r - \eta\cdot  \frac{1}{M}\sum_{i\in \mathcal{S}_r} \Delta_i^r $;}
\end{algorithm}

\subsection{Technical Novelty}

Before proceeding to the proof of our theoretical results, we summarize the technical novelty as follows. %

(1) We need to develop new techniques to incorporate Assumptions~\ref{assumption:global-lipschitz-gradient} and \ref{assumption:gradient-to-model}. {  In the proof of Theorem~\ref{thm:non-convex} shown in Section~\ref{sec:proof-non-convex}}, we need to characterize the difference between local gradients. In the literature, this is done by applying the local Lipschitz constant as shown in Assumption~\ref{assumption:local-lipschitz-gradient} in the main paper. In our paper, since Assumption~\ref{assumption:local-lipschitz-gradient} is replaced by our newly introduced Assumption~\ref{assumption:gradient-to-model}, the proof techniques in the literature cannot be applied. It requires to develop new proof techniques to use Assumption~\ref{assumption:gradient-to-model} as shown in the proof of Lemma~\ref{lemma:local-gradient-deviation}-\ref{lemma:distance-averaged-models}. For example, in Lemma~\ref{lemma:local-gradient-deviation}, due to the application of Assumptions~\ref{assumption:global-lipschitz-gradient} and \ref{assumption:gradient-to-model}, we have to cope with a new term, the local gradient deviation $\|\frac{1}{N}\sum_{i=1}^N \nabla F_i(\mathbf{x}_i) - \nabla F_j(\mathbf{x}_j) \|^2$, which cannot be computed using existing techniques. %
Another example is that in the proof of Theorem~\ref{thm:partial-participation}. Due to that we only use the global Lipschitz gradient assumption, we have to derive a new method to bound and incorporate the sampling related term, %
which can be seen from (\ref{eq:tech-novelty-partial-participation}) to (\ref{eq:tech-novely-partial-participation-2}) in Section~\ref{sec:proof-partial-participation}. %

(2) In addition to using Assumption~\ref{assumption:gradient-to-model} to characterize the new convergence rate of FedAvg, we also validate this assumption from the theoretical perspective. We develop the proof for Proposition~\ref{lemma:local-assump}-\ref{pro:explain-D} { in Sections~\ref{sec:proof-prop5.1} and \ref{sec:proof-prop5.2}.} 

(3) Another novelty of our techniques is that in Theorem~\ref{thm:quadratic}, since we use iteration-by-iteration analysis in the proof, the learning rate is not a function of $I$ which is in contrast to that in the literature. For example, it can be seen that in the literature, such as Theorem IV in \citet{karimireddy2020scaffold}, for quadratic objective functions, the learning rate is upper bounded by $\frac{1}{I}$. The advantage of that $\gamma$ is not a function of $I$ can be explained as follows. In Theorem~\ref{thm:quadratic}, %
in order to obtain the optimal learning rate,
we %
choose $\gamma = \frac{1}{\sqrt{RI}}$. This requires that $\frac{1}{\sqrt{RI}}\le \frac{1}{L_g}$, which means that $I$ can be as large as possible. However, if $\gamma \le \frac{1}{IL_g}$ as in \citet{karimireddy2020scaffold}, we will have $\frac{1}{\sqrt{RI}}\le\frac{1}{IL_g} $ such that $I\le \frac{R}{L_g^2}$, which means that to achieve the convergence rate of $\mathcal{O}(\frac{1}{\sqrt{RI}})$, $I$ cannot be arbitrarily large. Therefore, the range of the learning rate in Theorem~\ref{thm:quadratic} can significantly improves the convergence rate. %

\subsection{Additional Lemmas}
In the proof, we use $\x_i$ to denote the local model of worker $i$ regardless of the number of iterations, and use $\bar{\x}:=\frac{1}{N}\sum_{i=1}^N \x_i$ %
to denote the averaged model. 
Following lemmas are useful in the proof for main theorems.

\begin{lemma}[Local Gradient Deviation]\label{lemma:local-gradient-deviation}
With Assumption~\ref{assumption:bounded-gradient-divergence},  \ref{assumption:global-lipschitz-gradient} and \ref{assumption:gradient-to-model}, we have
\begin{align}
&\frac{1}{N}\sum_{j=1}^N\normsq{\frac{1}{N}\sum_{i=1}^N \nabla F_i(\x_i) - \nabla F_j\left(\x_j\right)} \le 3(L_h^2+L_g^2)\cdot \frac{1}{N}\sum_{j=1}^N \normsq{\bar{\x} - \x_j} + 3 \zeta^2.
\end{align}
\end{lemma}

{ 
\begin{lemma}[Model Divergence]\label{lemma:model-divergence}
With $\gamma \le \frac{1}{\sqrt{6(L_h^2+L_g^2)}I}$, we have
\begin{align}
\sum_{k=0}^{I-1} \frac{1}{N}\sum_{i=1}^N \Expect\normsq{\x_i^{r,k} -\hat{\x}^{r,k}} \le 12(I-1)^3\gamma^2 \zeta^2+ 4(I-1)^2\gamma^2 \sigma^2,
\end{align}
\end{lemma}
where $\hat{\x}^{r,k} = \frac{1}{N}\sum_{i=1}^N \x_i^{r,k}$.
}

\begin{lemma}[The Change of Averaged Models]\label{lemma:distance-averaged-models}
{  With $\gamma \le \frac{1}{2\sqrt{3}IL_g}$,} at the $r$th round, we have
\begin{align}
\Expect\normsq{\hat{\x}^{r,k} - \bar{\x}^{r}} \le&  5(I-1) \cdot\frac{\gamma^2\sigma^2}{N} + 30I\gamma^2 \sum_{k=0}^{I-1} \frac{L_h^2}{N}\sum_{i=1}^N \Expect\normsq{\x_i^{r,k} -\hat{\x}^{r,k}}\notag \\
&+30I(I-1) \gamma^2 \Expect\normsq{\nabla f(\bar{\x}^r)}.
\end{align}
\end{lemma}

\begin{lemma}[Model Divergence for FedAvg with Momentum]\label{lemma:model-divergence-momentum}
With $1 - \frac{6\gamma^2I^2(L_h^2 + L_g^2)}{(1-\beta)^2}>0$, for FedAvg with momentum in \citet{yu2019linear}, we have
\begin{align}
    \frac{1}{T}\sum_{t=0}^{T-1}\frac{1}{N}\sum_{i=1}^N \E\normsq{\bar{\x}^t - \x_i^t} \le \frac{1}{1 - \frac{6\gamma^2I^2(L_h^2 + L_g^2)}{(1-\beta)^2}}\cdot \left(\frac{2\gamma^2I\sigma^2}{(1-\beta^2)} + \frac{6\gamma^2I^2\zeta^2}{(1-\beta)^2}  \right).
\end{align}
\end{lemma}

\subsection{Proof of Lemma~\ref{lemma:local-gradient-deviation}}
We start with the LHS of the inequality in Lemma~\ref{lemma:local-gradient-deviation}.
\begin{align}
&\frac{1}{N}\sum_{j=1}^N\normsq{\frac{1}{N}\sum_{i=1}^N \nabla F_i(\x_i) - \nabla F_j\left(\x_j\right)} \notag\\
&= \frac{1}{N}\sum_{j=1}^N \normsq{\frac{1}{N}\sum_{i=1}^N \nabla F_i(\x_i) - \nabla f(\bar{\x}) +  \nabla f(\bar{\x})- \nabla f(\x_j)+\nabla f(\x_j)- \nabla F_j\left(\x_j\right)}\notag \\
&\le 3\normsq{\frac{1}{N}\sum_{i=1}^N \nabla F_i(\x_i) - \nabla f(\bar{\x})} + 3\cdot \frac{1}{N}\sum_{j=1}^N \normsq{\nabla f(\bar{\x})- \nabla f(\x_j)} + 3 \cdot\frac{1}{N}\sum_{j=1}^N \normsq{\nabla f(\x_j)- \nabla F_j\left(\x_j\right)} \notag \\
&\overset{(a)}{\le} 3\normsq{\frac{1}{N}\sum_{i=1}^N \nabla F_i(\x_i) - \nabla f(\bar{\x})} + 3L_g^2\cdot \frac{1}{N}\sum_{j=1}^N \normsq{\bar{\x} - \x_j} + 3 \zeta^2\notag\\
&\overset{(b)}{\le} 3\cdot \frac{L_h^2}{N} \sum_{i=1}^N \normsq{\bar{\x}-\x_i} +3L_g^2 \cdot\frac{1}{N}\sum_{j=1}^N \normsq{\bar{\x} - \x_j} + 3 \zeta^2\notag\\
&= 3(L_h^2+L_g^2) \cdot\frac{1}{N}\sum_{j=1}^N \normsq{\bar{\x} - \x_j} + 3 \zeta^2,
\end{align}
where $(a)$ is due to Assumptions~\ref{assumption:bounded-gradient-divergence} and \ref{assumption:global-lipschitz-gradient} and $(b)$ is due to Assumption~\ref{assumption:gradient-to-model}.

\subsection{Proof of Lemma~\ref{lemma:model-divergence}}
At the $r$th round of FedAvg, we have
\begin{align}
    &\frac{1}{N}\sum_{i=1}^N \Expect\normsq{\x_i^{r,k} -\hat{\x}^{r,k}} \notag\\
    &= \frac{\gamma^2 }{N}\sum_{i=1}^N\Expect\normsq{\sum_{m=0}^{k-1} \left(\g_i(\x_i^{r,m})- \frac{1}{N}\sum_{j=1}^N\g_j(\x_j^{r,m})\right)} \notag \\
    &= \frac{\gamma^2}{N}\sum_{i=1}^N \Expect%
    \Bigg\|\sum_{m=0}^{k-1}\Biggl( \g_i(\x_i^{r,m}) -\nabla F_i(\x_i^{r,m})+\nabla F_i(\x_i^{r,m}) \notag\\
    &\quad \left. -\frac{1}{N}\sum_{j=1}^N\nabla F_j(\x_j^{r,m})+\frac{1}{N}\sum_{j=1}^N\nabla F_j(\x_j^{r,m}) - \frac{1}{N}\sum_{j=1}^N\g_j(\x_j^{r,m})\right) \Bigg\|^2 %
    \notag \\
    &\le 2\cdot\frac{\gamma^2 }{N}\sum_{i=1}^N \Expect\normsq{\sum_{m=0}^{k-1}\left( \nabla F_i(\x_i^{r,m})- \frac{1}{N}\sum_{j=1}^N\nabla F_j(\x_j^{r,m})\right)} \notag \\
    &+ 2\cdot \frac{\gamma^2}{N} \sum_{i=1}^N\normsq{\sum_{m=0}^{k-1}\left(\g_i(\x_i^{r,m}) -\nabla F_i(\x_i^{r,m}) + \frac{1}{N}\sum_{j=1}^N\nabla F_j(\x_j^{r,m}) - \frac{1}{N}\sum_{j=1}^N\g_j(\x_j^{r,m})\right)}\notag \\
    &\overset{(a)}{\le} 2\cdot\frac{\gamma^2}{N}\sum_{i=1}^N \Expect\normsq{\sum_{m=0}^{k-1}\left( \nabla F_i(\x_i^{r,m})- \frac{1}{N}\sum_{j=1}^N\nabla F_j(\x_j^{r,m})\right)} \notag\\
    &\quad + 2\cdot\frac{\gamma^2 }{N}\sum_{i=1}^N \Expect\normsq{\sum_{m=0}^{k-1}\bigg(\g_i(\x_i^{r,m}) -\nabla F_i(\x_i^{r,m})\bigg)} \notag \\
     &\leq 2\cdot\frac{\gamma^2 }{N}\sum_{i=1}^N \Expect\normsq{\sum_{m=0}^{k-1}\left( \nabla F_i(\x_i^{r,m})- \frac{1}{N}\sum_{j=1}^N\nabla F_j(\x_j^{r,m})\right)} +2 \gamma^2  k\sigma^2 \notag\\
     &\le 2k\cdot \frac{\gamma^2 }{N}\cdot\sum_{i=1}^N \sum_{m=0}^{k-1}\Expect\normsq{ \nabla F_i(\x_i^{r,m})- \frac{1}{N}\sum_{j=1}^N\nabla F_j(\x_j^{r,m})} + 2\gamma^2  k\sigma^2\notag \\
     &\overset{(b)}{\le} 2k\gamma^2  \sum_{m=0}^{k-1} \bigg( 3(L_h^2+L_g^2) \frac{1}{N}\sum_{k=1}^N \Expect\normsq{\hat{\x}^{r,m} - \x_k^{r,m}} + 3 \zeta^2  \bigg) + 2\gamma^2  k\sigma^2\notag\\
     &= 6k\gamma^2(L_h^2+L_g^2)\sum_{m=0}^{k-1}\frac{1}{N} \sum_{i=1}^N \Expect\normsq{\hat{\x}^{r,m} - \x_i^{r,m}} + 6k^2\gamma^2 \zeta^2+ 2\gamma^2  k\sigma^2,
\end{align}
where $(a)$ is due to $ \frac{1}{N}\sum_{i=1}^N \normsq{\y_i-\bar{\y}} = \frac{1}{N}\sum_{i=1}^N \normsq{\y_i} - \normsq{\bar{\y}}  \le \frac{1}{N}\sum_{i=1}^N \normsq{\y_i} $,  where $\y_i\in \R^d, \forall i$ and $\bar{\y}=\frac{1}{N}\sum_{i=1}^N \y_i$, and we let $\y_i = \sum_{m=0}^{k-1} \left[\g_i(\x_i^{r,m}) -\nabla F_i(\x_i^{r,m})\right] $, 
and $(b)$ is due to Lemma~\ref{lemma:local-gradient-deviation}.

Note that when $k=I$, we have $\x_i^{r,k} = \x_i^{r+1,0} = \bar{\x}^{r+1}$, and when $k=0$, we have $\x_i^{r,k} = \bar{\x}^r$. So we have $\normsq{\x_i^{r,I} -\hat{\x}^{r,I}}=0$, for $k=0,I$. 
Then sum over $k$ for one round on both sides, we have 
{  
\begin{align}\label{eq:sum-model-divergence}
&\sum_{k=0}^{I-1} \frac{L_h^2}{N}\sum_{i=1}^N \Expect\normsq{\x_i^{r,k} -\hat{\x}^{r,k}} \notag\\
&=\sum_{k=1}^{I-1} \frac{L_h^2}{N}\sum_{i=1}^N \Expect\normsq{\x_i^{r,k} -\hat{\x}^{r,k}} \notag\\
&\le \sum_{k=1}^{I-1} \bigg(6k\gamma^2(L_h^2+L_g^2)\sum_{m=0}^{k-1}\frac{1}{N} \sum_{i=1}^N \Expect\normsq{\hat{\x}^{r,m} - \x_i^{r,m}}+ 6k^2\gamma^2 \zeta^2+ 2\gamma^2  k\sigma^2\bigg)\notag \\
&\overset{(a)}{\le}  3\gamma^2(L_h^2+L_g^2)I(I-1)\sum_{m=0}^{I-1} \frac{1}{N}\sum_{i=1}^N \Expect\normsq{\x_i^{r,m} -\hat{\x}^{r,m}}  \notag \\
&+ 6(I-1)^3\gamma^2 \zeta^2+ 2(I-1)^2\gamma^2  \sigma^2,
\end{align} 
where $(a)$ is due to that $k\le I$ and  $\sum_{k=1}^{I-1}\sum_{m=0}^{k-1} D_m \le \frac{I(I-1)}{2}\sum_{m=0}^{I-1} D_m$ and we let $D_m = \frac{1}{N}\sum_{i=1}^N \Expect\normsq{\x_i^{r,m} -\hat{\x}^{r,m}}$.
}

Moving the first term on RHS of (\ref{eq:sum-model-divergence}) to LHS, %
we have
\begin{align}
&\bigg(1 - 3\gamma^2(L_h^2+L_g^2)I(I-1) \bigg)\sum_{k=0}^{I-1} \frac{1}{N}\sum_{i=1}^N \Expect\normsq{\x_i^{r,k} -\hat{\x}^{r,k}} \notag \\
&\le 6(I-1)^3\gamma^2 \zeta^2+ 2(I-1)^2\gamma^2  \sigma^2.
\end{align}
{ 
With $\gamma < \frac{1}{\sqrt{3(L_h^2+L_g^2)}I}$, we have
\begin{align}
1 - 3\gamma^2(L_h^2+L_g^2)I(I-1) >0.
\end{align}
Then we have 
\begin{align}
\sum_{k=0}^{I-1} \frac{1}{N}\sum_{i=1}^N \Expect\normsq{\x_i^{r,k} -\hat{\x}^{r,k}} \le \frac{1}{1 - 3\gamma^2(L_h^2+L_g^2)I(I-1) } \cdot \left( 6(I-1)^3\gamma^2 \zeta^2+ 2(I-1)^2\gamma^2  \sigma^2 \right).
\end{align}

With $\gamma \le \frac{1}{\sqrt{6(L_h^2+L_g^2)}I}$, we have $\frac{1}{1 - 3\gamma^2(L_h^2+L_g^2)I(I-1) } \le 2$. Then we obtain
\begin{align}
\sum_{k=0}^{I-1} \frac{1}{N}\sum_{i=1}^N \Expect\normsq{\x_i^{r,k} -\hat{\x}^{r,k}} \le 12(I-1)^3\gamma^2 \zeta^2+ 4(I-1)^2\gamma^2  \sigma^2. 
\end{align}
}

\subsection{Proof of Lemma~\ref{lemma:distance-averaged-models}}
At $r$th round, for $k=0$, we have
\begin{align}
    \Expect\normsq{\hat{\x}^{r,k} - \bar{\x}^{r}}=0.
\end{align}

At $r$th round, for $1\le k \le I-1$, 
we have
{ 
\begin{align}
    &\Expect\normsq{\hat{\x}^{r,k} - \bar{\x}^{r}} \notag\\
    &= \Expect\normsq{\hat{\x}^{r,k-1} -\frac{\gamma}{N}\sum_{i=1}^N\g_i(\x_i^{r,k-1})  - \bar{\x}^{r}} \notag \\
    &= \Expect\Bigg\Vert\hat{\x}^{r,k-1}  - \bar{\x}^{r} -\gamma\Bigg(\frac{1}{N}\sum_{i=1}^N\g_i(\x_i^{r,k-1}) - \frac{1}{N}\sum_{i=1}^N\nabla F_i(\x_i^{r,k-1}) +\frac{1}{N}\sum_{i=1}^N\nabla F_i(\x_i^{r,k-1}) \notag\\
    &\quad\quad\quad - \nabla f(\hat{\x}^{r,k-1}) + \nabla f(\hat{\x}^{r,k-1}) - \nabla f(\bar{\x}^r) + \nabla f(\bar{\x}^r) \Bigg) \Bigg\Vert^2   \notag\\
    &\le \E \normsq{\hat{\x}^{r,k-1}  - \bar{\x}^{r}  - \gamma \left(\frac{1}{N}\sum_{i=1}^N\nabla F_i(\x_i^{r,k-1})- \nabla f(\hat{\x}^{r,k-1}) + \nabla f(\hat{\x}^{r,k-1}) - \nabla f(\bar{\x}^r) + \nabla f(\bar{\x}^r)   \right)} \notag\\
    &+ \frac{\gamma^2 \sigma^2}{N} \notag \\
    &\overset{(a)}{\leq} \left(1+ \frac{1}{2I-1} \right) \Expect{\normsq{\hat{\x}^{r,k-1}  - \bar{\x}^r}} + \frac{\gamma^2 \sigma^2}{N}\notag\\
    &+ \gamma^2(1+2I-1)\E\normsq{\frac{1}{N}\sum_{i=1}^N\nabla F_i(\x_i^{r,k-1})- \nabla f(\hat{\x}^{r,k-1}) + \nabla f(\hat{\x}^{r,k-1}) - \nabla f(\bar{\x}^r) + \nabla f(\bar{\x}^r) }  \notag \\
    &\le \left(1+\frac{1}{2I-1}\right) \Expect{\normsq{\hat{\x}^{r,k-1}  - \bar{\x}^r}} + \frac{\gamma^2 \sigma^2}{N} + 6I\gamma^2 \Expect{\normsq{\frac{1}{N}\sum_{i=1}^N\nabla F_i(\x_i^{r,k-1}) - \nabla f(\hat{\x}^{r,k-1})}} \notag \\
    &\quad\quad\quad + 6I\gamma^2 \Expect{\normsq{\nabla f(\hat{\x}^{r,k-1}) - \nabla f(\bar{\x}^r)}} + 6I\gamma^2 \Expect{\normsq{\nabla f(\bar{\x}^r)}} \notag \\
    & \\
    &\overset{(b)}{\leq} \left(1+\frac{1}{2I-1} + 6I\gamma^2 L_g^2\right) \Expect{\normsq{\hat{\x}^{r,k-1}  - \bar{\x}^r}} + \frac{\gamma^2 \sigma^2}{N} + \frac{6I\gamma^2 L_h^2}{N} \sum_{i=1}^N \Expect{\normsq{\x_i^{r,k-1} - \hat{\x}^{r,k-1}}} \notag \\
    &\quad\quad\quad + 6I\gamma^2 \Expect{\normsq{\nabla f(\bar{\x}^r)}} \notag\\
     & \\
    &\overset{(c)}{\le} \left(1+\frac{1}{I-1} \right) \Expect{\normsq{\hat{\x}^{r,k-1}  - \bar{\x}^r}} + \frac{\gamma^2 \sigma^2}{N} + \frac{6I\gamma^2 L_h^2}{N} \sum_{i=1}^N \Expect{\normsq{\x_i^{r,k-1} - \hat{\x}^{r,k-1}}}+ 6I\gamma^2 \Expect{\normsq{\nabla f(\bar{\x}^r)}} \notag \\
    &\overset{(d)}{\le} 5(I-1)\cdot \frac{\gamma^2\sigma^2}{N} + 30I\gamma^2 \sum_{k=0}^{I-1} \frac{L_h^2}{N}\sum_{i=1}^N \Expect\normsq{\x_i^{r,k} -\hat{\x}^{r,k}}   +30I(I-1) \gamma^2 \Expect\normsq{\nabla f(\bar{\x}^r)},
\end{align}
where $(a)$ is due to that $\normsq{\x+\y}\le (1+p)\normsq{\x}+(1+\frac{1}{p})\normsq{\y}, \forall p>0, \forall \x,\y \in \R^d$,
}

$(b)$ is due to Assumptions~\ref{assumption:global-lipschitz-gradient} and \ref{assumption:gradient-to-model}, $(c)$ is due to that by choosing $\gamma \le \frac{1}{2\sqrt{3}L_g I}$, we have
\begin{align}
1+\frac{1}{2I-1} + 6I\gamma^2 L_g^2 \le 1+ \frac{1}{2(I-1)} + \frac{1}{2I} \le 1+ \frac{1}{I-1},
\end{align}
and $(d)$ is due to $(1+\frac{1}{q})^q < e, \forall q>0$, where $e$ is the natural exponent.

\subsection{Proof of Lemma~\ref{lemma:model-divergence-momentum}}
By Lemma 5 in \citet{yu2019linear}, for FedAvg with momentum, we have 
\begin{align}\label{eq:model-divergence-momentum}
    &\frac{1}{N}\sum_{i=1}^N \E\normsq{\bar{\x}^t - \x_i^t}\notag \\
    &\le \frac{2\gamma^2I\sigma^2}{(1-\beta^2)} + 2\gamma^2\cdot \frac{1}{N}\sum_{i=1}^N \E  \normsq{\sum_{\tau =t_0}^{t-1} \left[ \nabla F_i(\x_i^\tau ) - \frac{1}{N}\sum_{j=1}^N \nabla F_j(\x_j^\tau )   \right]\frac{1-\beta^{t-\tau}}{1-\beta}}, 
\end{align}
{ where $t = aI+b, 1\le b\le I$ and $t_0 = aI$.}

{  Note that $t-t_0 \le I$.}
For the second term in the RHS of (\ref{eq:model-divergence-momentum}), we have
{ 
\begin{align}
    &\frac{1}{N}\sum_{i=1}^N \E  \normsq{\sum_{\tau =t_0}^{t-1} \left[ \nabla F_i(\x_i^\tau ) - \frac{1}{N}\sum_{j=1}^N \nabla F_j(\x_j^\tau )   \right]\frac{1-\beta^{t-\tau}}{1-\beta}}  \notag \\
    &\le (t-t_0)\sum_{\tau =t_0}^{t-1}\frac{1}{N}\sum_{i=1}^N \E\normsq{\left[ \nabla F_i(\x_i^\tau ) - \frac{1}{N}\sum_{j=1}^N \nabla F_j(\x_j^\tau )   \right]\frac{1-\beta^{t-\tau}}{1-\beta}} \notag \\
    &\le (t-t_0)\sum_{\tau =t_0}^{t-1}\frac{1}{N}\sum_{i=1}^N \E\normsq{ \nabla F_i(\x_i^\tau ) - \frac{1}{N}\sum_{j=1}^N \nabla F_j(\x_j^\tau ) }\left(\frac{1-\beta^{t-\tau}}{1-\beta}\right)^2 \notag \\
    &\le \frac{I}{(1-\beta)^2} \sum_{\tau =t_0}^{t-1}\frac{1}{N}\sum_{i=1}^N \E\normsq{\nabla F_i(\x_i^\tau ) - \frac{1}{N}\sum_{j=1}^N \nabla F_j(\x_j^\tau )}  \notag\\
    &\overset{(a)}{\le}\frac{I}{(1-\beta)^2} \sum_{\tau =t_0}^{t-1}\left(3(L_h^2+L_g^2)\cdot \frac{1}{N}\sum_{j=1}^N \normsq{\bar{\x} - \x_j} + 3 \zeta^2 \right) \notag \\
    &\le \frac{3I(L_h^2 + L_g^2)}{(1-\beta)^2}\sum_{\tau =t_0}^{t-1} \frac{1}{N}\sum_{i=1}^N \E\normsq{\x_i^\tau - \hat{\x}^\tau} + \frac{3I^2\zeta^2}{(1-\beta)^2},
\end{align}
}

where $(a)$ is due to Lemma~\ref{lemma:local-gradient-deviation}.
Substituting back to (\ref{eq:model-divergence-momentum}), we obtain
\begin{align}
    \frac{1}{N}\sum_{i=1}^N \E\normsq{\bar{\x}^t - \x_i^t}\le \frac{2\gamma^2I\sigma^2}{(1-\beta^2)} + \frac{6\gamma^2I(L_h^2 + L_g^2)}{(1-\beta)^2}\sum_{\tau =t_0}^{t-1} \frac{1}{N}\sum_{i=1}^N \E\normsq{\x_i^\tau - \hat{\x}^\tau} + \frac{6\gamma^2I^2\zeta^2}{(1-\beta)^2}.
\end{align}
{ Taking the average over $t$ on both sides, we obtain
\begin{align}
\frac{1}{T}\sum_{t=0}^{T-1}\frac{1}{N}\sum_{i=1}^N \E\normsq{\bar{\x}^t - \x_i^t}&\le \frac{2\gamma^2I\sigma^2}{(1-\beta^2)} + \frac{6\gamma^2I(L_h^2 + L_g^2)}{(1-\beta)^2}\frac{1}{T}\sum_{t=0}^{T-1}\sum_{\tau =t_0}^{t-1} \frac{1}{N}\sum_{i=1}^N \E\normsq{\x_i^\tau - \hat{\x}^\tau} + \frac{6\gamma^2I^2\zeta^2}{(1-\beta)^2}\notag \\
&\le \frac{2\gamma^2I\sigma^2}{(1-\beta^2)} + \frac{6\gamma^2I^2(L_h^2 + L_g^2)}{(1-\beta)^2}\frac{1}{T}\sum_{t=0}^{T-1} \frac{1}{N}\sum_{i=1}^N \E\normsq{\x_i^\tau - \hat{\x}^\tau} + \frac{6\gamma^2I^2\zeta^2}{(1-\beta)^2}
\end{align}
}
Rearranging the above inequality, with $1 - \frac{6\gamma^2I^2(L_h^2 + L_g^2)}{(1-\beta)^2}>0$,  we get
\begin{align}
    \frac{1}{T}\sum_{t=0}^{T-1}\frac{1}{N}\sum_{i=1}^N \E\normsq{\bar{\x}^t - \x_i^t} \le \frac{1}{1 - \frac{6\gamma^2I^2(L_h^2 + L_g^2)}{(1-\beta)^2}}\cdot \left(\frac{2\gamma^2I\sigma^2}{(1-\beta^2)} + \frac{6\gamma^2I^2\zeta^2}{(1-\beta)^2}  \right).
\end{align}

\subsection{Proof of Theorem~\ref{thm:non-convex}}\label{sec:proof-non-convex}
With Assumption~\ref{assumption:global-lipschitz-gradient}, we have 
\begin{align}\label{eq:Lip-first-step}
    \E\left[f(\bar{\x}^{r+1})\right]&\leq \E\left[f(\bar{\x}^{r}) \right] - \gamma\eta \E\innerprod{\nabla f(\bar{\x}^{r}), \frac{1}{N}\sum_{i=1}^N \sum_{k=0}^{I-1} \g_i(\x_i^{r,k})} + \frac{\gamma^2\eta^2 L_g}{2}\E{\normsq{\frac{1}{N}\sum_{i=1}^N \sum_{k=0}^{I-1} \g_i(\x_i^{r,k})}}\notag \\
    &= \E\left[f(\bar{\x}^{r}) \right] - \gamma\eta \E\innerprod{\nabla f(\bar{\x}^{r}), \frac{1}{N}\sum_{i=1}^N \sum_{k=0}^{I-1} \E_{\x_i^{r,k}}\left[\g_i(\x_i^{r,k})\right]} + \frac{\gamma^2\eta^2 L_g}{2}\E\normsq{\frac{1}{N}\sum_{i=1}^N \sum_{k=0}^{I-1} \g_i(\x_i^{r,k})} \notag \\
    &= \E\left[f(\bar{\x}^{r}) \right] - \gamma\eta \E\innerprod{\nabla f(\bar{\x}^{r}), \frac{1}{N}\sum_{i=1}^N \sum_{k=0}^{I-1} \nabla F_i(\x_i^{r,k})} + \frac{\gamma^2\eta^2 L_g}{2}\E\normsq{\frac{1}{N}\sum_{i=1}^N \sum_{k=0}^{I-1} \g_i(\x_i^{r,k})}. 
\end{align}

The second term in the RHS of (\ref{eq:Lip-first-step}) can be computed as follows. 
\begin{align}
    &- \gamma\eta \E\innerprod{\nabla f(\bar{\x}^{r}), \frac{1}{N}\sum_{i=1}^N \sum_{k=0}^{I-1} \nabla F_i(\x_i^{r,k})} \notag\\
    &=- \frac{\gamma\eta}{I} \E\innerprod{I \nabla f(\bar{\x}^{r}), \frac{1}{N}\sum_{i=1}^N \sum_{k=0}^{I-1} \nabla F_i(\x_i^{r,k})} \notag\\
    &= \frac{\gamma\eta}{2I}\left\{\E\normsq{\frac{1}{N}\sum_{i=1}^N \sum_{k=0}^{I-1} \left(\nabla F_i(\x_i^{r,k}) - \nabla f(\bar{\x}^{r})\right)} - I^2 \E\normsq{\nabla f(\bar{\x}^{r})} -\E \normsq{\frac{1}{N}\sum_{i=1}^N \sum_{k=0}^{I-1} \nabla F_i(\x_i^{r,k}) }\right\}  \notag\\
    &= \frac{\gamma\eta}{2I}\left\{\E\normsq{\sum_{k=0}^{I-1} \left(\frac{1}{N}\sum_{i=1}^N \nabla F_i(\x_i^{r,k}) - \nabla f(\hat{\x}^{r,k}) \right) + \sum_{k=0}^{I-1}\left(\nabla f(\hat{\x}^{r,k})  - \nabla f(\bar{\x}^{r})\right)} \right.\notag\\
    &\quad\quad\quad \left. - I^2 \E\normsq{\nabla f(\bar{\x}^{r})} - \E {\normsq{\frac{1}{N}\sum_{i=1}^N \sum_{k=0}^{I-1} \nabla F_i(\x_i^{r,k}) }} \right\}\notag\\
    &\leq \frac{\gamma\eta}{2I}\left\{2I \sum_{k=0}^{I-1} \E{\normsq{ \frac{1}{N}\sum_{i=1}^N \nabla F_i(\x_i^{r,k}) - \nabla f(\hat{\x}^{r,k})}} + 2I \sum_{k=0}^{I-1}\E{\normsq{\nabla f(\hat{\x}^{r,k})  - \nabla f(\bar{\x}^{r})}} \right. \notag\\
    &\quad\quad\quad \left. - I^2 \E\normsq{\nabla f(\bar{\x}^{r})} - \E{\normsq{\frac{1}{N}\sum_{i=1}^N \sum_{k=0}^{I-1} \nabla F_i(\x_i^{r,k}) }} \right\}\notag\\
    &\overset{(a)}{\leq} \frac{\gamma\eta}{2I}\left\{\frac{2IL_h^2}{N} \sum_{k=0}^{I-1} \sum_{i=1}^N \E{\normsq{ \x_i^{r,k} - \hat{\x}^{r,k}}} + 2IL_g^2 \sum_{k=0}^{I-1}\E{\normsq{\hat{\x}^{r,k} - \bar{\x}^{r}}} \right. \notag\\
    &\quad\quad\quad \left. - I^2 \E\normsq{\nabla f(\bar{\x}^{r})} - \E{\normsq{\frac{1}{N}\sum_{i=1}^N \sum_{k=0}^{I-1} \nabla F_i(\x_i^{r,k}) }} \right\}, \label{eq:inner-prod-expansion}
\end{align}
where $(a)$ is due to Assumption~\ref{assumption:global-lipschitz-gradient} and Assumption~\ref{assumption:gradient-to-model}.

The third term in the RHS of (\ref{eq:Lip-first-step}) can be computed as follows. 
{ 
\begin{align}
    &\frac{\gamma^2 \eta^2 L_g}{2}\E{\normsq{\frac{1}{N}\sum_{i=1}^N \sum_{k=0}^{I-1} \g_i(\x_i^{r,k})}} \notag \\
    &= \frac{\gamma^2 \eta^2 L_g}{2}\E{\normsq{\frac{1}{N}\sum_{i=1}^N \sum_{k=0}^{I-1} \g_i(\x_i^{r,k})-\frac{1}{N}\sum_{i=1}^N \sum_{k=0}^{I-1} \nabla F_i(\x_i^{r,k}) + \frac{1}{N}\sum_{i=1}^N \sum_{k=0}^{I-1} \nabla F_i(\x_i^{r,k})}}  \notag \\
    &\le \gamma^2 \eta^2 L_g\E \normsq{\frac{1}{N}\sum_{i=1}^N \sum_{k=0}^{I-1} \nabla F_i(\x_i^{r,k})}+\gamma^2 \eta^2 L_g\E{\normsq{\frac{1}{N}\sum_{i=1}^N \sum_{k=0}^{I-1} \left[\g_i(\x_i^{r,k}) - \nabla F_i(\x_i^{r,k})\right]}} \notag \\
    &\overset{(a)}{\le} \gamma^2 \eta^2 L_g\E\normsq{\frac{1}{N}\sum_{i=1}^N \sum_{k=0}^{I-1} \nabla F_i(\x_i^{r,k})} + \frac{\gamma^2 \eta^2 IL_g \sigma^2}{N}.\label{eq:third-term-expansion}
\end{align}
Now we explain $(a)$ in (\ref{eq:third-term-expansion}). We have
\begin{align}
&\E{\normsq{\frac{1}{N}\sum_{i=1}^N \sum_{k=0}^{I-1} \left[\g_i(\x_i^{r,k}) - \nabla F_i(\x_i^{r,k})\right]}} \notag \\
&=\frac{1}{N^2}\sum_{i=1}^N\sum_{i'=1}^N \sum_{k=0}^{I-1}\sum_{k'=0}^{I-1}\E \left\langle \g_i(\x_i^{r,k}) - \nabla F_i(\x_i^{r,k}), \g_{i'}(\x_{i'}^{r,k'}) - \nabla F_{i'}(\x_{i'}^{r,k'}) \right\rangle.
\end{align}
When $i\neq i'$, we have
\begin{align}
&\E \left\langle \g_i(\x_i^{r,k}) - \nabla F_i(\x_i^{r,k}), \g_{i'}(\x_{i'}^{r,k'}) - \nabla F_{i'}(\x_{i'}^{r,k'}) \right\rangle \notag \\
&= \E\left( \E\left[\left\langle \g_i(\x_i^{r,k}) - \nabla F_i(\x_i^{r,k}), \g_{i'}(\x_{i'}^{r,k'}) - \nabla F_{i'}(\x_{i'}^{r,k'}) \right\rangle  |\x_i^{r,k}, \x_{i'}^{r,k'}\right]\right)\notag \\
&=0.
\end{align}
When $i=i'$ but $k\neq k'$, suppose that $k\le k'$,
\begin{align}
&\E \left\langle \g_i(\x_i^{r,k}) - \nabla F_i(\x_i^{r,k}), \g_{i}(\x_{i}^{r,k'}) - \nabla F_{i}(\x_{i}^{r,k'}) \right\rangle \notag \\
&\E \left( \E \left[\left\langle \g_i(\x_i^{r,k}) - \nabla F_i(\x_i^{r,k}), \g_{i}(\x_{i}^{r,k'}) - \nabla F_{i}(\x_{i}^{r,k'}) \right\rangle |\x_{i}^{r,0},\x_{i}^{r,1},\x_{i}^{r,2},\ldots, \x_{i}^{r,k'} \right] \right)\notag \\
&= 0.
\end{align}

Therefore, we have
\begin{align}
&\E{\normsq{\frac{1}{N}\sum_{i=1}^N \sum_{k=0}^{I-1} \left[\g_i(\x_i^{r,k}) - \nabla F_i(\x_i^{r,k})\right]}} \notag \\
&=\frac{1}{N^2}\sum_{i=1}^N\sum_{k=0}^{I-1} \E \left(\E\left[\normsq{\g_i(\x_i^{r,k}) - \nabla F_i(\x_i^{r,k}) }|\x_i^{r,k}\right] \right) \notag \\
&\le \frac{I\sigma^2}{N}.
\end{align}

}

{ 
Substituting (\ref{eq:inner-prod-expansion}) and (\ref{eq:third-term-expansion}) 
to (\ref{eq:Lip-first-step}), 
we have

\begin{align}\label{eq:apply-lemmas}
    &\E[f(\bar{\x}^{r+1})] \notag\\
    &\leq \E\left[f(\bar{\x}^{r}) \right] + \frac{\gamma\eta L_h^2}{N} \sum_{k=0}^{I-1} \sum_{i=1}^N \E{\normsq{ \x_i^{r,k} - \hat{\x}^{r,k}}} + \gamma\eta L_g^2 \sum_{k=0}^{I-1}\E{\normsq{\hat{\x}^{r,k} - \bar{\x}^{r}}} \notag\\
    & - \frac{\gamma\eta I}{2} \E\normsq{\nabla f(\bar{\x}^{r})} - \gamma\eta\left(\frac{1}{2I} - \gamma\eta L_g \right) \E{\normsq{\frac{1}{N}\sum_{i=1}^N \sum_{k=0}^{I-1} \nabla F_i(\x_i^{r,k}) }} + \frac{\gamma^2 \eta^2 IL_g \sigma^2}{N} \notag \\
    &\overset{(a)}{\leq}  \E\left[f(\bar{\x}^{r}) \right] + \frac{\gamma\eta L_h^2}{N} \sum_{k=0}^{I-1} \sum_{i=1}^N \E\normsq{ \x_i^{r,k} - \hat{\x}^{r,k}} + \gamma\eta L_g^2 \sum_{k=0}^{I-1}\E{\normsq{\hat{\x}^{r,k} - \bar{\x}^{r}}}\notag \\
    &- \frac{\gamma\eta I}{2} \E\normsq{\nabla f(\bar{\x}^{r})} + \frac{\gamma^2 \eta^2 IL_g \sigma^2}{N} \notag \\
    &\overset{(b)}{\le} \E\left[f(\bar{\x}^{r}) \right] + \frac{\gamma\eta L_h^2}{N}\sum_{k=0}^{I-1} \sum_{i=1}^N \E\normsq{ \x_i^{r,k} - \hat{\x}^{r,k}}- \frac{\gamma\eta I}{2} \E\normsq{\nabla f(\bar{\x}^{r})} + \frac{\gamma^2 \eta^2 IL_g \sigma^2}{N}\notag \\
    &+ \gamma\eta L_g^2 I \left( 5(I-1) \frac{\gamma^2\sigma^2}{N} + 30I\gamma^2 \sum_{k=0}^{I-1} \frac{L_h^2}{N}\sum_{i=1}^N \E\normsq{\x_i^{r,k} -\hat{\x}^{r,k}}  
    +30I(I-1) \gamma^2 \normsq{\nabla f(\bar{\x}^r)} \right)\notag \\
    &\le \E\left[f(\bar{\x}^{r}) \right] - \left(\frac{\gamma\eta I}{2} - 30\gamma^3\eta L_g^2I^2(I-1) \right)\E\normsq{\nabla f(\bar{\x}^{r})} + \frac{\gamma^2 \eta^2 IL_g \sigma^2}{N} + 5\gamma^3\eta L_g^2 I(I-1) \frac{\sigma^2}{N} \notag \\
    &+ \left(\frac{\gamma\eta L_h^2}{N} + \frac{30\gamma^3\eta L_g^2L_h^2I^2 }{N} \right)\sum_{k=0}^{I-1} \sum_{i=1}^N \E\normsq{ \x_i^{r,k}- \hat{\x}^{r,k}}.
\end{align}
where $(a)$ is {  due to $\gamma\eta < \frac{1}{2IL_g }$, 
$(b)$ is due to Lemma~\ref{lemma:distance-averaged-models}.}
By $\gamma \le \frac{1}{2\sqrt{30}L_g I}$, we have
\begin{align}
    \frac{\gamma\eta I}{2} - 30\gamma^3\eta L_g^2I^2(I-1) \le \frac{\gamma\eta I}{4},
\end{align}
and 
\begin{align}
    \frac{\gamma\eta L_h^2}{N} + \frac{30\gamma^3\eta L_g^2L_h^2I^2 }{N} \le \frac{3\gamma\eta L_h^2}{2N}.
\end{align}
Substituting back to (\ref{eq:apply-lemmas}), we obtain
\begin{align}
    &\E[f(\bar{\x}^{r+1})] \notag\\
    &\le \E\left[f(\bar{\x}^{r}) \right] -\frac{\gamma\eta I}{4}\E\normsq{\nabla f(\bar{\x}^{r})} + \frac{\gamma^2 \eta^2 IL_g \sigma^2}{N} + 5\gamma^3\eta L_g^2 I(I-1) \frac{\sigma^2}{N} + \frac{3\gamma\eta L_h^2}{2N}\sum_{k=0}^{I-1} \sum_{i=1}^N \E\normsq{ \x_i^{r,k}- \hat{\x}^{r,k}}\notag \\
    &\overset{(a)}{\le} \E\left[f(\bar{\x}^{r}) \right] -\frac{\gamma\eta I}{4}\E\normsq{\nabla f(\bar{\x}^{r})} + \frac{\gamma^2 \eta^2 IL_g \sigma^2}{N} + 5\gamma^3\eta L_g^2 I(I-1) \frac{\sigma^2}{N} \notag \\
    &+\frac{3\gamma\eta L_h^2}{2}\left[12(I-1)^3\gamma^2 \zeta^2+ 4(I-1)^2\gamma^2 \sigma^2\right].
\end{align}
where $(a)$ is due to using Lemma~\ref{lemma:model-divergence}.
Moving $-\frac{\gamma\eta I}{4}\E\normsq{\nabla f(\bar{\x}^{r})}$ to left and taking the average over $r$, we obtain
\begin{align}
&\frac{1}{R}\sum_{r=0}^{R-1} \E\normsq{\nabla f(\bar{\x}^{r})} \le \frac{4(f(\bar{\x}^0) - f^*)}{\gamma\eta IR} + \frac{4\gamma\eta L_g \sigma^2}{N} + \frac{20\gamma^2 L_g^2 (I-1) \sigma^2}{N} + 24\gamma^2 L_h^2(I-1)\sigma^2  + 72 \gamma^2L_h^2(I-1)^2\zeta^2.
\end{align}

Then we have
\begin{align}
\min_{r\in[R]}\Expect\normsq{\nabla f(\bar{\x}^{r})} &\leq \frac{1}{R}\sum_{r=0}^{R-1} \E\normsq{\nabla f(\bar{\x}^{r})} \nonumber\\
&= \mathcal{O} \bigg(\frac{f(\x^0)- f^*}{\gamma\eta IR} + \frac{\gamma\eta L_g \sigma^2}{N}
+ \gamma^2\left(\frac{L_g^2}{N} + L_h^2\right)(I-1)\sigma^2  +  \gamma^2L_h^2(I-1)^2\zeta^2\bigg). 
\end{align}
}

\subsection{Proof of Theorem~\ref{thm:partial-participation}}
\label{sec:proof-partial-participation}

{ 
In this section, we define an identity random variable to indicate the participation of workers in the following.
In each round, the server performs $M$ times of sampling. Then $\forall j \in [M], i\in [N]$, we have 
\begin{align}
    \mathds{1}^r_{j,i} = 
    \begin{cases}
    1, \text{worker $i$ is chosen at $j$th sampling  of $r$th round,} \\
    0, \text{else},
    \end{cases}
\end{align}
where $\sum_{i=1}^N \mathds{1}^r_{j,i} = 1$. Since we consider uniform sampling with replacement, we have
\begin{align}
p(\mathds{1}^r_{j,i} = 1) = \frac{1}{N}, \forall i,j,r,
\end{align}
and
\begin{align}
\E_{\mathcal{S}_r}[\mathds{1}^r_{j,i}] = \frac{1}{N},
\end{align}
where $\E_{\mathcal{S}_r}[\cdot]$ means taking the expectation over sampling at $r$th round. In addition, we have
\begin{align}
    \E_{\mathcal{S}_r} \normsq{\mathds{1}^r_{j,i}\mathbf{z}} = \E_{\mathcal{S}_r} \mathds{1}^r_{j,i}\normsq{\mathbf{z}}=\frac{1}{N}\normsq{\mathbf{z}},
\end{align}
for any $\mathbf{z}$ that is independent of $\mathds{1}^r_{j,i}$, where the first equality is because $(\mathds{1}^r_{j,i})^2=\mathds{1}^r_{j,i}$;
and for $i\neq i'$,
\begin{align}
    \E_{\mathcal{S}_r} \left[\mathds{1}^r_{j,i}\mathds{1}^r_{j,i'}\right]= 0.
\end{align}

Here we assume that the sampling workers and sampling gradients are independent.
}

{  
With Assumption~\ref{assumption:global-lipschitz-gradient}, after one round of FedAvg, we have 
\begin{align}\label{eq:partial-participation}
    \Expect \left[f(\bar{\x}^{r+1})\right] &\leq \Expect \left[f(\bar{\x}^{r})\right] - \gamma\eta \Expect  \innerprod{\nabla f(\bar{\x}^{r}), \frac{1}{M}\sum_{j=1}^M\sum_{i=1}^N \sum_{k=0}^{I-1}\mathds{1}^r_{j,i}  \g_i(\x_i^{r,k})} + \frac{\gamma^2\eta^2 L_g}{2}\Expect \normsq{\frac{1}{M}\sum_{j=1}^M\sum_{i=1}^N \sum_{k=0}^{I-1}\mathds{1}^r_{j,i}  \g_i(\x_i^{r,k})}.
\end{align}

It can be seen that the inner-product term is the same as that in (\ref{eq:inner-prod-expansion}). So we have 
\begin{align} \label{eq: inner-prod pp}
&- \gamma\eta \Expect  \innerprod{\nabla f(\bar{\x}^{r}), \frac{1}{M}\sum_{j=1}^M\sum_{i=1}^N \sum_{k=0}^{I-1}\mathds{1}^r_{j,i}  \g_i(\x_i^{r,k})} \notag \\
& = - \gamma\eta \Expect  \innerprod{\nabla f(\bar{\x}^{r}), \frac{1}{M}\sum_{j=1}^M\sum_{i=1}^N \sum_{k=0}^{I-1}\E_{\mathcal{S}^r}[\mathds{1}^r_{j,i}]  \E_{\x_i^{r,k}}[\g_i(\x_i^{r,k})]} \notag \\
&= - \gamma\eta \Expect{\innerprod{\nabla f(\bar{\x}^{r}), \frac{1}{N}\sum_{i=1}^N \sum_{k=0}^{I-1} \nabla F_i(\x_i^{r,k})}} \notag \\
&\le \frac{\gamma\eta}{2I}\left\{\frac{2IL_h^2}{N} \sum_{k=0}^{I-1} \sum_{i=1}^N \Expect\normsq{ \x_i^{r,k} - \hat{\x}^{r,k}} + 2IL_g^2 \sum_{k=0}^{I-1}\Expect{\normsq{\hat{\x}^{r,k} - \bar{\x}^{r}}} \right. \notag\\
    & \left. - I^2\Expect \normsq{\nabla f(\bar{\x}^{r})} - \Expect{\normsq{\frac{1}{N}\sum_{i=1}^N \sum_{k=0}^{I-1} \nabla F_i(\x_i^{r,k}) }} \right\}.
\end{align}

}

In this case, %
we consider $\x_i^{r,k},i\notin \mathcal{S}_r$ as a virtual local model on worker $i$, which is not computed in the system. 
The virtual local model is mainly used for %
analysis.
Similar to (\ref{eq:third-term-expansion}), for the third term in the RHS of (\ref{eq:partial-participation}), we have
{ 
\begin{align}\label{eq:moment-partial-participation}
&\Expect \normsq{\frac{1}{M}\sum_{j=1}^M\sum_{i=1}^N \sum_{k=0}^{I-1}\mathds{1}^r_{j,i}  \g_i(\x_i^{r,k})} \notag \\
&=\Expect\normsq{\frac{1}{M}\sum_{j=1}^M\sum_{i=1}^N \sum_{k=0}^{I-1}\mathds{1}^r_{j,i}\left[ \g_i(\x_i^{r,k}) - \nabla F_i(\x_i^{r,k}) + \nabla F_i(\x_i^{r,k})\right]}  \notag \\
&\le 2\Expect\normsq{\frac{1}{M}\sum_{j=1}^M\sum_{i=1}^N \sum_{k=0}^{I-1}\mathds{1}^r_{j,i}\left[ \g_i(\x_i^{r,k}) - \nabla F_i(\x_i^{r,k})\right]} + 2\Expect\normsq{\frac{1}{M}\sum_{j=1}^M\sum_{i=1}^N \sum_{k=0}^{I-1}\mathds{1}^r_{j,i} \nabla F_i(\x_i^{r,k})}\notag\\
&\le \frac{2I\sigma^2}{M} + 2\Expect \left[\E_{\mathcal{S}_r}
\normsq{\frac{1}{M}\sum_{j=1}^M\sum_{i=1}^N \sum_{k=0}^{I-1}\mathds{1}^r_{j,i} \nabla F_i(\x_i^{r,k})}\right].
\end{align}

Now we consider the expectation on sampling.
}
Let $Q_i = \sum_{k=0}^{I-1} \nabla F_i(\x_i^{r,k})$, then for the second term in the RHS of (\ref{eq:moment-partial-participation}), we have 

{ 
\begin{align}
&\E_{\mathcal{S}_r}
\normsq{\frac{1}{M}\sum_{j=1}^M\sum_{i=1}^N \sum_{k=0}^{I-1}\mathds{1}^r_{j,i} \nabla F_i(\x_i^{r,k})} = \Expect_{\mathcal{S}_r}\normsq{\frac{1}{M}\sum_{j=1}^M\sum_{i=1}^N \mathds{1}^r_{j,i} Q_i} \notag \\
&= \Expect_{\mathcal{S}_r}\normsq{\frac{1}{M}\sum_{j=1}^M\sum_{i=1}^N \mathds{1}^r_{j,i} Q_i - \frac{1}{N}\sum_{i=1}^N Q_i + \frac{1}{N}\sum_{i=1}^N Q_i} \notag \\
&\overset{(a)}{=} \Expect_{\mathcal{S}_r}\normsq{\frac{1}{M}\sum_{j=1}^M\sum_{i=1}^N \mathds{1}^r_{j,i} Q_i - \frac{1}{N}\sum_{i=1}^N Q_i} + \normsq{\frac{1}{N}\sum_{i=1}^N Q_i}, 
\end{align}
} 
{ 
where $(a)$ is due to
\begin{align}
&\E_{\mathcal{S}^r}\left\langle \frac{1}{M}\sum_{j=1}^M\sum_{i=1}^N \mathds{1}^r_{j,i} Q_i - \frac{1}{N}\sum_{i=1}^N Q_i, \frac{1}{N}\sum_{i=1}^N Q_i\right\rangle =0.
\end{align}
}
{ 
Further, we have
\begin{align}
&\Expect_{\mathcal{S}_r}\normsq{\frac{1}{M}\sum_{j=1}^M\sum_{i=1}^N \mathds{1}^r_{j,i} Q_i - \frac{1}{N}\sum_{i=1}^N Q_i} \notag \\
&=\frac{1}{M^2}\Expect_{\mathcal{S}_r}\left[  \sum_{j=1}^M\normsq{\frac{1}{N}\sum_{i=1}^N \left(N\mathds{1}^r_{j,i} - 1 \right)Q_i} + \sum_{i\neq j'} \left\langle \frac{1}{N}\sum_{i=1}^N \left(N\mathds{1}^r_{j,i} - 1 \right)Q_i, \frac{1}{N}\sum_{i=1}^N \left(N\mathds{1}^r_{j',i} - 1 \right)Q_i \right\rangle \right] \notag \\
&= \frac{1}{M^2}\sum_{j=1}^M\Expect_{\mathcal{S}_r} \normsq{\frac{1}{N}\sum_{i=1}^N \left(N\mathds{1}^r_{j,i} - 1 \right)Q_i} \notag \\
&=\frac{1}{M^2}\sum_{j=1}^M \Expect_{\mathcal{S}_r} \normsq{\sum_{i=1}^N \mathds{1}^r_{j,i} Q_i -  \frac{1}{N}\sum_{i=1}^N Q_i }  \notag \\
&= \frac{1}{M^2}\sum_{j=1}^M \left[\Expect_{\mathcal{S}_r}\normsq{\sum_{i=1}^N \mathds{1}^r_{j,i} Q_i} - 2\Expect_{\mathcal{S}_r}\left\langle \sum_{i=1}^N \mathds{1}^r_{j,i} Q_i,\frac{1}{N}\sum_{i=1}^N Q_i \right\rangle + \normsq{\frac{1}{N}\sum_{i=1}^N Q_i} \right] \notag \\
&= \frac{1}{M^2}\sum_{j=1}^M \Expect_{\mathcal{S}_r}\normsq{\sum_{i=1}^N \mathds{1}^r_{j,i} Q_i} - \frac{1}{M}\normsq{\frac{1}{N}\sum_{i=1}^N Q_i} \notag \\
&= \frac{1}{M^2}\sum_{j=1}^M \left[ \sum_{i=1}^N\Expect_{\mathcal{S}_r}\normsq{\mathds{1}^r_{j,i}}\normsq{Q_i} + \sum_{i\neq i'}\Expect_{\mathcal{S}_r}\left\langle  \mathds{1}^r_{j,i} Q_i, \mathds{1}^r_{j,i'} Q_{i'} \right\rangle  \right]- \frac{1}{M}\normsq{\frac{1}{N}\sum_{i=1}^N Q_i} \notag \\
&= \frac{1}{MN}\sum_{i=1}^N\normsq{Q_i} - \frac{1}{M}\normsq{\frac{1}{N}\sum_{i=1}^N Q_i}. 
\end{align}
}

Substituting above results back to (\ref{eq:moment-partial-participation}), we obtain
{ 
\begin{align}\label{eq:moment-partial-participation-2}
\Expect \normsq{\frac{1}{M}\sum_{j=1}^M\sum_{i=1}^N \sum_{k=0}^{I-1}\mathds{1}^r_{j,i}  \g_i(\x_i^{r,k})} \le \frac{2I\sigma^2}{M} + \frac{2}{MN}\sum_{i=1}^N \Expect \normsq{\sum_{k=0}^{I-1} \nabla F_i(\x_i^{r,k})} + 2\cdot\frac{M-1}{M}\Expect \normsq{\frac{1}{N}\sum_{j=1}^N \sum_{k=0}^{I-1} \nabla F_j(\x_j^{r,k})}.
\end{align}
}
For the second term of (\ref{eq:moment-partial-participation-2}), we have 
\begin{align}\label{eq:tech-novelty-partial-participation}
&\Expect \normsq{\sum_{k=0}^{I-1} \nabla F_i(\x_i^{r,k})} \notag \\
&= \Expect \normsq{\sum_{k=0}^{I-1} \left[ \nabla F_i(\x_i^{r,k}) - \nabla f(\x_i^{r,k}) +\nabla f(\x_i^{r,k}) - \nabla f(\hat{\x}^{r,k}) + \nabla f(\hat{\x}^{r,k}) - \nabla f(\bar{\x}^r) +  \nabla f(\bar{\x}^r)\right]} \notag \\
&\overset{(a)}{\le}  4I^2\zeta^2 + 4L_g^2I\sum_{k=0}^{I-1}\Expect \normsq{\x_i^{r,k} - \hat{\x}^{r,k}} + 4L_g^2I \sum_{k=0}^{I-1}\Expect\normsq{\hat{\x}^{r,k} - \bar{\x}^r} + 4I^2\Expect\normsq{\nabla f(\bar{\x}^r)},
\end{align}
where $(a)$ is due to Assumption~\ref{assumption:bounded-gradient-divergence} and Assumption~\ref{assumption:global-lipschitz-gradient}.
Substituting back and rearranging, we have 
{ 
\begin{align} \label{eq: gradient pp}
\frac{\gamma^2\eta^2 L_g}{2}\Expect\normsq{\frac{1}{M}\sum_{j=1}^M\sum_{i=1}^N \sum_{k=0}^{I-1}\mathds{1}^r_{j,i}  \g_i(\x_i^{r,k})} &\le \frac{\gamma^2\eta^2L_gI\sigma^2}{M} + \frac{\gamma^2\eta^2L_g(M-1)}{M} \Expect\normsq{\frac{1}{N}\sum_{i=1}^N \sum_{k=0}^{I-1} \nabla F_i(\x_i^{r,k})} \notag \\
&+ \frac{4\gamma^2\eta^2L_gI^2\zeta^2}{M} + \frac{4\gamma^2\eta^2L_g^3I}{MN}\sum_{i=1}^N\sum_{k=0}^{I-1}\Expect \normsq{\x_i^{r,k} - \hat{\x}^{r,k}} \notag \\
&+ \frac{4\gamma^2\eta^2L_g^3I}{MN}\sum_{i=1}^N\sum_{k=0}^{I-1}\Expect\normsq{\hat{\x}^{r,k} - \bar{\x}^r} + \frac{4\gamma^2\eta^2L_gI^2}{M}\Expect\normsq{\nabla f(\bar{\x}^r)}.
\end{align}

Substituting (\ref{eq: inner-prod pp}) and (\ref{eq: gradient pp}) back to (\ref{eq:partial-participation}), we have
\begin{align}
\Expect \left[f(\bar{\x}^{r+1})\right] &\leq \Expect \left[f(\bar{\x}^{r})\right] - \left(\frac{\gamma\eta I}{2} - \frac{4\gamma^2\eta^2L_g I^2}{M}\right)\Expect \normsq{\nabla f(\bar{\x}^{r})} \notag\\
&- \left( \frac{\gamma\eta}{2I} - \frac{\gamma^2\eta^2L_g(M-1)}{M}\right)\Expect\normsq{\frac{1}{N}\sum_{i=1}^N \sum_{k=0}^{I-1} \nabla F_i(\x_i^{r,k})}\notag \\
&+ \frac{\gamma^2\eta^2L_gI\sigma^2}{M} + \frac{4\gamma^2\eta^2L_gI^2\zeta^2}{M} + \left( \gamma\eta L_h^2 + \frac{4\gamma^2 \eta^2 L_g^3 I}{M}\right)\cdot \frac{1}{N}\sum_{k=0}^{I-1} \sum_{i=1}^N\Expect\normsq{ \x_i^{r,k} - \hat{\x}^{r,k}} \notag \\
&+ \left(\gamma\eta L_g^2 + \frac{4\gamma^2\eta^2L_g^3 I}{M} \right) \sum_{k=0}^{I-1}\Expect{\normsq{\hat{\x}^{r,k} - \bar{\x}^{r}}}.
\end{align}

By Lemma~\ref{lemma:distance-averaged-models}, we have
\begin{align}
\Expect \left[f(\bar{\x}^{r+1})\right] &\leq \Expect \left[f(\bar{\x}^{r})\right] - \left(\frac{\gamma\eta I}{2} - \frac{4\gamma^2\eta^2L_g I^2}{M}\right)\Expect \normsq{\nabla f(\bar{\x}^{r})}+ \frac{\gamma^2\eta^2L_gI\sigma^2}{M} + \frac{4\gamma^2\eta^2L_gI^2\zeta^2}{M} \notag \\
&+ \left( \gamma\eta L_h^2 + \frac{4\gamma^2 \eta^2 L_g^3 I}{M}\right)\cdot \frac{1}{N}\sum_{k=0}^{I-1} \sum_{i=1}^N\Expect\normsq{ \x_i^{r,k} - \hat{\x}^{r,k}} \notag \\
&+ \left(\gamma\eta I L_g^2 + \frac{4\gamma^2\eta^2L_g^3 I^2}{M} \right) \notag\\
&\cdot \left( 5(I-1) \cdot\frac{\gamma^2\sigma^2}{N} + 30I\gamma^2 \sum_{k=0}^{I-1} \frac{L_h^2}{N}\sum_{i=1}^N \Expect\normsq{\x_i^{r,k} -\hat{\x}^{r,k}}   +30I(I-1) \gamma^2 \Expect\normsq{\nabla f(\bar{\x}^r)}\right).
\end{align}
By $\gamma\eta \le \frac{M}{16IL_g}$ and $\gamma \le \frac{1}{10\sqrt{3}IL_g}$, we have
\begin{align}
&-\left(\frac{\gamma\eta I}{2} - \frac{4\gamma^2\eta^2L_g I^2}{M}\right)\Expect \normsq{\nabla f(\bar{\x}^{r})} + \left(\gamma\eta I L_g^2 + \frac{4\gamma^2\eta^2L_g^3 I^2}{M} \right)\cdot 30I(I-1) \gamma^2 \Expect\normsq{\nabla f(\bar{\x}^r)} \notag \\
&\le -\left(\frac{\gamma\eta I}{2} - \frac{\gamma\eta I}{4}\right)\Expect \normsq{\nabla f(\bar{\x}^{r})} + \left( \gamma\eta I + \frac{\gamma \eta I}{4}\right) \cdot \frac{1}{10}\cdot \Expect\normsq{\nabla f(\bar{\x}^r)} \notag \\
&\le -\frac{\gamma\eta I}{8}\Expect\normsq{\nabla f(\bar{\x}^r)}.
\end{align}

By $\gamma\eta \le \frac{M}{4IL_g}$, we have
\begin{align}
\left(\gamma\eta I L_g^2 + \frac{4\gamma^2\eta^2L_g^3 I^2}{M} \right)\cdot 5(I-1) \cdot\frac{\gamma^2\sigma^2}{N} \le \gamma\eta I\cdot \frac{10 \gamma^2L_g^2(I-1)\sigma^2}{N}.
\end{align}

Then we have
\begin{align}
\Expect \left[f(\bar{\x}^{r+1})\right] &\leq \Expect \left[f(\bar{\x}^{r})\right] - \frac{\gamma\eta I}{8} \Expect \normsq{\nabla f(\bar{\x}^{r})}+ \frac{\gamma^2\eta^2L_gI\sigma^2}{M} + \frac{4\gamma^2\eta^2L_gI^2\zeta^2}{M} + \gamma\eta I\cdot \frac{10 \gamma^2L_g^2(I-1)\sigma^2}{N} \notag \\
&+ \left( \gamma\eta L_h^2 + \frac{4\gamma^2\eta^2L_g^3 I}{M} + 30\gamma^2IL_h^2\left(\gamma\eta I L_g^2 + \frac{4\gamma^2\eta^2L_g^3 I^2}{M}\right)\right)\cdot \frac{1}{N}\sum_{k=0}^{I-1} \sum_{i=1}^N\Expect\normsq{ \x_i^{r,k} - \hat{\x}^{r,k}}. 
\end{align}

With Lemma~\ref{lemma:model-divergence}, we have
\begin{align}
\Expect \left[f(\bar{\x}^{r+1})\right] &\leq \Expect \left[f(\bar{\x}^{r})\right] - \frac{\gamma\eta I}{8} \Expect \normsq{\nabla f(\bar{\x}^{r})}+ \frac{\gamma^2\eta^2L_gI\sigma^2}{M} + \frac{4\gamma^2\eta^2L_gI^2\zeta^2}{M} + \gamma\eta I\cdot \frac{10 \gamma^2L_g^2(I-1)\sigma^2}{N} \notag \\
&+ \left( \gamma\eta L_h^2 + \frac{4\gamma^2\eta^2L_g^3 I}{M} + 30\gamma^2IL_h^2\left(\gamma\eta I L_g^2 + \frac{4\gamma^2\eta^2L_g^3 I^2}{M}\right)\right)\cdot \left( 12(I-1)^3\gamma^2\zeta^2 + 4(I-1)^2\gamma^2\sigma^2 \right). 
\end{align}

Then we obtain
\begin{align}
\min_{r\in[R]}\Expect \normsq{\nabla f(\bar{\x}^{r})}&\le\frac{1}{R}\sum_{r=0}^{R-1}\Expect \normsq{\nabla f(\bar{\x}^{r})} \le \frac{8( f^0- f^*)}{\gamma\eta I R} + \frac{8\gamma\eta L_g\sigma^2}{M} + \frac{32\gamma\eta L_gI\zeta^2}{M} + \frac{80 \gamma^2L_g^2(I-1)\sigma^2}{N}\notag \\
&+ 8\left( L_h^2 + \frac{4\gamma\eta L_g^3 I}{M} + 30\gamma^2IL_h^2\left( I L_g^2 + \frac{4\gamma \eta L_g^3 I^2}{M}\right)\right)\cdot \left( 12(I-1)^2\gamma^2\zeta^2 + 4(I-1)\gamma^2\sigma^2 \right).
\end{align}

By $\gamma\eta \le \frac{M}{16IL_g}$ and $\gamma \le \frac{1}{10\sqrt{3}IL_g}$, we have
\begin{align}
\frac{32\gamma\eta L_g^3 I}{M}\cdot \left( 12(I-1)^2\gamma^2\zeta^2 + 4(I-1)\gamma^2\sigma^2 \right) \le \frac{\gamma\eta L_g }{M} \cdot \left(6I\zeta^2 + 2\sigma^2 \right),
\end{align}
and
\begin{align}
&30\gamma^2IL_h^2\left( I L_g^2 + \frac{4\gamma \eta L_g^3 I^2}{M}\right)\cdot \left( 12(I-1)^2\gamma^2\zeta^2 + 4(I-1)\gamma^2\sigma^2 \right) \notag \\
&\le \left( \frac{1}{10} + \frac{1}{40} \right)\cdot L_h^2\cdot \left( 12(I-1)^2\gamma^2\zeta^2 + 4(I-1)\gamma^2\sigma^2 \right) \notag\\
&\le \gamma^2L_h^2(I-1)^2\zeta^2 + \gamma^2L_h^2(I-1)\sigma^2.
\end{align}
Finally, we obtain
\begin{align}
\min_{r\in[R]}\Expect \normsq{\nabla f(\bar{\x}^{r})}&\le \frac{8( f^0- f^*)}{\gamma\eta I R} + \frac{10\gamma\eta L_g\sigma^2}{M} + \frac{38\gamma\eta L_gI\zeta^2}{M} + \frac{80 \gamma^2L_g^2(I-1)\sigma^2}{N} \notag \\
&+ 97\gamma^2L_h^2(I-1)^2\zeta^2 + 33\gamma^2L_h^2(I-1)\sigma^2.
\end{align}

}
Rearrange,
\begin{align}\label{eq:tech-novely-partial-participation-2}
\min_{r\in[R]}\Expect \normsq{\nabla f(\bar{\x}^{r})}= \mathcal{O}\left( \frac{ (f^0- f^*)}{\gamma\eta I R} + \frac{\gamma\eta L_g\sigma^2}{M} + \frac{\gamma\eta L_gI\zeta^2}{M} + \frac{\gamma^2L_g^2(I-1)\sigma^2}{N} + \gamma^2L_h^2(I-1)\sigma^2 + \gamma^2L_h^2(I-1)^2\zeta^2 \right).
\end{align}

\subsection{Proof of Proposition~\ref{lemma:local-assump}}\label{sec:proof-prop5.1}
First, using $\nabla f(\x) = \frac{1}{N}\sum_{i=1}^N\nabla F_i(\x)$, it is straightforward to show that Assumption~\ref{assumption:local-lipschitz-gradient} implies Assumption~\ref{assumption:global-lipschitz-gradient} holds by choosing $L_g=\tilde{L}$.

Second, we can see that 
\begin{align}
\normsq{\frac{1}{N}\sum_{i=1}^N \nabla F_i(\x_i) - \nabla f\left(\bar{\x}\right)}
= &\normsq{\frac{1}{N}\sum_{i=1}^N \left[ \nabla F_i(\x_i) - \nabla F_i\left(\bar{\x}\right)\right]} \notag \\
\le& \frac{1}{N}\sum_{i=1}^N \normsq{\nabla F_i(\x_i) - \nabla F_i\left(\bar{\x}\right)} \notag\\
\overset{(a)}{\le}& \frac{\tilde{L}^2}{N}\sum_{i=1}^N \normsq{\x_i-\bar{\x}},\label{eq:proof-lemma}
\end{align}
where $(a)$ is due to Assumption~\ref{assumption:local-lipschitz-gradient}.
By choosing $L_h=\tilde{L}$, Assumption~\ref{assumption:gradient-to-model} holds. 
\qed

\subsection{Proof of Proposition~\ref{pro:explain-D}}\label{sec:proof-prop5.2}

Recall that $\hat{\x}^{r,k}$ is the virtual averaged model defined in (\ref{eq:virtual-averaged-model}) in the main paper.
During one local iteration, we have
\begin{align}\label{eq:global-virtual}
	\Expect[\hat{\x}^{r,k+1}|\hat{\x}^{r,k}] = \hat{\x}^{r,k}- \gamma\cdot \frac{1}{N}\sum_{i=1}^N \nabla F_i(\x_i^{r,k}). 
\end{align}
Using (\ref{eq:centralized-update}), if we use centralized update at this iteration, we have
\begin{align}\label{eq:centralized}
	\Expect[\x_c^{r,k+1}|\hat{\x}^{r,k}] = \hat{\x}^{r,k} - \gamma \nabla f(\hat{\x}^{r,k}).
\end{align}
Using Assumption~\ref{assumption:gradient-to-model}, 
we obtain 
\begin{align}\label{eq:explanation-assumption}
\normsq{\Expect[\hat{\x}^{r,k+1}|\hat{\x}^{r,k}]-\Expect[\x_c^{r,k+1}|\hat{\x}^{r,k}]} 
&= \gamma^2 \normsq{\frac{1}{N}\sum_{i=1}^N \nabla F_i(\x_i^{r,k}) - \nabla f\left(\hat{\x}^{r,k}\right)} \notag\\
&\le \gamma^2\cdot\frac{L_h^2}{N}\sum_{i=1}^N \normsq{\x_i^{r,k} -\hat{\x}^{r,k}}.
\end{align}

\subsection{Proof of Proposition~\ref{lemma:DQuadratic}}
For quadratic functions, we have
\begin{align}
    \nabla F_i(\x) = \A_i\x + \bb_i, \x\in \mathbb{R}^d. 
\end{align}
{  Recall that $\A := \frac{1}{N}\sum_{i=1}^N \A_i$ and $\bb := \frac{1}{N}\sum_{i=1}^N \bb_i$. }
We have
\begin{align}
    &\normsq{\frac{1}{N}\sum_{i=1}^N \nabla F_i(\x_i) - \nabla f\left(\bar{\x}\right)}\notag \\
    &=\normsq{\frac{1}{N}\sum_{i=1}^N \left(\A_i\x_i + \bb_i\right) - \left(\A\bar{\x} + \bb\right)}\notag \\
    &=\normsq{\frac{1}{N}\sum_{i=1}^N \A_i\x_i - \A\bar{\x}}\notag \\
    &={\normsq{\frac{1}{N}\sum_{i=1}^N \A_i\x_i - 2\A\bar{\x} + \A\bar{\x}}}\notag \\
    &={\normsq{\frac{1}{N}\sum_{i=1}^N \A_i\x_i - \frac{1}{N}\sum_{i=1}^N\A_i\bar{\x} - \frac{1}{N}\sum_{i=1}^N\A\x_i + \A\bar{\x}}}\notag \\
    &= {  \normsq{\frac{1}{N}\sum_{i=1}^N\A_i(\x_i-\bar{\x}) - \frac{1}{N}\sum_{i=1}^N \A(\x_i - \bar{\x}) }} \notag \\
    &={\normsq{\frac{1}{N}\sum_{i=1}^N \left(\A_i - \A\right)(\x_i - \bar{\x})}}\notag \\
    &\leq{\frac{1}{N}\sum_{i=1}^N\normsq{ \left(\A_i - \A\right)(\x_i - \bar{\x})}}\notag \\
    &\leq{\frac{|\lambda_{\mathrm{diff}}|^2_\mathrm{max}}{N}\sum_{i=1}^N\normsq{ \x_i - \bar{\x}}}.
\end{align}  

{ 
For local Lipschitz gradient, we have
\begin{align}
    &\norm{\nabla F_i(\x) - \nabla F_i(\y)}\notag \\
    &= \norm{\A_i\x - \A_i\y} \notag \\
    &\le \norm{\A_i} \norm{\x-\y}.  
\end{align}
Since local Lipschitz gradient holds for each worker $i\in [N]$, we can choose $\tilde{L}$ as
\begin{align}
    \tilde{L} = \max_i \norm{\A_i}.
\end{align}
}

\subsection{Proof of Theorem~\ref{thm:quadratic}}
It %
can be observed that for quadratic objective functions when $\A_i = \A, \forall i$, we have $L_h=0$ and $L_g=|\lambda(\A)|$. 

With Assumption~\ref{assumption:global-lipschitz-gradient}, after one local iteration, we have
\begin{align}\label{eq:lipschitz-one-step}
    \Expect\left[f(\hat{\x}^{t+1})\right]  &\leq \Expect\left[ f(\hat{\x}^{t})\right] - \gamma \Expect{\innerprod{\nabla f(\hat{\x}^{t}), \frac{1}{N}\sum_{i=1}^N \g_i(\x_i^t)}} + \frac{\gamma^2 L_g}{2}\Expect{\normsq{\frac{1}{N}\sum_{i=1}^N \g_i(\x_i^t)}} \notag\\
    &= \Expect\left[ f(\hat{\x}^{t})\right] - \gamma \Expect \innerprod{\nabla f(\hat{\x}^{t}), \frac{1}{N}\sum_{i=1}^N \nabla F_i(\x_i^t)} + \frac{\gamma^2 L_g}{2}\Expect{\normsq{\frac{1}{N}\sum_{i=1}^N \g_i(\x_i^t)}}.
\end{align}
For the second term in the RHS of (\ref{eq:lipschitz-one-step}), we have
\begin{align}\label{eq:inner-product-one-step}
    &- \gamma\Expect \innerprod{\nabla f(\hat{\x}^{t}), \frac{1}{N}\sum_{i=1}^N \nabla F_i(\x_i^t)} \notag\\
    &= \frac{\gamma}{2}\left(\Expect\normsq{\frac{1}{N}\sum_{i=1}^N \nabla F_i(\x_i^t) - \nabla f(\hat{\x}^{t})} - \Expect \normsq{\nabla f(\hat{\x}^{t})} - \Expect\normsq{\frac{1}{N}\sum_{i=1}^N \nabla F_i(\x_i^t) } \right)\notag\\
    &\leq \frac{\gamma}{2}\left(\frac{L_h^2}{N}\sum_{i=1}^N \Expect\normsq{\x_i^t -\hat{\x}^t} - \Expect\normsq{\nabla f(\hat{\x}^{t})} - \Expect\normsq{\frac{1}{N}\sum_{i=1}^N \nabla F_i(\x_i^t) } \right).
\end{align}
For the third term of (\ref{eq:lipschitz-one-step}), we have 
{ 
\begin{align}\label{eq:noise-one-step}
    &\frac{\gamma^2 L_g}{2}\Expect{\normsq{\frac{1}{N}\sum_{i=1}^N \g_i(\x_i^t)}} \notag \\
    &=\frac{\gamma^2 L_g}{2}\E\left[\E_{\x_i^t}\normsq{\frac{1}{N}\sum_{i=1}^N \left(\g_i(\x_i^t) - \nabla F_i(\x_i^t) + \nabla F_i(\x_i^t)\right)}\right]   \notag \\
    &\overset{(a)}{=}  \frac{\gamma^2 L_g}{2} \E\normsq{\frac{1}{N}\sum_{i=1}^N \nabla F_i(\x_i^t)} + \frac{\gamma^2 L_g}{2}\E\left[\E_{\x_i^t}\normsq{\frac{1}{N}\sum_{i=1}^N \left(\g_i(\x_i^t) - \nabla F_i(\x_i^t) \right)}\right]  \notag \\
    &\leq \frac{\gamma^2 L_g}{2}\Expect\normsq{\frac{1}{N}\sum_{i=1}^N \nabla F_i(\x_i^t)} + \frac{\gamma^2 L_g \sigma^2}{2N},
\end{align}
where $\E_{\x_i^t}[\cdot] = \E[\cdot|\x_i^t]$ and $(a)$ is due to that $\E_{\x_i^t}\left\langle \nabla F_i(\x_i^t), \g_i(\x_i^t) - \nabla F_i(\x_i^t)\right\rangle =0$.
}

Substitute (\ref{eq:inner-product-one-step}) and (\ref{eq:noise-one-step}) back to (\ref{eq:lipschitz-one-step}), we obtain
\begin{align}\label{eq:quadratic-one-step}
    &\Expect\left[f(\hat{\x}^{t+1})\right]\notag \\
    &\leq \Expect\left[ f(\hat{\x}^{t})\right] + \frac{\gamma L_h^2}{2N}\sum_{i=1}^N \Expect \normsq{\x_i^t -\hat{\x}^t} - \frac{\gamma}{2}\Expect \normsq{\nabla f(\hat{\x}^{t})} \notag \\
    &- \left(\frac{\gamma}{2} - \frac{\gamma^2 L_g}{2}\right)\Expect \normsq{\frac{1}{N}\sum_{i=1}^N \nabla F_i(\x_i^t) }   + \frac{\gamma^2 L_g \sigma^2}{2N} \notag\\
    &\overset{(a)}{\leq} \Expect\left[ f(\hat{\x}^{t})\right] + \frac{\gamma L_h^2}{2N}\sum_{i=1}^N \Expect \normsq{\x_i^t -\hat{\x}^t} - \frac{\gamma}{2}\Expect \normsq{\nabla f(\hat{\x}^{t})}  + \frac{\gamma^2 L_g \sigma^2}{2N},
\end{align}
where $(a)$ is due to $\gamma < \frac{1}{L_g}$. Rearrange the above inequality with $L_h=0$, we have 
\begin{align}
\Expect \normsq{\nabla f(\hat{\x}^{t})} &\le \frac{2\Expect\left[ f(\hat{\x}^{t})\right]- 2\Expect f(\hat{\x}^{t+1})}{\gamma} + \frac{ L_h^2}{N}\sum_{i=1}^N \Expect \normsq{\x_i^t -\hat{\x}^t} + \frac{\gamma L_g \sigma^2}{N} \notag \\
&= \frac{2\Expect\left[ f(\hat{\x}^{t})\right]- 2\Expect f(\hat{\x}^{t+1})}{\gamma}+ \frac{\gamma L_g \sigma^2}{N}.
\end{align}
Take the average over $t$ on both sides, we obtain
\begin{align}
\min_{t\in{[T]}}\Expect \normsq{\nabla f(\hat{\x}^{t})} \le \frac{1}{T}\sum_{t=0}^{T-1} \Expect \normsq{\nabla f(\hat{\x}^{t})} \le \frac{2f(\hat{\x}^{t})- 2f^*}{\gamma T} + \frac{\gamma L_g \sigma^2}{N}.
\end{align}

\subsection{Proof of Corollary~\ref{cor:local-better-minibatch}}

{ 
In Corollary 5.6, for both local SGD and mini-batch SGD, we choose the learning rate as $\gamma = \frac{1}{L_g}$. 
The proof (order-wise, ignoring the constants) is as follows.

Let $h(\gamma)$ denote the order of the convergence upper bound in Theorem 5.5. 
For mini-batch SGD, we have 
\begin{equation}
\label{eq: h}
    h(\gamma) = \frac{\mathcal{F}}{\gamma R} + \frac{\gamma L_g\sigma^2}{NI}. 
\end{equation}
By minimizing $h(\gamma)$, we obtain $\gamma^* = \sqrt{\frac{\mathcal{F}NI}{RL_g\sigma^2}}$. Because $\sigma \le \sqrt{\frac{\mathcal{F}NL_g}{RI}}$ as specified in Corollary 5.6, we have
$$
\gamma^* \geq \sqrt{\frac{\mathcal{F}NI}{RL_g} \cdot \frac{RI}{\mathcal{F}NL_g}} = \frac{I}{L_g}.
$$
Therefore, when $\gamma \in \left(0, \frac{I}{L_g}\right]$, $h(\gamma)$ is monotonically decreasing.
Now, note that Theorem 5.5 requires $\gamma \in \Big(0,\frac{1}{L_g}\Big]$. When choosing $\gamma = \frac{1}{L_g} \leq \frac{I}{L_g}$ (since $I\geq 1$), $h(\gamma)$ is minimized under the condition of $\gamma \in \Big(0,\frac{1}{L_g}\Big]$ and we obtain
$$
h\left(\gamma\right) = h\left(\frac{1}{L_g}\right) = \frac{\mathcal{F}L_g}{R} + \frac{1}{L_g} \frac{L_g\sigma^2}{NI} \overset{(a)}{\le} \frac{\mathcal{F}L_g}{R} + \frac{\mathcal{F}L_g}{RI^2},
$$
where $(a)$ is due to $\sigma \le \sqrt{\frac{\mathcal{F}NL_g}{RI}}$. The case of local SGD can be proven similarly.
}

\subsection{Proof of Theorem~\ref{thm:quad-localvsmini}}\label{sec:proof-quad-vs}
{  First, we introduce a useful lemma, which is used in this section.
\begin{lemma}\label{lemma:geometric-square}
With $x> 1$ and $k\in \mathbb{N}^+$, we have
\begin{align}
    \sum_{l=0}^{k-1} l^2x^l \le \frac{x^{k-1}}{(x-1)^3} \cdot k^2x^2 \cdot x = \frac{k^2x^{k+2}}{(x-1)^3}.
\end{align}
\end{lemma}

\begin{proof}
For the geometric series, we have
\begin{align}
    \sum_{l=0}^{k-1} x^l = \frac{x^k-1}{x-1}.
\end{align}
When $k\ge 2$, taking the derivative over $x$ on both sides, we obtain
\begin{align}
    \sum_{l=0}^{k-1} lx^{l-1} = \frac{kx^{k-1}}{x-1} - \frac{x^k-1}{(x-1)^2}.
\end{align}
Multiplying $x$ on both sides, we obtain
\begin{align}
    \sum_{l=0}^{k-1} lx^l = \frac{kx^k}{x-1} - \frac{x^{k+1}-x}{(x-1)^2}.
\end{align}
Taking the derivative over $x$ on both sides again, we obtain
\begin{align}
    \sum_{l=0}^{k-1} l^2x^{l-1} &= \frac{k^2x^{k-1}}{x-1}-\frac{kx^k}{(x-1)^2} - \frac{(k+1)x^k-1}{(x-1)^2}+\frac{2x^{k+1}-2x}{(x-1)^3}\notag \\
    &= \frac{k^2x^{k-1}(x-1)^2 - (2k+1)x^k(x-1)+(x-1)+2x^{k+1}-2x}{(x-1)^3} \notag \\
    &\le \frac{k^2x^{k-1}(x-1)^2 - (2k+1)x^k(x-1)+2x^{k+1}}{(x-1)^3} \notag \\
    &= \frac{x^{k-1}}{(x-1)^3}\left[(k^2-2k+1)x^2-2k^2x+(2k+1)x+k^2 \right] \notag\\
    &\overset{(a)}{\le} \frac{x^{k-1}}{(x-1)^3}\left[(k^2-2k+1)x^2-k^2x+(2k+1)x \right] \notag \\
    &\overset{(b)}{\le}  \frac{x^{k-1}}{(x-1)^3} \cdot k^2x^2
\end{align}
where $(a)$ is due to $x> 1$ and $(b)$ is due to $k\ge 2$.

When $k=1$, we have
\begin{align}
\sum_{l=0}^{k-1} l^2x^{l-1} = 0 \le \frac{x^{k-1}}{(x-1)^3} \cdot k^2x^2.
\end{align}

Then we have
\begin{align}
\sum_{l=0}^{k-1} l^2x^l \le  \frac{x^{k-1}}{(x-1)^3} \cdot k^2x^2 \cdot x = \frac{k^2x^{k+2}}{(x-1)^3}.
\end{align}
\end{proof}

}

For quadratic objective functions, the global objective functions is 
\begin{align}
    f(\x) = \frac{1}{2}\x^T\A\x + \bb^T\x + c.
\end{align}
The local objective function of worker $i$ is 
\begin{align}
    F_i(\x) = \frac{1}{2}\x^T\A_i\x + \bb_i^T\x + c_i,
\end{align}
where $\A = \frac{1}{N}\sum_{i=1}^N \A_i$, $\bb = \frac{1}{N}\sum_{i=1}^N \bb_i$ and $c = \frac{1}{N}\sum_{i=1}^N c_i$.
The local stochastic gradient is 
\begin{align}
\g_i(\x) = \A_i\x + \bb_i + \nn_i,
\end{align}
{ 
where $\nn_i \in \R^d$ is the noise vector. 
Since we assume the stochastic gradient is unbiased, we have
\begin{align}
    \E[\g_i(\x)] = \A_i\x + \bb_i + \E[\nn_i] = \nabla F_i(\x) = \A_i\x + \bb_i, \forall \x.
\end{align}
Therefore, we can get $\E[\nn_i] = \mathbf{0}$.
By Assumption~\ref{assumption:bounded-stochastic-variance}, we have
\begin{align}
\E\normsq{\g_i(\x) -\nabla F_i(\x) } = \E\normsq{\nn_i} \le \sigma^2.
\end{align}
}

By Assumption~\ref{assumption:weaker-gradient-divergence}, we have
\begin{align}
\normsq{\nabla F_i(\x) - \nabla f(\x)} = \normsq{(\A_i-\A)\x + (\bb_i-\bb) } \le \zeta_q^2, \forall i.
\end{align}

By Proposition~\ref{lemma:DQuadratic}, $L_g = |\lambda(\A)|$ and $L_h = 2\max_i |\lambda(\A-\A_i)| $. In the following, we define $\lambda_{\max} := \max_i \left\| \A_i\right\|_2$.

{  In the following, we $\nn_i^{r,k}$ to denote the noise vector on worker $i$ at $k$th iteration of $r$th round.
During local updates, we have 
\begin{align}
    &\x_i^{r,k} = \x_i^{r,k-1} - \gamma \g_i(\x_i^{r,k-1}) \notag \\
    &= \x_i^{r,k-1} - \gamma(\A_i\x_i^{r,k-1} + \bb_i + \nn_i^{r,k-1}) \notag \\
    &= (\mI-\gamma\A_i)\x_i^{r,k-1} - \gamma\left(\bb_i +\nn_i^{r,k-1}\right) \notag \\
    &=(\mI-\gamma\A_i)\left[ (\mI-\gamma\A_i)\x_i^{r,k-2} - \gamma\left(\bb_i +\nn_i^{r,k-2}\right) \right] - \gamma\left(\bb_i +\nn_i^{r,k-1}\right)  \notag \\
    &= (\mI-\gamma\A_i)^2\x_i^{r,k-2} - \gamma(\mI-\gamma\A_i)\left(\bb_i +\nn_i^{r,k-2}\right) - \gamma\left(\bb_i +\nn_i^{r,k-1}\right)\notag \\
    &= \ldots \notag \\
    &= (\mI-\gamma \A_i)^k \bar{\x}^r - \gamma \sum_{l=0}^{k-1} (\mI-\gamma \A_i)^l (\bb_i+\nn_i^{r,k-1-l})\notag \\
    &= (\mI-\gamma \A_i)^k \bar{\x}^r - \gamma\sum_{l=0}^{k-1} (\mI-\gamma \A_i)^l\bb_i - \gamma\sum_{l=0}^{k-1} (\mI-\gamma \A_i)^l\nn_i^{r,k-1-l}\notag \\
    &\overset{(a)}{=} \left[\mI - \gamma \sum_{l=0}^{k-1} (\mI - \gamma \A_i)^l\A_i\right]\bar{\x}^r  - \gamma\sum_{l=0}^{k-1} (\mI-\gamma \A_i)^l\bb_i - \gamma\sum_{l=0}^{k-1} (\mI-\gamma \A_i)^l\nn_i^{r,k-1-l} \notag \\
    &= \bar{\x}^r - \gamma \sum_{l=0}^{k-1} (\mI-\gamma \A_i)^l \left[\A_i\bar{\x}^r + \bb_i \right] - \gamma\sum_{l=0}^{k-1} (\mI-\gamma \A_i)^l\nn_i^{r,l-1-l}\notag \\
    &= \bar{\x}^r - \gamma \sum_{l=0}^{k-1} (\mI-\gamma \A_i)^l  \nabla F_i(\bar{\x}^r) - \gamma\sum_{l=0}^{k-1} (\mI-\gamma \A_i)^l\nn_i^{r,k-1-l}.
\end{align}
Now we explain $(a)$. 
For the sum of geometric series~\citep{hubbard2015vector} of matrix $\mI-\gamma \A_i$, we have
\begin{align}
\sum_{l=0}^{k-1} \left(\mI-\gamma \A_i \right)^l = \left[\mI-\left(\mI-\gamma \A_i \right)  \right]^{-1} \left[\mI -\left(\mI-\gamma \A_i \right)^k  \right] = \frac{1}{\gamma}\A_i^{-1}\left[\mI -\left(\mI-\gamma \A_i \right)^k  \right].
\end{align}
Since $\A_i$ is symmetric, $\A_i^{-1}$ and $\mI -\left(\mI-\gamma \A_i \right)^k $ are also symmetric. Thus, we have
\begin{align}
\frac{1}{\gamma}\A_i^{-1}\left[\mI -\left(\mI-\gamma \A_i \right)^k  \right] = \frac{1}{\gamma}\left[\mI -\left(\mI-\gamma \A_i \right)^k  \right]\A_i^{-1} =\sum_{l=0}^{k-1} \left(\mI-\gamma \A_i \right)^l.
\end{align}
Multiplying $\A_i$ on both sides and rearranging, we obtain
\begin{align}
\left(\mI-\gamma \A_i \right)^k = \mI - \gamma \sum_{l=0}^{k-1} \left(\mI-\gamma \A_i \right)^l\A_i.
\end{align}

}

Then for the model divergence, we have 
{ 
\begin{align}\label{eq:quad-model-divergence}
&\E\normsq{\x_i^{r,k} - \hat{\x}^{r,k}} \notag \\
&= \E\bigg\|\gamma \sum_{l=0}^{k-1} (\mI-\gamma \A_i)^l  \nabla F_i(\bar{\x}^r) + \gamma\sum_{l=0}^{k-1} (\mI-\gamma \A_i)^l\nn_i^{r,k-1-l}- \gamma\cdot \frac{1}{N}\sum_{j=1}^N\sum_{l=0}^{k-1} (\mI-\gamma \A_j)^l  \nabla F_j(\bar{\x}^r)   \notag \\
&-\gamma \cdot \frac{1}{N}\sum_{j=1}^N\sum_{l=0}^{k-1} (\mI-\gamma \A_j)^l\nn_j^{r,k-1-l} \bigg\|^2\notag \\
&=\E\normsq{\gamma\sum_{l=0}^{k-1} (\mI-\gamma \A_i)^l  \nabla F_i(\bar{\x}^r)  - \gamma\cdot \frac{1}{N}\sum_{j=1}^N\sum_{l=0}^{k-1} (\mI-\gamma \A_j)^l  \nabla F_j(\bar{\x}^r)  } \notag \\
&+ \E \normsq{\gamma\sum_{l=0}^{k-1} (\mI-\gamma \A_i)^l\nn_i^{r,l} - \gamma \cdot \frac{1}{N}\sum_{j=1}^N\sum_{l=0}^{k-1} (\mI-\gamma \A_j)^l\nn_j^{r,k-1-l}} \notag \\
&+\E\bigg[  \E_{\bar{\x}^r}\!\bigg\langle  \gamma\sum_{l=0}^{k-1} (\mI-\gamma \A_i)^l  \nabla F_i(\bar{\x}^r)  - \gamma\cdot \frac{1}{N}\sum_{j=1}^N\sum_{l=0}^{k-1} (\mI-\gamma \A_j)^l  \nabla F_j(\bar{\x}^r), \notag \\
&\quad\quad\gamma\sum_{l=0}^{k-1} (\mI-\gamma \A_i)^l\nn_i^{r,l} - \gamma \cdot \frac{1}{N}\sum_{j=1}^N\sum_{l=0}^{k-1} (\mI-\gamma \A_j)^l\nn_j^{r,k-1-l}\bigg\rangle\bigg] \notag \\
&\overset{(a)}{=} \E\normsq{\gamma\sum_{l=0}^{k-1} (\mI-\gamma \A_i)^l  \nabla F_i(\bar{\x}^r)  - \gamma\cdot \frac{1}{N}\sum_{j=1}^N\sum_{l=0}^{k-1} (\mI-\gamma \A_j)^l  \nabla F_j(\bar{\x}^r)  }\notag \\
&+ \E \normsq{\gamma\sum_{l=0}^{k-1} (\mI-\gamma \A_i)^l\nn_i^{r,l} - \gamma \cdot \frac{1}{N}\sum_{j=1}^N\sum_{l=0}^{k-1} (\mI-\gamma \A_j)^l\nn_j^{r,k-1-l}},
\end{align}
where $(a)$ is due to $\E_{\bar{\x}^r}[\nn_i^{r,l}]=0$.
}

For the first term in the RHS of (\ref{eq:quad-model-divergence}), 
we have 
\begin{align}\label{eq:quad-model-divergence1}
&\E\normsq{\gamma\sum_{l=0}^{k-1} (\mI-\gamma \A_i)^l  \nabla F_i(\bar{\x}^r)  - \gamma\cdot \frac{1}{N}\sum_{j=1}^N\sum_{l=0}^{k-1} (\mI-\gamma \A_j)^l  \nabla F_j(\bar{\x}^r) } \notag \\
&= \E\bigg\|\gamma\sum_{l=0}^{k-1} (\mI-\gamma \A_i)^l  \nabla F_i(\bar{\x}^r) -  \gamma\sum_{l=0}^{k-1} (\mI-\gamma \A_i)^l  \nabla f(\bar{\x}^r) + \gamma\sum_{l=0}^{k-1} (\mI-\gamma \A_i)^l  \nabla f(\bar{\x}^r)\notag\\
&-\gamma\sum_{l=0}^{k-1} (\mI-\gamma \A)^l  \nabla f(\bar{\x}^r) + \gamma\sum_{l=0}^{k-1} (\mI-\gamma \A)^l  \nabla f(\bar{\x}^r)- \gamma  \cdot \frac{1}{N}\sum_{j=1}^N\sum_{l=0}^{k-1} (\mI-\gamma \A_j)^l  \nabla f(\bar{\x}^r) \notag \\
&+ \gamma  \cdot \frac{1}{N}\sum_{j=1}^N\sum_{l=0}^{k-1} (\mI-\gamma \A_j)^l  \nabla f(\bar{\x}^r) 
- \gamma\cdot \frac{1}{N}\sum_{j=1}^N\sum_{l=0}^{k-1} (\mI-\gamma \A_j)^l  \nabla F_j(\bar{\x}^r)  \bigg\|^2\notag \\
&\le 4\E\normsq{\gamma\sum_{l=0}^{k-1} (\mI-\gamma \A_i)^l  \nabla F_i(\bar{\x}^r) -  \gamma\sum_{l=0}^{k-1} (\mI-\gamma \A_i)^l  \nabla f(\bar{\x}^r)} \notag \\
&+ 4 \E\normsq{\gamma\sum_{l=0}^{k-1} (\mI-\gamma \A_i)^l  \nabla f(\bar{\x}^r)-\gamma\sum_{l=0}^{k-1} (\mI-\gamma \A)^l  \nabla f(\bar{\x}^r)} \notag \\
&+ 4 \cdot \frac{1}{N}\sum_{j=1}^N \E\normsq{\gamma\sum_{l=0}^{k-1} (\mI-\gamma \A)^l  \nabla f(\bar{\x}^r)- \gamma\sum_{l=0}^{k-1} (\mI-\gamma \A_j)^l  \nabla f(\bar{\x}^r)}\notag\\
& + 4 \cdot \frac{1}{N}\sum_{j=1}^N \E\normsq{\gamma  \sum_{l=0}^{k-1} (\mI-\gamma \A_j)^l  \nabla f(\bar{\x}^r) 
- \gamma\sum_{l=0}^{k-1} (\mI-\gamma \A_j)^l  \nabla F_j(\bar{\x}^r) }.
\end{align}

For the first term in RHS of (\ref{eq:quad-model-divergence1}), we have 
{ 
\begin{align}\label{eq:quad-grad-diff}
&\E\normsq{\gamma\sum_{l=0}^{k-1} (\mI-\gamma \A_i)^l  \nabla F_i(\bar{\x}^r) -  \gamma\sum_{l=0}^{k-1} (\mI-\gamma \A_i)^l  \nabla f(\bar{\x}^r)}\notag \\
&\le \gamma^2k  \sum_{l=0}^{k-1} \normsq{(\mI-\gamma \A_i)^l}\E\normsq{\nabla F_i(\bar{\x}^r)-\nabla f(\bar{\x}^r)} \notag \\
&\overset{(a)}{\le} \gamma^2k \cdot\phi(\kappa,k)\cdot\zeta_q^2,
\end{align}
where
\begin{align}
\phi(\kappa,k) =
\begin{cases}
      k & 0\le \kappa < 1\\
      \frac{\kappa^{2k}-1}{\kappa^2-1} & 1\le \kappa \le 2.
    \end{cases}
\end{align}

Now we explain $(a)$. We can rewrite $\norm{(\mI-\gamma \A_i)^l}$ as 
\begin{align}
\norm{(\mI-\gamma \A_i)^l} = \left[\max_j 1- \gamma \lambda_j(\A_i)\right]^l.
\end{align}

When $\A_i$ is positive definite, which means $\lambda_j(\A_i)>0, \forall i,j$, since $0<\gamma\le \frac{1}{\lambda_{\max}}$, we have
\begin{align}
    1>\max_j 1- \gamma \lambda_j(\A_i) >\max_j 1- \frac{\lambda_j(\A_i)}{\lambda_{\max}} >0.
\end{align}
Then we have
\begin{align}
\sum_{l=0}^{k-1}\normsq{(\mI-\gamma \A_i)^l} \le k.
\end{align}
In this case, we also have
\begin{align}
\kappa=\max_{i,j} 1 - \frac{\lambda_j(\mA_i)}{\left\| \A_i\right\|_2} < 1.
\end{align}

When $\A_i$ is not positive definite, which means that $\exists j\in [d]$, such that $\lambda_j(\A_i)\le 0$, we have
\begin{align}
    1 \le \max_j 1- \gamma \lambda_j(\A_i) \le \max_j 1- \frac{\lambda_j(\A_i)}{\lambda_{\max}}  \le \max_{i,j} 1 - \frac{\lambda_j(\mA_i)}{\left\| \A_i\right\|_2} = \kappa.
\end{align}

Then we have
\begin{align}
\sum_{l=0}^{k-1}\normsq{(\mI-\gamma \A_i)^l} \le \sum_{l=0}^{k-1} \kappa^{2l} = \frac{\kappa^{2k}-1}{\kappa^2-1}.
\end{align}

}

For the second term in RHS of (\ref{eq:quad-model-divergence1}), we have 
\begin{align}\label{eq:quad-hessian-diff}
& \E\normsq{\gamma\sum_{l=0}^{k-1} (\mI-\gamma \A_i)^l  \nabla f(\bar{\x}^r)-\gamma\sum_{l=0}^{k-1} (\mI-\gamma \A)^l  \nabla f(\bar{\x}^r)}\notag \\
&\le \gamma^2 k \sum_{l=0}^{k-1} \normsq{(\mI-\gamma \A_i)^l - (\mI-\gamma \A)^l}\E\normsq{\nabla f(\bar{\x}^r)}\notag\\
&\overset{(a)}{\le} \gamma^4L_h^2k\sum_{l=0}^{k-1} l^2[\varphi(\kappa)]^{2l} \E\normsq{\nabla f(\bar{\x}^r)} \notag \\
&\overset{(b)}{\le} \gamma^4L_h^2 \cdot  \frac{k^3[\varphi(\kappa)]^{2(k+2)}}{([\varphi(\kappa)]^2-1)^3}\cdot  \E\normsq{\nabla f(\bar{\x}^r)},
\end{align}
{ where
\begin{align}
\varphi(\kappa)=
\begin{cases}
1 & 0\le \kappa <1\\
\kappa & 1\le \kappa \le 2,
\end{cases}
\end{align}
$(a)$ is due to 
\begin{align}
\normsq{(\mI-\gamma \A_i)^l - (\mI-\gamma \A)^l}\overset{(c)}{\le} l^2[\varphi(\kappa)]^{2l}\normsq{\mI-\gamma \A_i -\mI+\gamma \A }= \gamma^2l^2[\varphi(\kappa)]^{2l} \normsq{\A_i - \A} \le \gamma^2 L_h^2l^2[\varphi(\kappa)]^{2l},
\end{align}
and
$(b)$ is due to Lemma~\ref{lemma:geometric-square} by letting $x=[\varphi(\kappa)]^2$.
}

Now we prove $(c)$. { Let $\B_i=\mI-\gamma \A_i$ and $\B = \mI-\gamma \A$. Then we have 
\begin{align}
\|\B_i^l \| \le 
\begin{cases}
1 & 0\le \kappa <1\\
\kappa^l & 1\le \kappa \le 2.
\end{cases}
\end{align}
Thus, we get $\|\B_i^l \|\le [\varphi(\kappa)]^l$.

}

{ 
Then we have 
\begin{align}
&\|\B_i^l -\B^l\| = \|(\B_i -\B)(\B_i^{l-1} + \B_i^{l-2}\B+\B_i^{l-3}\B^2+\ldots+ \B^{l-1}) \|\notag \\
&\le \| \B_i -\B\| \| \B_i^{l-1} + \B_i^{l-2}\B+\B_i^{l-3}\B^2+\ldots+ \B^{l-1}\|\notag \\
&\le  \| \B_i -\B\| (\| \B_i^{l-1}\|+ \|\B_i^{l-2}\|\|\B \|+ \| \B_i^{l-3}\|\|\B^2\|+\ldots +\|\B^{l-1} \|)\notag \\
&\le l[\varphi(\kappa)]^{l-1} \| \B_i -\B\|.
\end{align}

}

{ 
Taking (\ref{eq:quad-grad-diff}) and (\ref{eq:quad-hessian-diff}) back to (\ref{eq:quad-model-divergence1}), we can obtain
\begin{align}\label{eq:diff-hession-grad}
    &\E\normsq{\gamma\sum_{l=0}^{k-1} (\mI-\gamma \A_i)^l  \nabla F_i(\bar{\x}^r)  - \gamma\cdot \frac{1}{N}\sum_{j=1}^N\sum_{l=0}^{k-1} (\mI-\gamma \A_j)^l  \nabla F_j(\bar{\x}^r) } \notag \\
    &\le  8\gamma^2k \cdot\phi(\kappa,k)\cdot\zeta_q^2 + 8\gamma^4L_h^2 \cdot  \frac{k^3[\varphi(\kappa)]^{2(k+2)}}{([\varphi(\kappa)]^2-1)^3}\cdot  \E\normsq{\nabla f(\bar{\x}^r)}.
\end{align}
}

For the second term in RHS of (\ref{eq:quad-model-divergence}), we have 
\begin{align}\label{eq:diff-hessian-noise}
    &\E \normsq{\gamma\sum_{l=0}^{k-1} (\mI-\gamma \A_i)^l\nn_i^{r,k-1-l} - \gamma \cdot \frac{1}{N}\sum_{j=1}^N\sum_{l=0}^{k-1} (\mI-\gamma \A_j)^l\xi_j^{r,k-1-l}}\notag \\
    &= \E \normsq{\gamma\sum_{l=0}^{k-1} (\mI-\gamma \A_i)^l\nn_i^{r,l}} + \E\normsq{\gamma \cdot \frac{1}{N}\sum_{j=1}^N\sum_{l=0}^{k-1} (\mI-\gamma \A_j)^l\xi_j^{r,k-1-l}}\notag \\
    &\le \gamma^2 \sum_{l=0}^{k-1} \normsq{(\mI-\gamma \A_i)^l}\sigma^2 + \gamma^2 \sum_{l=0}^{k-1} \normsq{ (\mI-\gamma \A_j)^l}\sigma^2\notag \\
    &\overset{(a)}{\le} 2\gamma^2\cdot \phi(\kappa,k)\cdot \sigma^2,
\end{align}
where $(a)$ is due to $\sum_{l=0}^{k-1}\normsq{\mI-\gamma \A_i}\le \phi(\kappa,k)$.

{  Substituting (\ref{eq:quad-model-divergence}), (\ref{eq:diff-hession-grad}) and (\ref{eq:diff-hessian-noise}) back to (\ref{eq:quadratic-one-step}), we have
}
\begin{align}\label{eq:quad-final-func}
    &\Expect\left[f(\hat{\x}^{t+1})\right]\notag \\
    &\leq \Expect\left[ f(\hat{\x}^{t})\right] + \frac{\gamma L_h^2}{2N}\sum_{i=1}^N \Expect \normsq{\x_i^t -\hat{\x}^t} - \frac{\gamma}{2}\Expect \normsq{\nabla f(\hat{\x}^{t})} 
    - \left(\frac{\gamma}{2}- \frac{\gamma^2 L_g}{2}\right)\Expect \normsq{\frac{1}{N}\sum_{i=1}^N \nabla F_i(\x_i^t) }   + \frac{\gamma^2 L_g \sigma^2}{2N} \notag\\
    &\le \Expect\left[ f(\hat{\x}^{t})\right] + \frac{\gamma L_h^2}{2} \left( 8\gamma^2I \cdot\phi(\kappa,I)\cdot\zeta_q^2 + 8\gamma^4L_h^2 \cdot  \frac{I^3[\varphi(\kappa)]^{2(I+2)}}{([\varphi(\kappa)]^2-1)^3}\cdot  \E\normsq{\nabla f(\bar{\x}^r)} +2\gamma^2\cdot \phi(\kappa,I)\cdot \sigma^2 \right) \notag \\
    &- \frac{\gamma}{2}\Expect \normsq{\nabla f(\hat{\x}^{t})}  + \frac{\gamma^2 L_g \sigma^2}{2N}\notag \\
    &= \Expect\left[ f(\hat{\x}^{t})\right]+ \frac{\gamma^2 L_g \sigma^2}{2N} + \frac{\gamma}{2} \left(8\gamma^2L_h^2 I \cdot\phi(\kappa,I)\cdot\zeta_q^2 +2\gamma^2L_h^2\cdot \phi(\kappa,I)\cdot \sigma^2   \right) \notag\\
    &\quad - \left(\frac{\gamma}{2} - 4\gamma^5L_h^4 \cdot  \frac{I^3[\varphi(\kappa)]^{2(I+2)}}{([\varphi(\kappa)]^2-1)^3} \right)\E\normsq{\nabla f(\bar{\x}^r)}.
\end{align}
 
{ 
Let $\gamma \le \frac{1}{2L_h}\cdot \min\{\frac{1}{I}, \frac{([\varphi(\kappa)]^2-1)^3}{[\varphi(\kappa)]^{2(I+2)}} \}$, we have 
\begin{align}
\frac{\gamma}{2} - 4\gamma^5L_h^4 \cdot  \frac{I^3[\varphi(\kappa)]^{2(I+2)}}{([\varphi(\kappa)]^2-1)^3} \ge \frac{\gamma}{4}.
\end{align}
}
Rearranging (\ref{eq:quad-final-func}), we obtain
\begin{align}
    \min_{t\in [T]} \E\normsq{\nabla f(\hat{\x}^{t})}\le \frac{1}{T}\sum_{t=0}^{T-1}\E\normsq{\nabla f(\hat{\x}^{t})}\le \frac{4\mathcal{F}}{\gamma T}+ \frac{2\gamma L_g \sigma^2}{N} + 16\gamma^2L_h^2 I \cdot\phi(\kappa,I)\cdot\zeta_q^2 + 4\gamma^2L_h^2 \cdot\phi(\kappa,I)\cdot\sigma^2,
\end{align}
where $T=RI$.

\subsection{Proof of Theorem~\ref{thm:fedavg-momentum}}\label{sec:proof-fed-momentum}
In this section, we provide the proof for the convergence analysis of FedAvg with momentum. First, we %
apply the technique of auxiliary sequence used by \citet{yu2019linear} to construct the proof. Then, we apply %
Assumptions~\ref{assumption:global-lipschitz-gradient} and \ref{assumption:gradient-to-model} 
at key steps.

Before proceeding to the proof, we introduce the auxiliary sequence $\{\hat{\z}^t\}$. That is, 
\begin{align}
    \hat{\z}^t :=  
    \begin{cases}
      \hat{\x}^t & t=0 ,\\
      \frac{1}{1-\beta} \hat{\x}^t - \frac{\beta}{1-\beta} \hat{\x}^{t-1} & t>0.
    \end{cases}
\end{align}
{  Then during each iteration, we have
\begin{align}
\hat{\z}^{t+1} = \hat{\z}^t - \frac{\gamma}{1-\beta}\cdot\frac{1}{N}\sum_{i=1}^N \g_i(\x_i^t).
\end{align}
}

{  First, we incorporate Lemma 4 in \cite{yu2019linear} as follows to support our proof.
\begin{lemma}[Lemma 4 in \cite{yu2019linear}]\label{lemma:lemma4-yu} 
For FedAvg with momentum, we have
\begin{align}
\sum_{t=0}^{T-1} \normsq{\hat{\z}^t - \hat{\x}^t} \le \frac{\gamma^2\beta^2}{(1-\beta)^4} \sum_{t=0}^{T-1}\normsq{\frac{1}{N}\sum_{i=1}^N\g_i(\x_i^t)}.
\end{align}
\end{lemma}

}

Using Assumption~\ref{assumption:global-lipschitz-gradient}, we have 
\begin{align}\label{eq:iter-momentum-first}
	\E f(\hat{\z}^{t+1}) \le \E f(\hat{\z}^t) -\frac{\gamma}{1-\beta}\E \left\langle \nabla f(\hat{\z}^t), \frac{1}{N}\sum_{i=1}^N \g_i(\x_i^t)\right\rangle + \frac{\gamma^2 L_g}{2(1-\beta)^2} \E\normsq{\frac{1}{N}\sum_{i=1}^N \g_i(\x_i^t)}.
\end{align}
For the inner product in the RHS of (\ref{eq:iter-momentum-first}), we have
\begin{align}
		&-\frac{\gamma}{1-\beta}\E \left\langle \nabla f(\hat{\z}^t), \frac{1}{N}\sum_{i=1}^N \g_i(\x_i^t)\right\rangle \notag \\
		&= -\frac{\gamma}{1-\beta}\E \left[\E_{\x_i^t}\left\langle \nabla f(\hat{\z}^t) - \nabla f(\hat{\x}^t), \frac{1}{N}\sum_{i=1}^N \g_i(\x_i^t)\right\rangle\right] - \frac{\gamma}{1-\beta}\E\left[\E_{\x_i^t}\left\langle \nabla f(\hat{\x}^t), \frac{1}{N}\sum_{i=1}^N \g_i(\x_i^t)\right\rangle\right]\notag \\
        &= -\frac{\gamma}{1-\beta} \E \left\langle \nabla f(\hat{\z}^t) - \nabla f(\hat{\x}^t),  \frac{1}{N}\sum_{i=1}^N\nabla F_i(\x_i^t) \right\rangle- \frac{\gamma}{1-\beta}\E\left\langle \nabla f(\hat{\x}^t), \frac{1}{N}\sum_{i=1}^N \nabla F_i(\x_i^t)\right\rangle.
\end{align}
For the first inner-product term, we have
\begin{align}
	&-\frac{\gamma}{1-\beta} \E \left\langle \nabla f(\hat{\z}^t) - \nabla f(\hat{\x}^t),  \frac{1}{N}\sum_{i=1}^N\nabla F_i(\x_i^t) \right\rangle\notag \\
	&\overset{(a)}{\le} \frac{(1-\beta)}{2\beta L_g}\E\normsq{\nabla f(\hat{\z}^t) - \nabla f(\hat{\x}^t)}+\frac{\gamma^2\beta L_g}{2(1-\beta)^3}\E\normsq{\frac{1}{N}\sum_{i=1}^N\nabla F_i(\x_i^{r,k})},
\end{align}
where $(a)$ is due to $\langle  a,b  \rangle \le \frac{c}{2}\normsq{a}+\frac{1}{2c}\normsq{b}, c>0$. {Here, we let $a=\nabla f(\hat{\z}^t) - \nabla f(\hat{\x}^t)$, $b=\frac{1}{N}\sum_{i=1}^N\nabla F_i(\x_i^t)$ and $c=\frac{(1-\beta)^2}{\gamma \beta L_g}$}.

For the second inner-product term, we have
\begin{align}
	&- \frac{\gamma}{1-\beta} \E \left\langle \nabla f(\hat{\x}^t), \frac{1}{N}\sum_{i=1}^N \g_i(\x_i^t)\right\rangle \notag \\
	&=- \frac{\gamma }{1-\beta}\E \left\langle \nabla f(\hat{\x}^t),  \frac{1}{N}\sum_{i=1}^N\nabla F_i(\x_i^t) \right\rangle \notag \\
	&=\frac{\gamma }{2(1-\beta)}\E\normsq{\nabla f(\hat{\x}^t) -\frac{1}{N}\sum_{i=1}^N\nabla F_i(\x_i^t) } - \frac{\gamma }{2(1-\beta)}\E\normsq{\nabla f(\hat{\x}^t)}  \notag \\
	&- \frac{\gamma }{2(1-\beta)}\E \normsq{\frac{1}{N}\sum_{i=1}^N\nabla F_i(\x_i^t)}.
\end{align}
For the norm square in the RHS of (\ref{eq:iter-momentum-first}), we have
\begin{align}
	&\frac{\gamma^2 L_g}{2(1-\beta)^2} \E\normsq{\frac{1}{N}\sum_{i=1}^N \g_i(\x_i^t)}\notag \\
	&= \frac{\gamma^2 L_g}{2(1-\beta)^2} \E\normsq{\frac{1}{N}\sum_{i=1}^N \left( \g_i(\x_i^t) - \nabla F_i(\x_i^t) + \nabla F_i(\x_i^t) \right)} \notag \\
	&\le \frac{\gamma^2 L_g\sigma^2}{2N(1-\beta)^2} + \frac{\gamma^2 L_g}{2(1-\beta)^2} \E\normsq{\frac{1}{N}\sum_{i=1}^N\nabla F_i(\x_i^t)}.
\end{align}
Substituting back to (\ref{eq:iter-momentum-first}), we get
\begin{align}\label{eq:one-iter-fedavg-momentum}
	&\E f(\hat{\z}^{t+1}) \le \E f(\hat{\z}^t) - \frac{\gamma }{2(1-\beta)}\E\normsq{\nabla f(\hat{\x}^t)}+\frac{\gamma^2 L_g\sigma^2}{2N(1-\beta)^2}+\frac{\gamma }{2(1-\beta)}\E\normsq{\nabla f(\hat{\x}^t) -\frac{1}{N}\sum_{i=1}^N\nabla F_i(\x_i^t) } \notag \\
	&+ \frac{(1-\beta)}{2\beta L_g}\E\normsq{\nabla f(\hat{\z}^t) - \nabla f(\hat{\x}^t)} - \left(\frac{\gamma}{2(1-\beta)}  - \frac{\gamma^2 L_g}{2(1-\beta)^2}  - \frac{\gamma^2\beta L_g}{2(1-\beta)^3}\right)\E\normsq{\frac{1}{N}\sum_{i=1}^N\nabla F_i(\x_i^t)}\notag \\
	& \le \E f(\hat{\z}^t) - \frac{\gamma }{2(1-\beta)}\E\normsq{\nabla f(\hat{\x}^t)}+\frac{\gamma^2 L_g\sigma^2}{2N(1-\beta)^2}+\frac{\gamma L_h^2 }{2(1-\beta)}\frac{1}{N}\sum_{i=1}^N\E\normsq{\hat{\x}^t -\x_i^t } \notag \\
	&+ \frac{(1-\beta)L_g}{2\beta}\E\normsq{ \hat{\z}^t - \hat{\x}^t}- \left(\frac{\gamma}{2(1-\beta)}  - \frac{\gamma^2 L_g}{2(1-\beta)^2}  - \frac{\gamma^2\beta L_g}{2(1-\beta)^3}\right)\E\normsq{\frac{1}{N}\sum_{i=1}^N\nabla F_i(\x_i^t)}.
\end{align}

{ By Lemma~\ref{lemma:lemma4-yu}, }we obtain 
\begin{align}\label{eq:diff-auxi-model}
	&\frac{1}{T}\sum_{t=0}^{T-1} \E\normsq{\hat{\z}^t - \hat{\x}^t} \le \frac{\gamma^2\beta^2}{(1-\beta)^4}\cdot \frac{1}{T}\sum_{t=0}^{T-1} \E\normsq{\frac{1}{N}\sum_{i=1}^N\g_i(\x_i^t)} \notag\\
	&= 	\frac{\gamma^2\beta^2}{(1-\beta)^4}\cdot \frac{1}{T}\sum_{t=0}^{T-1} \E\normsq{\frac{1}{N}\sum_{i=1}^N\left[\g_i(\x_i^t) - \nabla F_i(\x_i^t) + \nabla F_i(\x_i^t)\right]} \notag \\
	&\le \frac{\gamma^2\beta^2\sigma^2}{N(1-\beta)^4} + \frac{\gamma^2\beta^2}{(1-\beta)^4}\cdot \frac{1}{T}\sum_{t=0}^{T-1}\E\normsq{\frac{1}{N}\sum_{i=1}^N \nabla F_i(\x_i^t)}
\end{align}

By Lemma~\ref{lemma:model-divergence-momentum}, with $\frac{1}{1 - \frac{6\gamma^2I^2(L_h^2 + L_g^2)}{(1-\beta)^2}} >0$, we get
\begin{align}\label{eq:diff-model-fedavg-momentum}
		\frac{1}{T}\sum_{t=0}^{T-1} \frac{1}{N}\sum_{i=1}^N\E\normsq{\hat{\x}^t -\x_i^t }\le \frac{1}{1 - \frac{6\gamma^2I^2(L_h^2 + L_g^2)}{(1-\beta)^2}}\cdot \left(\frac{2\gamma^2I\sigma^2}{(1-\beta^2)} + \frac{6\gamma^2I^2\zeta^2}{(1-\beta)^2}  \right).
\end{align}
{ Substituting (\ref{eq:diff-auxi-model}) and (\ref{eq:diff-model-fedavg-momentum}) back to (\ref{eq:one-iter-fedavg-momentum}),} we obtain
\begin{align}
		&\frac{\gamma }{2(1-\beta)}\cdot\frac{1}{T}\sum_{t=0}^{T-1}\E\normsq{\nabla f(\hat{\x}^t)}\le \frac{(f_0-f_*)}{T}+\frac{\gamma^2 L_g\sigma^2}{2N(1-\beta)^2} + \frac{(1-\beta)L_g}{2\beta}\cdot \frac{\gamma^2\beta^2\sigma^2}{N(1-\beta)^4} \notag \\
		&+\frac{\gamma L_h^2 }{2(1-\beta)} \cdot \frac{1}{1 - \frac{6\gamma^2I^2(L_h^2 + L_g^2)}{(1-\beta)^2}}\cdot \left(\frac{2\gamma^2I\sigma^2}{(1-\beta^2)} + \frac{6\gamma^2I^2\zeta^2}{(1-\beta)^2}  \right) \notag \\
		& - \left(\frac{\gamma}{2(1-\beta)}  - \frac{\gamma^2 L_g}{2(1-\beta)^2}  - \frac{\gamma^2\beta L_g}{2(1-\beta)^3} - \frac{(1-\beta)L_g}{2\beta}\cdot\frac{\gamma^2\beta^2}{(1-\beta)^4}\right)\frac{1}{T}\sum_{t=0}^{T-1}\E\normsq{\frac{1}{N}\sum_{i=1}^N\nabla F_i(\x_i^t)}.
\end{align}
{ 
Dividing both sides by $\frac{\gamma}{2(1-\beta)}$, we obtain
\begin{align}
&\frac{1}{T}\sum_{t=0}^{T-1}\E\normsq{\nabla f(\hat{\x}^t)}\notag \\
&\le \frac{2(1-\beta)(f_0-f_*)}{\gamma T}+\frac{\gamma L_g\sigma^2}{N(1-\beta)^2} + L_h^2 \cdot \frac{1}{1 - \frac{6\gamma^2I^2(L_h^2 + L_g^2)}{(1-\beta)^2}}\cdot \left(\frac{2\gamma^2I\sigma^2}{(1-\beta^2)} + \frac{6\gamma^2I^2\zeta^2}{(1-\beta)^2}  \right)\notag \\
&-\left(1  - \frac{\gamma L_g}{1-\beta}  - \frac{2\gamma\beta L_g}{(1-\beta)^2} \right)\frac{1}{T}\sum_{t=0}^{T-1}\E\normsq{\frac{1}{N}\sum_{i=1}^N\nabla F_i(\x_i^t)}.
\end{align}
By $\gamma \le \frac{1-\beta}{\sqrt{18(L_g^2 + L_h^2)}I}$, we have
\begin{align}
\frac{1}{1 - \frac{6\gamma^2I^2(L_h^2 + L_g^2)}{(1-\beta)^2}} \le \frac{3}{2}.
\end{align}
By $\gamma \le \frac{(1-\beta)^2}{L_g(1+\beta)}$, we have
\begin{align}
1  - \frac{\gamma L_g}{1-\beta}  - \frac{2\gamma\beta L_g}{(1-\beta)^2} \ge 1-\frac{1-\beta}{1+\beta} - \frac{2\beta}{1+\beta} = 0.
\end{align}
With $\gamma\le \min\{ \frac{(1-\beta)^2}{L_g(1+\beta)}, \frac{1-\beta}{\sqrt{18(L_g^2 + L_h^2)}I}\}$
we obtain
\begin{align}
\min_t\E\normsq{\nabla f(\hat{\x}^t)} \le \frac{1}{T}\sum_{t=0}^{T-1}\E\normsq{\nabla f(\hat{\x}^t)} \le \frac{2(1-\beta)(f_0-f_*)}{\gamma T}+\frac{\gamma L_g\sigma^2}{N(1-\beta)^2} + \frac{3\gamma^2L_h^2I\sigma^2}{(1-\beta)^2} + \frac{9\gamma^2L_h^2 I^2\zeta^2}{(1-\beta)^2}. 
\end{align}
}

\subsection{Proof of Theorem~\ref{thm:fedadam}}\label{sec:proof-fedadam}
In this section, we use the techniques of \citet{reddi2020adaptive} in the proof. 
{  We define the update at $r$th round $\Delta_r$ as
\begin{align}
    \Delta_r := \frac{1}{N}\sum_{i=1}^N \x_i^{r,I} - \bar{\x}^r.
\end{align}

}

In FedAdam, the global update is given by
\begin{align}
    \bar{\x}^{r+1} = \bar{\x}^r + \eta \frac{\Delta_r}{\sqrt{\bv^r}+ \tau},
\end{align}
where 
\begin{align}
\bv^r = \beta_2 \bv^{r-1} + (1-\beta_2) \Delta_r^2. 
\end{align}
We use $\Delta_{r,j}$ to denote the $j$th element of $\Delta_r$. We use $v^r_j$ to denote the $j$th element of $\bv^r$. The division is element-wise. $\frac{\Delta_r}{\sqrt{\bv^r}+ \tau}$ means that for each element $j\in [d]$, we perform $\frac{\Delta_{r,j}}{\sqrt{v^r_j}+\tau}$, and $\Delta_r^2$ means that for each element $j\in [d]$, we perform $\Delta_{r,j}^2$.

By Assumption~\ref{assumption:global-lipschitz-gradient}, we obtain
\begin{align}\label{eq:fedadam-first}
    f(\bar{\x}^{r+1}) \le f(\bar{\x}^r) + \eta \left\langle \nabla f(\bar{\x}^r),\frac{\Delta_r}{\sqrt{\bv^r}+ \tau} \right\rangle + \frac{\eta^2 L_g}{2}\sum_{j=1}^d \frac{\Delta_{r,j}^2}{(\sqrt{\bv_j^r}+ \tau)^2}.
\end{align}
For the inner-product term, we have
\begin{align}\label{eq:fedadam-inner-prod}
    &\eta \left\langle \nabla f(\bar{\x}^r),\frac{\Delta_r}{\sqrt{\bv^r}+ \tau} \right\rangle \notag \\
    &=\eta \left\langle \nabla f(\bar{\x}^r),\frac{\Delta_r}{\sqrt{\bv^r}+ \tau} -  \frac{\Delta_r}{\sqrt{\beta_2\bv^{r-1}}+ \tau} \right\rangle + \eta \left\langle \nabla f(\bar{\x}^r),\frac{\Delta_r}{\sqrt{\beta_2\bv^{r-1}}+ \tau} \right\rangle.
\end{align}
By (14)--(15) in \citet{reddi2020adaptive}, we have
\begin{align}
    &\eta \E \left\langle \nabla f(\bar{\x}^r),\frac{\Delta_r}{\sqrt{\bv^r}+ \tau} -  \frac{\Delta_r}{\sqrt{\beta_2\bv^{r-1}}+ \tau} \right\rangle  \notag \\
    & = \eta\sqrt{1-\beta_2}\E \sum_{j=1}^d \frac{G}{\tau} \cdot \frac{\Delta^2_{r,j}}{\sqrt{v^r_j}+\tau},
\end{align}
and
\begin{align}
   & \eta\E \left\langle \nabla f(\bar{\x}^r),\frac{\Delta_r}{\sqrt{\beta_2\bv^{r-1}}+ \tau} \right\rangle \notag\\
   &= -\gamma\eta I\E\sum_{j=1}^d \frac{[\nabla f(\bar{\x}^r)]_j^2}{\sqrt{\beta_2v_j^{r-1}}+\tau} + \eta\E \left\langle \frac{\nabla f(\bar{\x}^r)}{\sqrt{\beta_2\bv^{r-1}}+\tau}, \Delta_r + \gamma I \nabla f(\bar{\x}^r) \right\rangle.
\end{align}
For the second term in the RHS of above inequality, we have 
\begin{align}
    &\eta\E \left\langle \frac{\nabla f(\bar{\x}^r)}{\sqrt{\beta_2\bv^{r-1}}+\tau}, \Delta_r + \gamma I \nabla f(\bar{\x}^r) \right\rangle \notag\\
    &{  = -\eta\gamma \E\left[ \E_{\bar{\x}^r}\left\langle \frac{\nabla f(\bar{\x}^r)}{\sqrt{\beta_2\bv^{r-1}}+\tau}, \frac{1}{N}\sum_{i=1}^N\sum_{k=0}^{I-1} \g_i(\x_i^{r,k}) -  I \nabla f(\bar{\x}^r) \right\rangle \right]}\notag \\
    &= -\eta\gamma \E \left\langle \frac{\nabla f(\bar{\x}^r)}{\sqrt{\beta_2\bv^{r-1}}+\tau}, \frac{1}{N}\sum_{i=1}^N\sum_{k=0}^{I-1} \left(\nabla F_i(\x_i^{r,k}) -  \nabla f(\bar{\x}^r) \right)\right\rangle \notag\\
    &\le \frac{\eta\gamma I}{2}\E\sum_{j=1}^d \frac{[\nabla f(\bar{\x}^r)]_j}{\sqrt{\beta_2 v_j^{r-1}}+\tau} + \frac{\eta\gamma }{2(\sqrt{\beta_2\bv^{r-1}}+\tau)} \sum_{k=0}^{I-1} \E \normsq{\frac{1}{N}\sum_{i=1}^N \nabla F_i(\x_i^{r,k}) -  \nabla f(\bar{\x}^r)} \notag \\
    &\le \frac{\eta\gamma I}{2}\E\sum_{j=1}^d \frac{[\nabla f(\bar{\x}^r)]_j}{\sqrt{\beta_2 v_j^{r-1}}+\tau} + \frac{\eta\gamma}{2\tau}\sum_{k=0}^{I-1}\E \normsq{\frac{1}{N}\sum_{i=1}^N \nabla F_i(\x_i^{r,k}) - \nabla f(\hat{\x}^{r,k}) +\nabla f(\hat{\x}^{r,k})  -  \nabla f(\bar{\x}^r)} \notag\\
    &\le \frac{\eta\gamma I}{2}\E\sum_{j=1}^d \frac{[\nabla f(\bar{\x}^r)]_j}{\sqrt{\beta_2 v_j^{r-1}}+\tau} + \frac{\eta\gamma L_h^2}{2\tau}\sum_{k=0}^{I-1}\frac{1}{N}\sum_{i=1}^N\E\normsq{\x_i^{r,k}-\hat{\x}^{r,k}} + \frac{\eta\gamma I L_g^2}{2\tau}\E\normsq{\hat{\x}^{r,k}- \bar{\x}^r}.
\end{align}
Since the local updates of FedAdam are the same as that of FedAvg, we can apply Lemma~\ref{lemma:model-divergence} and Lemma~\ref{lemma:distance-averaged-models} in the above inequality. Then, by (\ref{eq:apply-lemmas}), 
with $\gamma\le \frac{1}{\sqrt{6(L_h^2+L_g^2)}I}$ %
and $\gamma \le \frac{1}{2\sqrt{3}IL_g}$, we obtain
\begin{align}\label{eq:sum-of-model}
    &\frac{\eta\gamma L_h^2}{2\tau}\sum_{k=0}^{I-1}\frac{1}{N}\sum_{i=1}^N\E\normsq{\x_i^{r,k}-\hat{\x}^r} + \frac{\eta\gamma I L_g^2}{2\tau}\E\normsq{\hat{\x}^{r,k}- \bar{\x}^r}\notag \\
    &\le  \frac{\gamma\eta L_g^2 I}{2\tau} \left( 5(I-1) \frac{\gamma^2\sigma^2}{N} + 30I\gamma^2 \sum_{k=0}^{I-1} \frac{L_h^2}{N}\sum_{i=1}^N \Expect\normsq{\x_i^{r,k} -\hat{\x}^{r,k}}  
    +30I(I-1) \gamma^2 \Expect\normsq{\nabla f(\bar{\x}^r)} \right) \notag \\
    &+ \frac{\eta\gamma L_h^2}{2\tau}\sum_{k=0}^{I-1}\frac{1}{N}\sum_{i=1}^N\E\normsq{\x_i^{r,k}-\hat{\x}^r} \notag \\
    &\le \frac{5\gamma^3\eta L_g^2 I^2\sigma^2 }{2\tau N} + \frac{15\gamma^3\eta L_g^2 I^3}{\tau}\Expect\normsq{\nabla f(\bar{\x}^r)} + \left(\frac{15\gamma^3\eta L_g^2L_h^2 I^2}{\tau} +  \frac{\eta\gamma L_h^2}{2\tau} \right) \sum_{k=0}^{I-1}\frac{1}{N}\sum_{i=1}^N\E\normsq{\x_i^{r,k}-\hat{\x}^r}  \notag \\
    & \le \frac{5\gamma^3\eta L_g^2 I^2\sigma^2 }{2\tau N} + \frac{15\gamma^3\eta L_g^2 I^3}{\tau}\Expect\normsq{\nabla f(\bar{\x}^r)} + \frac{\eta\gamma}{\tau}\left( 12(I-1)^3\gamma^2L_h^2 \zeta^2+ 4(I-1)^2\gamma^2 L_h^2 \sigma^2\right).
\end{align}

According to (16) and its proof in \citet{reddi2020adaptive}, with $\gamma\le \min  \left\{\frac{1}{16L_gI}, \frac{\tau^{\frac{1}{3}}}{16K(120L_g^2G)^{\frac{1}{3}}}\right\} $, we have
\begin{align}\label{eq:lr-moment}
\frac{15\gamma^2L_g^2I^2}{\tau} \E\normsq{\nabla f(\bar{\x})}  \le  \frac{1}{4}\E\sum_{j=1}^d\frac{[\nabla f(\bar{\x}^r)]_j^2}{\beta_2\sqrt{v^r_j}+\tau}.
\end{align}

Substituting (\ref{eq:fedadam-inner-prod})--(\ref{eq:sum-of-model}) back to (\ref{eq:fedadam-first}),  we can get
\begin{align}\label{eq:fedadam-last}
    &\E f(\bar{\x}^{r+1}) \le \E f(\bar{\x}^r) + \eta\sqrt{1-\beta_2}\E \sum_{j=1}^d \frac{G}{\tau} \cdot \frac{\Delta^2_{r,j}}{\sqrt{v^r_j}+\tau} -\gamma\eta I\E\sum_{j=1}^d \frac{[\nabla f(\bar{\x}^r)]_j^2}{\sqrt{\beta_2v_j^{r-1}}+\tau} + \frac{\eta\gamma I}{2}\E\sum_{j=1}^d \frac{[\nabla f(\bar{\x}^r)]^2_j}{\sqrt{\beta_2 v_j^{r-1}}+\tau}\notag \\ 
    &+  \frac{5\gamma^3\eta L_g^2 I^2\sigma^2 }{2\tau N} + \frac{15\gamma^3\eta L_g^2 I^3}{\tau}\Expect\normsq{\nabla f(\bar{\x}^r)} + \frac{\eta\gamma}{\tau}\left( 12(I-1)^3\gamma^2L_h^2 \zeta^2+ 4(I-1)^2\gamma^2 L_h^2 \sigma^2\right) + \frac{\eta^2 L_g}{2}\E\sum_{j=1}^d \frac{\Delta_{r,j}^2}{(\sqrt{v_j^r}+ \tau)^2} \notag \\
    &\overset{(a)}{\le} \E f(\bar{\x}^r) + \left(\frac{\eta\sqrt{1-\beta_2}G}{\tau} + \frac{\eta^2L_g}{2}\right)\E\sum_{j=1}^d\frac{\Delta^2_{r,j}}{\sqrt{v^r_j}+\tau}-\frac{\gamma\eta I}{4}\E\sum_{j=1}^d \frac{[\nabla f(\bar{\x}^r)]_j^2}{\sqrt{\beta_2v_j^{r-1}}+\tau} + \frac{5\gamma^3\eta L_g^2 I^2\sigma^2 }{2\tau N} \notag \\
    &+\frac{\eta\gamma}{\tau}\left( 12(I-1)^3\gamma^2L_h^2 \zeta^2+ 4(I-1)^2\gamma^2 L_h^2 \sigma^2\right),
\end{align}
where $(a)$ is due to (\ref{eq:lr-moment}).

Similar to the proof of Lemma 4 in \citet{reddi2020adaptive}, we obtain
\begin{align}\label{eq:fedadam-lemma4}
&\E\sum_{j=1}^d \frac{\Delta_{r,j}^2}{(\sqrt{v_j^r}+ \tau)^2} \notag \\
&\le \E\sum_{j=1}^d \frac{\Delta_{r,j}^2}{\tau^2} \notag \\
&\le 2\E \normsq{\frac{\Delta_r + \gamma I\nabla f(\bar{\x}^r)}{\tau}} + 2\gamma^2I^2\E\normsq{\frac{\nabla f(\bar{\x}^r)}{\tau}}.
\end{align}
Furthermore, by Lemma~\ref{lemma:local-gradient-deviation} and Lemma~\ref{lemma:model-divergence}, we have
\begin{align}\label{eq:fedadam-grad-norm}
    &2\E \normsq{\frac{\Delta_r + \gamma I\nabla f(\bar{\x}^r)}{\tau}} \notag \\
    &\le \frac{4\gamma^2 I}{N\tau^2}
    \sigma^2 + \frac{4\gamma^2I}{\tau^2}\E\sum_{k=0}^{I-1} \normsq{\frac{1}{N}\sum_{i=1}^N \nabla F_i(\x_i^{r,k}) - \nabla f(\hat{\x}^{r,k}) + \nabla f(\hat{\x}^{r,k}) - \nabla f(\bar{\x}^r)} \notag \\
    &\le \frac{4\gamma^2 I}{N\tau^2}
    \sigma^2 + \frac{4\gamma^2 I}{\tau^2}\left(\frac{5\gamma^2 L_g^2 I^2\sigma^2 }{ N} + 30\gamma^2 L_g^2 I^3\Expect\normsq{\nabla f(\bar{\x}^r)} + 24(I-1)^3\gamma^2L_h^2 \zeta^2+ 8(I-1)^2\gamma^2 L_h^2 \sigma^2\right) .
\end{align}
According to the proof for Theorem 2 in \citet{reddi2020adaptive}, with $\gamma \le \min\{\frac{1}{16IL_g},\frac{\tau}{6(2G+\eta L_g)}\}$, we have
\begin{align}\label{eq:fact}
\left(\sqrt{1-\beta_2}G + \frac{\eta L_g}{2} \right)\frac{2\gamma^2I^2+ 120\gamma^4L_g^2I^4}{\tau^2} \le \frac{\gamma I}{8} \frac{1}{\sqrt{\beta_2}\gamma IG + \tau}.
\end{align}

Substituting (\ref{eq:fedadam-lemma4}) and (\ref{eq:fedadam-grad-norm}) back to (\ref{eq:fedadam-last}),  we obtain
\begin{align}
&\E f(\bar{\x}^{r+1}) \le \E f(\bar{\x}^r) -\frac{\gamma\eta I}{4}\E\sum_{j=1}^d \frac{[\nabla f(\bar{\x}^r)]_j^2}{\sqrt{\beta_2v_j^{r-1}}+\tau}  + \frac{5\gamma^3\eta L_g^2 I^2\sigma^2 }{2\tau N} \notag \\
&+ \left(\frac{\eta\sqrt{1-\beta_2}G}{\tau} + \frac{\eta^2L_g}{2}\right)\left(\frac{4\gamma^2 I}{N\tau^2}
    \sigma^2+ \frac{2\gamma^2I^2+ 120\gamma^4L_g^2I^4}{\tau^2}\Expect\normsq{\nabla f(\bar{\x}^r)} + \frac{96\gamma^4L_h^2I^4\zeta^2}{\tau^2} + \frac{32\gamma^4L_h^2I^3\sigma^2}{\tau^2} \right) \notag \\
&+\frac{\eta\gamma}{\tau}\left( 12(I-1)^3\gamma^2L_h^2 \zeta^2+ 4(I-1)^2\gamma^2 L_h^2 \sigma^2\right)\notag \\
&\overset{(a)}{\le} \E f(\bar{\x}^r) -\frac{\gamma\eta I}{8}\E\sum_{j=1}^d \frac{[\nabla f(\bar{\x}^r)]_j^2}{\sqrt{\beta_2v_j^{r-1}}+\tau} +\frac{5\gamma^3\eta L_g^2 I^2\sigma^2 }{2\tau N}+\frac{\eta\gamma}{\tau}\left( 12(I-1)^3\gamma^2L_h^2 \zeta^2+ 4(I-1)^2\gamma^2 L_h^2 \sigma^2\right)  \notag \\
&+ \left(\eta\sqrt{1-\beta_2}G + \frac{\eta^2L_g}{2}\right)\left(\frac{4\gamma^2 I}{N\tau^2}
    \sigma^2+ \frac{96\gamma^4L_h^2I^4\zeta^2}{\tau^2} + \frac{32\gamma^4L_h^2I^3\sigma^2}{\tau^2}   \right) ,
\end{align}
where $(a)$ is due to (\ref{eq:fact}).

Rearranging the above inequality, we have
\begin{align}\label{eq:a7-avg}
&\frac{1}{R}\sum_{r=0}^{R-1} \E\sum_{j=1}^d \frac{[\nabla f(\bar{\x}^r)]_j^2}{\sqrt{\beta_2v_j^{r-1}}+\tau} \le \frac{8(f_0-f_*)}{\gamma\eta IR} +\frac{\gamma L_g I\sigma^2 }{\tau N}+\frac{96\gamma^2I^2L_h^2\zeta^2}{\tau} + \frac{32\gamma^2L_hI\sigma^2}{\tau} \notag \\
&+\left(\sqrt{1-\beta_2}G + \frac{\eta L_g}{2}\right)\left(\frac{32\gamma }{N\tau^2}
    \sigma^2+ \frac{768\gamma^3 L_h^2 I^3\zeta^2}{\tau^2} + \frac{256\gamma^3L_h^2 I^2\sigma^2}{\tau^2}   \right).
\end{align}
By the proof of Theorem 2 in \citet{reddi2020adaptive}, we have
\begin{align}\label{eq:avg-min-momentum}
\frac{1}{R}\sum_{r=0}^{R-1} \E\sum_{j=1}^d \frac{[\nabla f(\bar{\x}^r)]_j^2}{\sqrt{\beta_2v_j^{r-1}}+\tau} \ge \frac{1}{R}\sum_{r=0}^{R-1} \E\sum_{j=1}^d \frac{[\nabla f(\bar{\x}^r)]_j^2}{\sqrt{\beta_2}\gamma IG+\tau}\ge \frac{1}{\sqrt{\beta_2}\gamma IG+\tau} \min_r \E\normsq{\nabla f(\bar{\x}^r)}.
\end{align}
Substituting (\ref{eq:avg-min-momentum}) back to (\ref{eq:a7-avg}) and rearranging, we have
\begin{align}
\min_r\E\normsq{\nabla f(\bar{\x}^r)} \le &\left( \sqrt{\beta_2}\gamma IG +\tau \right) \left(  \frac{8(f_0-f_*)}{\gamma\eta I R} + \frac{\gamma L_g \sigma^2}{\tau N} + \frac{96\gamma^2I^2L_h^2\zeta^2}{\tau} + \frac{32\gamma^2L_hI\sigma^2}{\tau} \right) \notag \\
&+ \left( \sqrt{\beta_2}\gamma IG +\tau \right) \left(\sqrt{1-\beta_2}G + \frac{\eta L_g}{2}\right)\left(\frac{32\gamma }{N\tau^2}
    \sigma^2+ \frac{768\gamma^3 L_h^2 I^3\zeta^2}{\tau^2} + \frac{256\gamma^3L_h^2 I^2\sigma^2}{\tau^2}   \right).
\end{align}
{ 
\subsection{Proof of Theorem~\ref{thm:mu-convex}}
In this section, we apply Assumption~\ref{assumption:gradient-to-model} in the convergence analysis for strongly convex objective functions in \citet{karimireddy2020scaffold}. First, we bound the term $A_1$ in the proof of Lemma 7 in \citet{karimireddy2020scaffold} using techniques in our paper.
\begin{align}
    A_1 &= \frac{2\gamma\eta}{N}\sum_{k,i}\langle \nabla F_i(\x_i^{k-1}), \x^* - \x \rangle \notag \\
    &= \frac{2\gamma\eta}{N}\sum_{k,i}\langle \nabla F_i(\x_i^{k-1}) - \nabla f(\x) + \nabla f(\x), \x^* - \x \rangle \notag \\
    &= \underbrace{\frac{2\gamma\eta}{N}\sum_{k,i}\langle \nabla F_i(\x_i^{k-1}) - \nabla f(\x), \x^* - \x \rangle}_{T_1} + \underbrace{ \frac{2\gamma\eta}{N}\sum_{k,i}\langle  \nabla f(\x), \x^* - \x \rangle}_{T_2}.
\end{align}
For $T_1$, we have
\begin{align}
    T_1&=2\gamma\eta\left\langle \frac{1}{N} \sum_{k,i}\nabla F_i(\x_i^{k-1}) - I\nabla f(\x), \x^* - \x \right\rangle  \notag \\ 
    &\le \gamma\eta\left( a \underbrace{\left\| \frac{1}{N} \sum_{k,i}\nabla F_i(\x_i^{k-1}) - I\nabla f(\x) \right\|^2}_{T_3} + \frac{1}{a}\left\| \x^* - \x\right\|^2 \right),
\end{align}
where $a>0$ is a constant. Furthermore, for $T_3$, we have
\begin{align}
    T_3 &= \left\| \frac{1}{N} \sum_{k,i}\nabla F_i(\x_i^{k-1}) - \sum_{k} \nabla f(\hat{\x}^{k-1}) + \sum_{k} \nabla f(\hat{\x}^{k-1}) - I\nabla f(\x) \right\|^2 \notag \\
    &\le \frac{2L_h^2I}{N}\sum_{i,k}\|\x_i^{k-1} - \hat{\x}^{k-1} \|^2  +  2L_g^2I\sum_k\|\hat{\x}^{k-1} - \x \|^2.
\end{align}

For $T_2$, according to Lemma 5 in \citet{karimireddy2020scaffold}, we obtain
\begin{align}
T_2 &\le \frac{2\gamma\eta}{N}\sum_{k,i} \left( f(\x^*) - f(\x) - \frac{\mu}{4} \| \x-\x^*\|^2  \right)\notag \\
&= 2\gamma\eta I \left( f(\x^*) - f(\x) - \frac{\mu}{4} \| \x-\x^*\|^2  \right).
\end{align}

Substituting $T_1$, $T_2$ and $T_3$ back to $A_1$, we obtain
\begin{align} 
A_1 \le & a\gamma\eta \left(\frac{2L_h^2I}{N}\sum_{i,k}\|\x_i^{k-1} - \hat{\x}^{k-1} \|^2  +  2L_g^2I\sum_k\|\hat{\x}^{k-1} - \x \|^2 \right) + \frac{\gamma\eta}{a} \left\| \x^* - \x\right\|^2\notag \\
&+ 2\gamma\eta I \left( f(\x^*) - f(\x) - \frac{\mu}{4} \| \x-\x^*\|^2  \right).
\end{align}

Now we bound the term $A_2$ in the proof of Lemma 7 in \citet{karimireddy2020scaffold}. 
\begin{align}
A_2 &= \gamma^2\eta^2I^2  \left\| \frac{1}{NI} \sum_{k,i}\left(\nabla F_i(\x_i^{k-1}) -  \nabla f(\hat{\x}^{k-1}) + \nabla f(\hat{\x}^{k-1}) - \nabla f(\x) + \nabla f(\x)\right) \right\|^2 \notag \\
&\le  \frac{3\gamma^2\eta^2L_h^2I}{N}\sum_{i,k} \normsq{\x_i^{k-1} - \hat{\x}^{k-1}} + 3\gamma^2\eta^2L_g^2I\sum_k \normsq{\hat{\x}^{k-1} - \x} + 3\gamma^2\eta^2I^2 \normsq{ \nabla f(\x)}\notag\\
&\le \frac{3\gamma^2\eta^2L_h^2I}{N}\sum_{i,k} \normsq{\x_i^{k-1} - \hat{\x}^{k-1}} + 3\gamma^2\eta^2L_g^2I\sum_k \normsq{\hat{\x}^{k-1} - \x} + 3\gamma^2\eta^2L_gI^2(f(\x) - f(\x^*)).
\end{align}
By Lemma 7 in \citet{karimireddy2020scaffold}, we have
\begin{align}
\normsq{\bar{\x}^{r+1} - \x^*} \le \normsq{\bar{\x}^r - \x^*} + A_1 + A_2 + \frac{\gamma^2\eta^2 I \sigma^2}{N}.
\end{align}
Substituting $A_1$ and $A_2$ to the above inequality, we obtain
\begin{align}
\normsq{\bar{\x}^{r+1} - \x^*} \le& \left( 1 + \frac{\gamma\eta}{a} - \frac{\gamma\eta\mu I}{2}\right)\normsq{\bar{\x}^r - \x^*} - \left(2\gamma\eta I - 3\gamma^2\eta^2L_gI^2  \right) \left( f(\bar{\x}^r) - f^* \right) \notag \\
&+ \left(  \frac{2a\gamma\eta L_h^2I}{N}+ \frac{3\gamma^2\eta^2L_h^2I}{N} \right)\sum_{i,k} \|\x_i^{r,k-1} - \hat{\x}^{r,k-1} \|^2 \notag \\
&+ \left( 2a\gamma\eta L_g^2 I + 3\gamma^2\eta^2L_g^2I  \right)\sum_k \normsq{\hat{\x}^{r,k-1} - \bar{\x}^r} + \frac{\gamma^2\eta^2I\sigma^2 }{N}.
\end{align}
By choosing $a=\frac{4}{\mu I}$, we obtain
\begin{align}\label{eq:choosing-a}
\normsq{\bar{\x}^{r+1} - \x^*} \le& \left( 1 + \frac{\gamma\eta\mu I}{4} - \frac{\gamma\eta\mu I}{2}\right)\normsq{\bar{\x}^r - \x^*} - \left(2\gamma\eta I - 3\gamma^2\eta^2L_gI^2  \right) \left( f(\bar{\x}^r) - f^* \right) \notag \\
&+ \left(  \frac{8\gamma\eta L_h^2I}{\mu I N}+ \frac{3\gamma^2\eta^2L_h^2I}{N} \right)\sum_{i,k} \|\x_i^{r,k-1} - \hat{\x}^{r,k-1} \|^2 \notag \\
&+ \left( \frac{8\gamma\eta L_g^2 I}{\mu I} + 3\gamma^2\eta^2L_g^2I  \right)\sum_k \normsq{\hat{\x}^{r,k-1} - \bar{\x}^r}+\frac{\gamma^2\eta^2I\sigma^2 }{N}  \notag \\
\le &  \left( 1 -\frac{\gamma\eta\mu I}{4}  \right)\normsq{\bar{\x}^r - \x^*} - \left(2\gamma\eta I - 3\gamma^2\eta^2L_gI^2  \right) \left( f(\bar{\x}^r) - f^* \right)+\frac{\gamma^2\eta^2I\sigma^2 }{N} \notag \\
&+ \underbrace{\left(8\gamma\eta L_h \frac{L_h}{\mu}  + 3\gamma^2\eta^2L_h^2 I  \right)\cdot \frac{1}{N} \sum_{i,k} \|\x_i^{r,k-1} - \hat{\x}^{r,k-1} \|^2}_{T_4} \notag \\
&+ \underbrace{\left( 8\gamma\eta L_g \frac{L_g}{\mu} + 3\gamma^2\eta^2L_g^2I \right)\sum_k \normsq{\hat{\x}^{r,k-1} - \bar{\x}^r}}_{T_5}.
\end{align}
Applying Lemma~\ref{lemma:model-divergence} and Lemma~\ref{lemma:distance-averaged-models} to (\ref{eq:choosing-a}), we obtain
\begin{align}
T_4 \le \left(8\gamma\eta L_h \frac{L_h}{\mu}  + 3\gamma^2\eta^2L_h^2 I  \right) \cdot \left( 12(I-1)^3\gamma^2 \zeta^2+ 4(I-1)^2\gamma^2 \sigma^2 \right),
\end{align}
and 
\begin{align}
T_5 \le& \left( 8\gamma\eta L_g I \frac{L_g}{\mu}  + 3\gamma^2\eta^2L_g^2I^2 \right)\cdot \bigg( 5(I-1) \cdot\frac{\gamma^2\sigma^2}{N} +30I(I-1) \gamma^2L_g \left(f(\bar{\x}^r) - f^* \right) \notag \\
&+ 30I\gamma^2L_h^2 \left(12(I-1)^3\gamma^2 \zeta^2+ 4(I-1)^2\gamma^2 \sigma^2 \right) \bigg)\notag \\
\le& \left( 8\gamma\eta L_g I \frac{L_g}{\mu}  + 3\gamma^2\eta^2L_g^2I^2\cdot \frac{L_g}{\mu} \right)\cdot \bigg( 5(I-1) \cdot\frac{\gamma^2\sigma^2}{N} +30I(I-1) \gamma^2L_g \left(f(\bar{\x}^r) - f^* \right)\notag \\
&+ 30I\gamma^2L_h^2 \left(12(I-1)^3\gamma^2 \zeta^2+ 4(I-1)^2\gamma^2 \sigma^2 \right) \bigg) \notag \\
\le & \left( \gamma\eta I   + \gamma^2\eta^2L_gI^2 \right) \cdot 8L_g\frac{L_g}{\mu}\cdot \bigg( 5(I-1) \cdot\frac{\gamma^2\sigma^2}{N} +30I(I-1) \gamma^2L_g \left(f(\bar{\x}^r) - f^* \right) \notag \\
&+ 30I\gamma^2L_h^2 \left(12(I-1)^3\gamma^2 \zeta^2+ 4(I-1)^2\gamma^2 \sigma^2 \right) \bigg).
\end{align}
In particular, we have 
\begin{align}
&\left( \gamma\eta I   + \gamma^2\eta^2L_gI^2 \right) \cdot 8L_g\frac{L_g}{\mu}\cdot 30I(I-1) \gamma^2L_g \left(f(\bar{\x}^r) - f^* \right)\notag \\
&\overset{(a)}{\le}   \frac{1}{2}\left( \gamma\eta I   + \gamma^2\eta^2L_gI^2 \right)\left(f(\bar{\x}^r) - f^* \right),
\end{align}
where $(a)$ is due to $\gamma \le \frac{1}{24L_gI}\sqrt{\frac{\mu}{L_g}}$.
Substituting $T_4$ and $T_5$ back to (\ref{eq:choosing-a}), by $\gamma \le \frac{1}{24L_gI}\sqrt{\frac{\mu}{L_g}}$ and $\gamma\eta \le \frac{1}{16L_g I}$, we obtain
\begin{align}
&\normsq{\bar{\x}^{r+1} - \x^*} \notag \\
&\le\left( 1 -\frac{\gamma\eta\mu I}{4}  \right)\normsq{\bar{\x}^r - \x^*} - \left(\frac{1}{2}\gamma\eta I - 4\gamma^2\eta^2L_gI^2  \right) \left( f(\bar{\x}^r) - f^* \right)+\frac{\gamma^2\eta^2I\sigma^2 }{N}  \notag \\
&+ \left(8\gamma\eta L_h \frac{L_h}{\mu}  + 3\gamma^2\eta^2L_h^2 I  \right) \cdot \left( 3c(I-1)^3\gamma^2 \zeta^2+ c(I-1)^2\gamma^2 \sigma^2 \right)\notag \\
&+ \left(\gamma\eta I + \gamma^2\eta^2L_g I^2  \right) \bigg( 40\gamma^2L_g I\cdot \frac{L_g\sigma^2}{\mu N} + 8\frac{L_g^2}{\mu}\cdot 30I\gamma^2L_h^2 \left(12(I-1)^3\gamma^2 \zeta^2+ 4(I-1)^2\gamma^2 \sigma^2 \right)   \bigg) \notag \\
&\le \left( 1 -\frac{\gamma\eta\mu I}{4}  \right)\normsq{\bar{\x}^r - \x^*} - \frac{\gamma\eta I}{4} \left( f(\bar{\x}^r) - f^* \right)+\frac{\gamma^2\eta^2I\sigma^2 }{N}  \notag \\
&+ \left(8\gamma\eta L_h \frac{L_h}{\mu}  + 3\gamma^2\eta^2L_h^2 I  \right) \cdot \left( 12(I-1)^3\gamma^2 \zeta^2+ 4(I-1)^2\gamma^2 \sigma^2 \right)\notag \\
&+ \left(\gamma\eta I + \gamma^2\eta^2L_g I^2  \right) \bigg( 40\gamma^2L_g I\frac{L_g\sigma^2}{\mu N} + 8\frac{L_g^2}{\mu}\cdot 30I\gamma^2L_h^2 \left(12(I-1)^3\gamma^2 \zeta^2+ 4(I-1)^2\gamma^2 \sigma^2 \right)   \bigg).
\end{align}
Rearranging the above inequality, we have
\begin{align}
f(\bar{\x}^r) - f^* &\le \frac{4}{\gamma\eta I}\left( 1 -\frac{\gamma\eta\mu I}{4}  \right)\normsq{\bar{\x}^r - \x^*} - \frac{4}{\gamma\eta I}\normsq{\bar{\x}^{r+1} - \x^*}+\frac{4\gamma\eta\sigma^2 }{N}  \notag \\
+& \underbrace{\left(32 L_h \frac{L_h}{\mu}  + 12\gamma \eta L_h^2 I  \right)\cdot \left( 12(I-1)^2\gamma^2 \zeta^2+ 4(I-1)\gamma^2 \sigma^2 \right)}_{T_6} \notag \\
+& \underbrace{(4+4\gamma\eta L_g I)\bigg( 40\gamma^2L_g I\frac{L_g\sigma^2}{\mu N}+ 8\frac{L_g^2}{\mu}\cdot 30I\gamma^2L_h^2 \left(12(I-1)^3\gamma^2 \zeta^2+ 4(I-1)^2\gamma^2 \sigma^2 \right)   \bigg)}_{T_7}.
\end{align}
By $\gamma\eta \le \frac{1}{4L_h I}$, we have 
\begin{align}
T_6 \le 45\gamma^2 \frac{L_h^2}{\mu}I^2\zeta^2 + 15\gamma^2\frac{L_h^2}{\mu} I \sigma^2.
\end{align}
By $\gamma\eta \le \frac{1}{16L_g I}$ and $\gamma \le \frac{1}{24L_gI}\sqrt{\frac{L_g}{\mu}}$, we have 
\begin{align}
T_7 \le 80\gamma^2L_g I \frac{L_g\sigma^2}{\mu N} + 18\gamma^2\frac{L_h^2}{\mu}I^2\zeta^2 + 6\gamma^2\frac{L_h^2}{\mu}I\sigma^2.
\end{align}
Substituting $T_6$ and $T_7$ back, we have
\begin{align}
f(\bar{\x}^r) - f^* &\le \frac{4}{\gamma\eta I}\left( 1 -\frac{\gamma\eta\mu I}{4}  \right)\normsq{\bar{\x}^r - \x^*} - \frac{4}{\gamma\eta I}\normsq{\bar{\x}^{r+1} - \x^*}+\frac{4\gamma\eta\sigma^2 }{N} \notag \\
+& 80\gamma^2L_g I \frac{L_g\sigma^2}{\mu N} + 63\gamma^2\frac{L_h^2}{\mu}I^2\zeta^2 + 21\gamma^2\frac{L_h^2}{\mu}I\sigma^2.
\end{align}
By Lemma 1 in [2], using $\frac{1}{\mu R} \le \gamma \eta I \le \frac{1}{16 L_g}$, we obtain
\begin{align}
    &\mathbb{E}[f(\bar{\x}^R)] - f^* \le 4\mu\|\mathbf{x} - \mathbf{x}^* \|^2 \exp(-\frac{\mu\gamma\eta IR}{4}) + \frac{4\gamma\eta\sigma^2 }{N} \notag \\
    &+80\gamma^2 \frac{L_g^2}{\mu} I\frac{\sigma^2}{N} + 63\gamma^2\frac{L_h^2}{\mu}I^2\zeta^2 + 21\gamma^2\frac{L_h^2}{\mu}I\sigma^2.
\end{align}
}

\section{Additional Details and Results of Experiments}
\label{sec:appendix-experiments}
In this section, we provide additional details of our experiments.  More experimental results are provided for full participation with the MNIST dataset, CINIC-10 dataset~\citep{darlow2018cinic10} and CIFAR-100 dataset.

\begin{figure*}[tb]
  \centering
  \begin{subfigure}{0.4\textwidth}
  \centering
  \includegraphics[width=5cm]{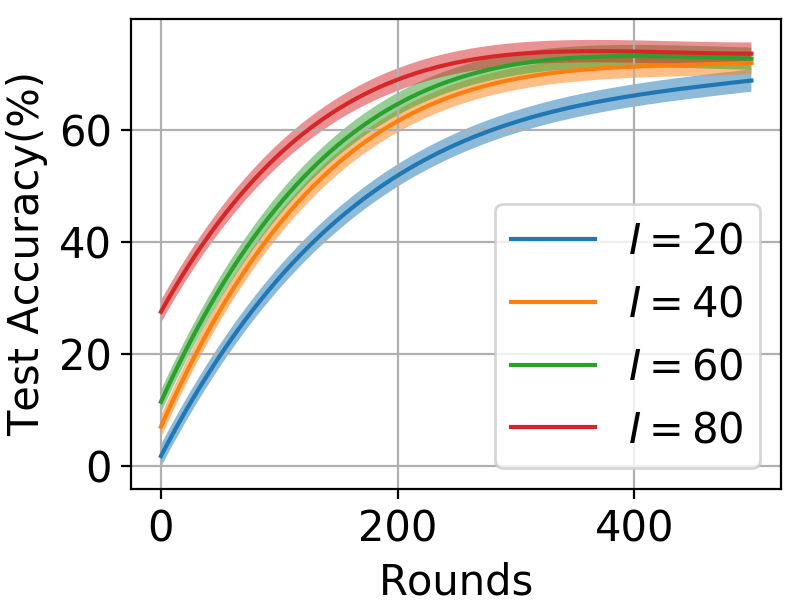}
  \caption{CNN ($50\%$ Non-IID).}
  \label{fig:cifar-cnn-50-test}
  \end{subfigure}
  \begin{subfigure}{0.4\textwidth}
  \centering
  \includegraphics[width=5cm]{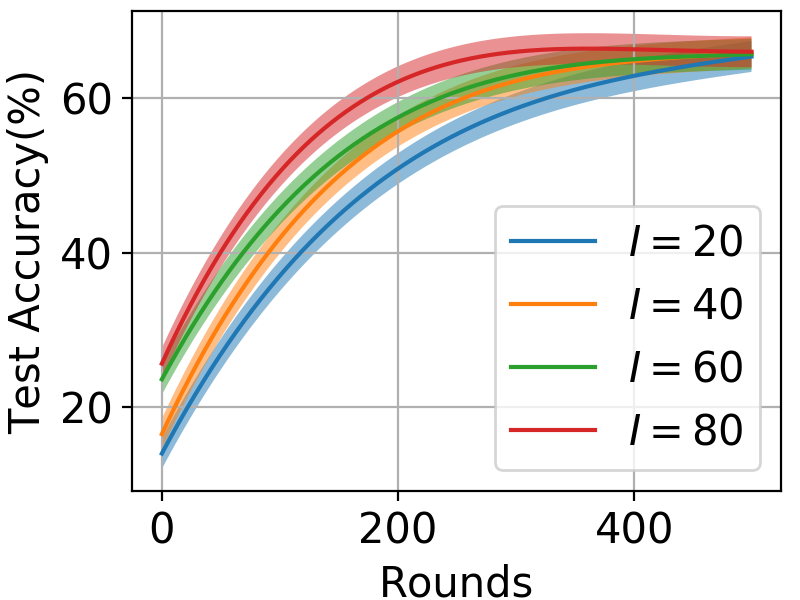}
  \caption{CNN ($75\%$ Non-IID).}
  \label{fig:cifar-cnn-75-test}
  \end{subfigure}
  \begin{subfigure}{0.4\textwidth}
  \centering
  \includegraphics[width=5cm]{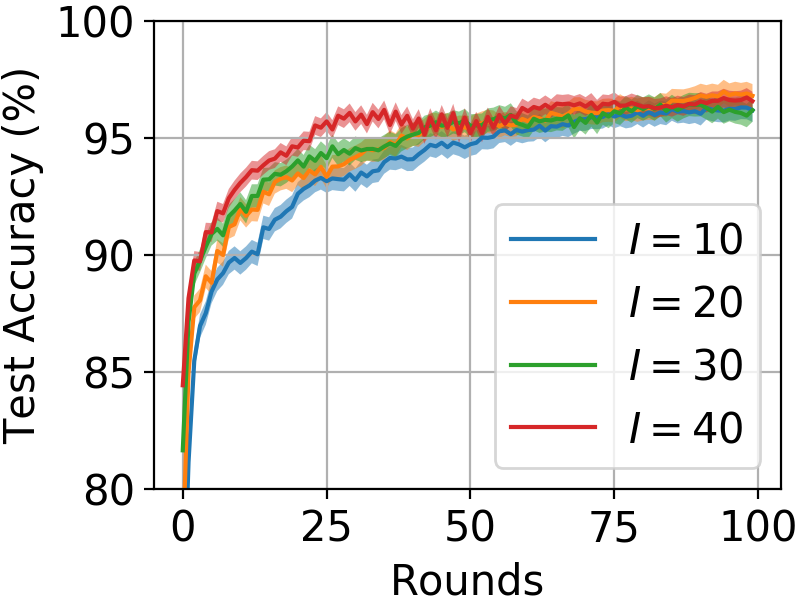}
  \caption{MLP ($50\%$ Non-IID).}
  \label{fig:mnist-mlp-50-test}
  \end{subfigure}
  \begin{subfigure}{0.4\textwidth}
  \centering
  \includegraphics[width=5cm]{figures/mnist-mlp-50-test.png}
  \caption{MLP ($75\%$ Non-IID).}
  \label{fig:mnist-mlp-75-test}
  \end{subfigure}
  \caption{The results of test accuracy for CNN with CIFAR-10 and MLP with MNIST, which are corresponding to the setting in Figure~\ref{fig:cnn-mlp}.}
  \label{fig:test-cnn-mlp}
\end{figure*}

\begin{figure*}[tb]
  \centering
  \begin{subfigure}{0.4\textwidth}
  \centering
  \includegraphics[width=5cm]{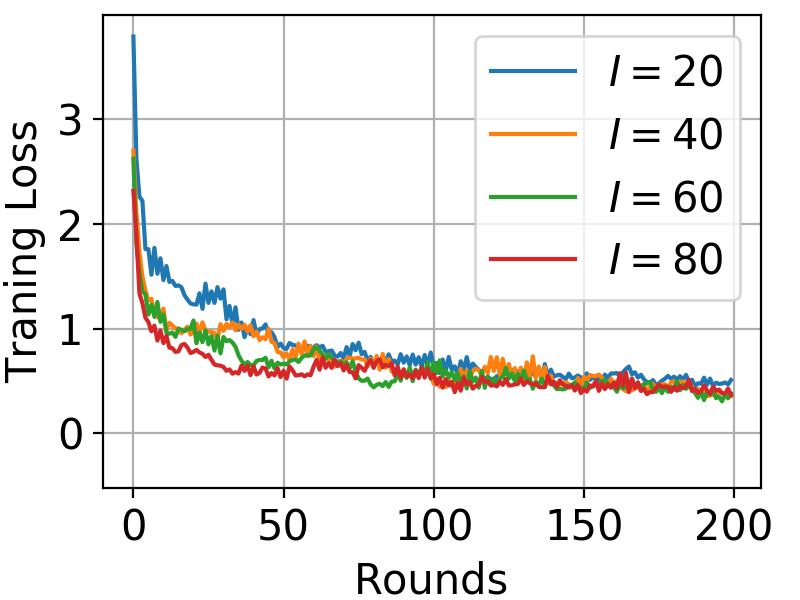}
  \caption{Training Loss.}
  \label{fig:vgg11-50-train}
  \end{subfigure}
  \begin{subfigure}{0.4\textwidth}
  \centering
  \includegraphics[width=5cm]{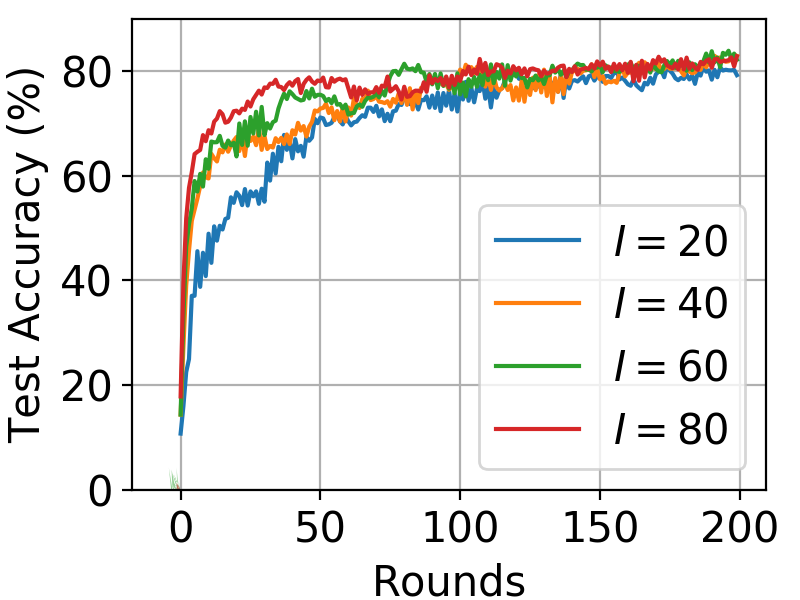}
  \caption{Test Accuracy.}
  \label{fig:vgg11-50-test}
  \end{subfigure}
  \caption{Results with CIFAR-10 dataset. The model is VGG-11. The percentage of heterogeneous data is $50\%$. The learning rates are chosen as $\eta=2$ and $\gamma= 0.01$.}
  \label{fig:vgg11-50}
\end{figure*}

\begin{figure*}[tb]
  \centering
  \begin{subfigure}{0.4\textwidth}
  \centering
  \includegraphics[width=5cm]{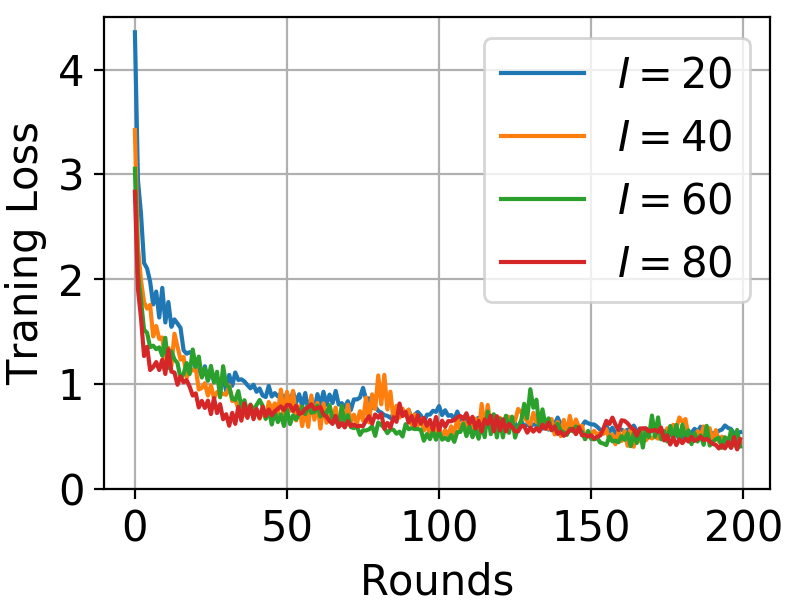}
  \caption{Training Loss.}
  \label{fig:vgg11-75-train}
  \end{subfigure}
  \begin{subfigure}{0.4\textwidth}
  \centering
  \includegraphics[width=5cm]{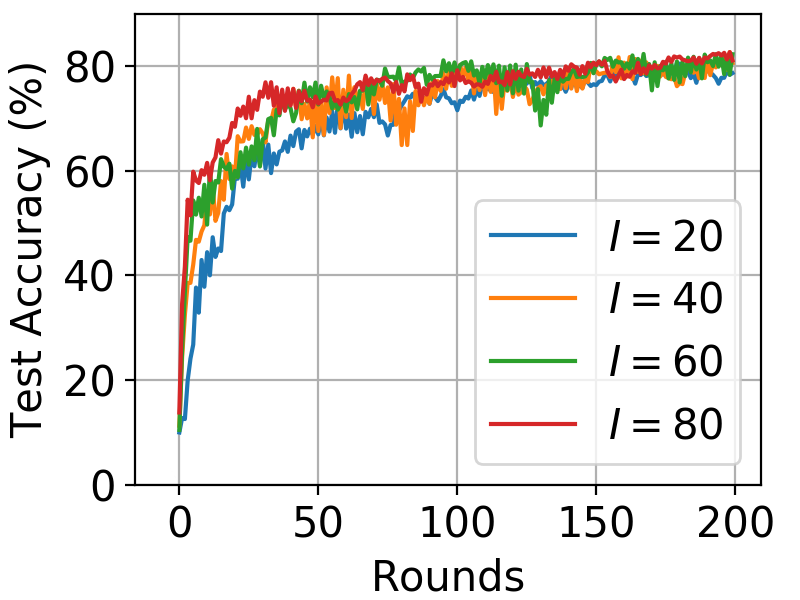}
  \caption{Test Accuracy.}
  \label{fig:vgg11-75-test}
  \end{subfigure}
  \caption{Results with CIFAR-10 dataset. The model is VGG-11. The percentage of heterogeneous data is $75\%$. The learning rates are chosen as $\eta=2$ and $\gamma= 0.01$.}
  \label{fig:vgg11-75}
\end{figure*}

\begin{figure*}[tb]
  \centering
  \begin{subfigure}{0.4\textwidth}
  \centering
  \includegraphics[width=5cm]{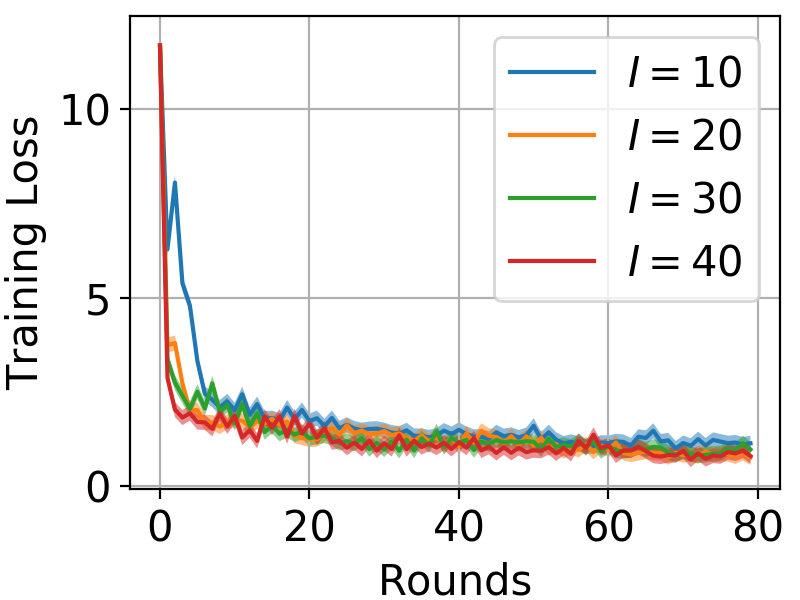}
  \caption{Training Loss.}
  \label{fig:cinic-50-train}
  \end{subfigure}
  \begin{subfigure}{0.4\textwidth}
  \centering
  \includegraphics[width=5cm]{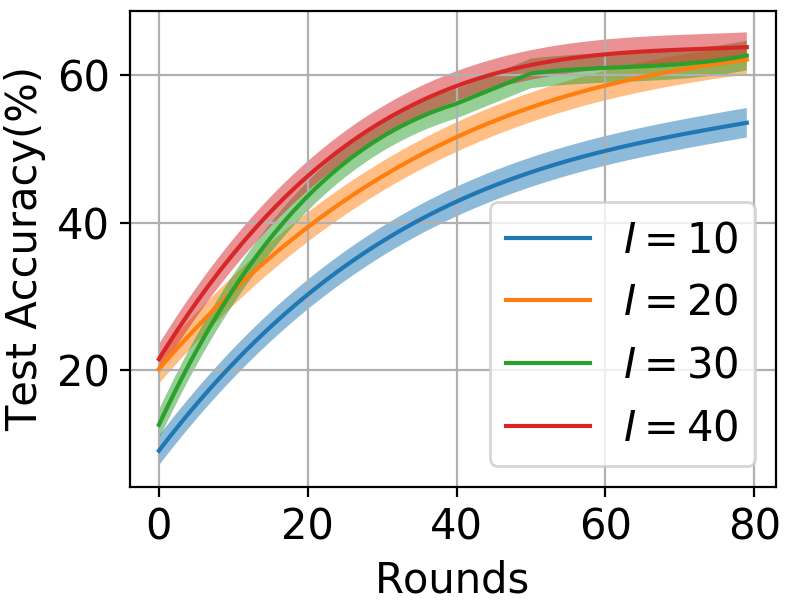}
  \caption{Test Accuracy.}
  \label{fig:cinic-50-test}
  \end{subfigure}
  \caption{Results with CINIC-10 dataset. The model is VGG-16. The percentage of heterogeneous data is $50\%$. The learning rates are chosen as $\eta=2$ and $\gamma= 0.01$.}
  \label{fig:cinic-50}
\end{figure*}
\begin{figure*}[tb]
  \centering
  \begin{subfigure}{0.4\textwidth}
  \centering
  \includegraphics[width=5cm]{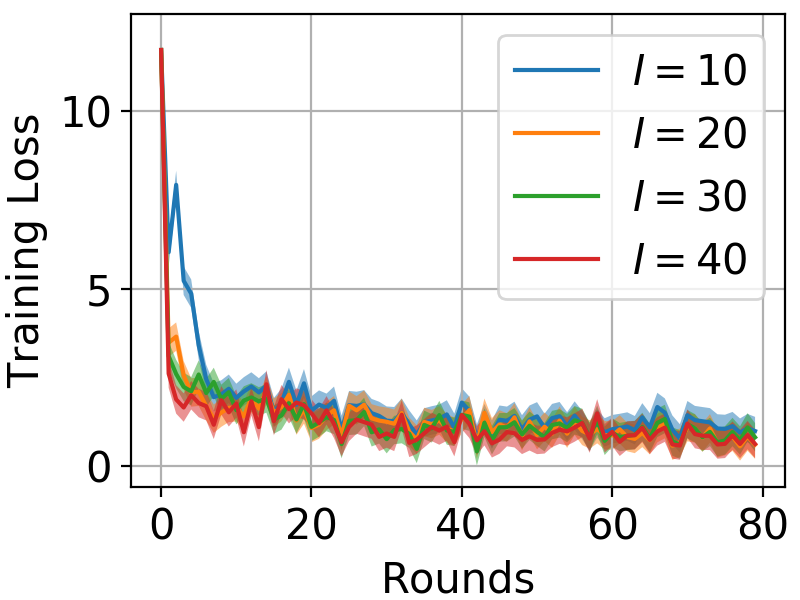}
  \caption{Training Loss.}
  \label{fig:cinic-75-train}
  \end{subfigure}
  \begin{subfigure}{0.4\textwidth}
  \centering
  \includegraphics[width=5cm]{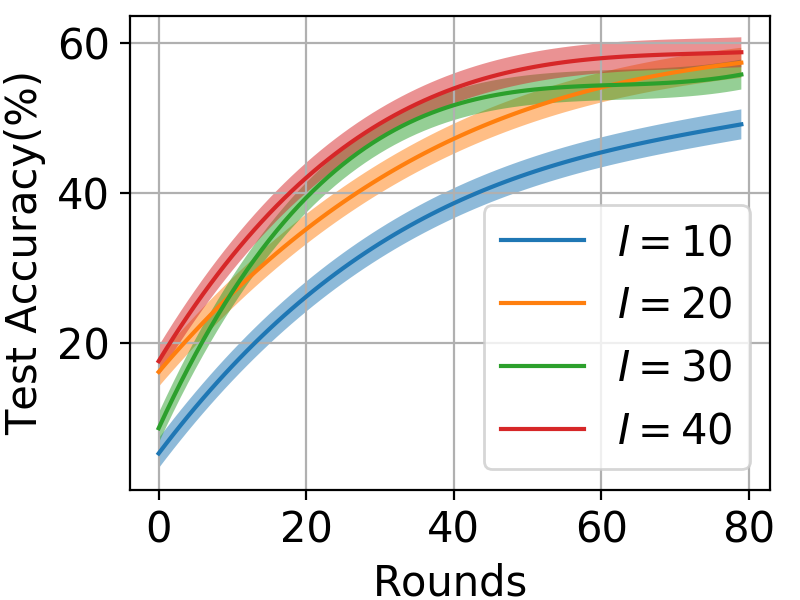}
  \caption{Test Accuracy.}
  \label{fig:cinic-75-test}
  \end{subfigure}
  \caption{ Results with CINIC-10 dataset. The model is VGG-16. The percentage of heterogeneous data is $75\%$. The learning rates are chosen as $\eta=2$ and $\gamma= 0.01$.}
  \label{fig:cinic-75}
\end{figure*}

\begin{figure*}[h!]
  \centering
  \begin{subfigure}{0.4\textwidth}
  \centering
  \includegraphics[width=5cm]{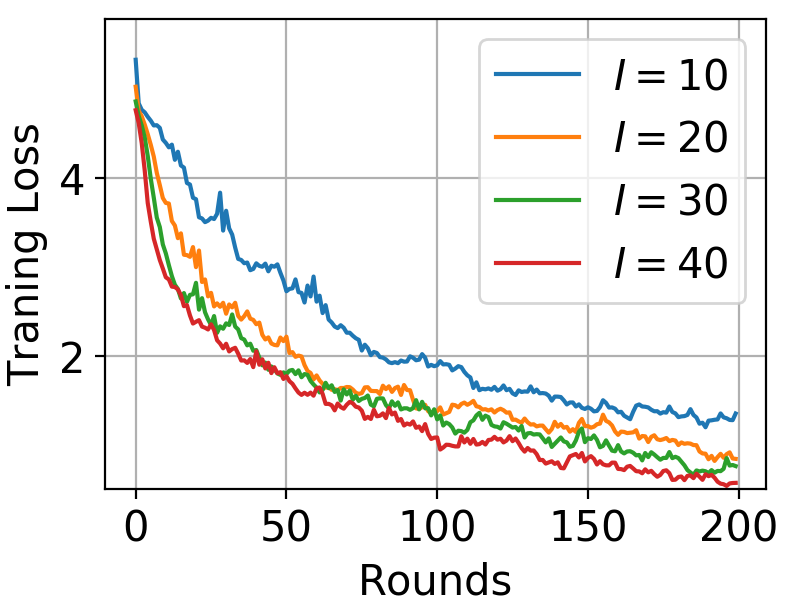}
  \caption{Training Loss.}
  \label{fig:vgg16-50-train}
  \end{subfigure}
  \begin{subfigure}{0.4\textwidth}
  \centering
  \includegraphics[width=5cm]{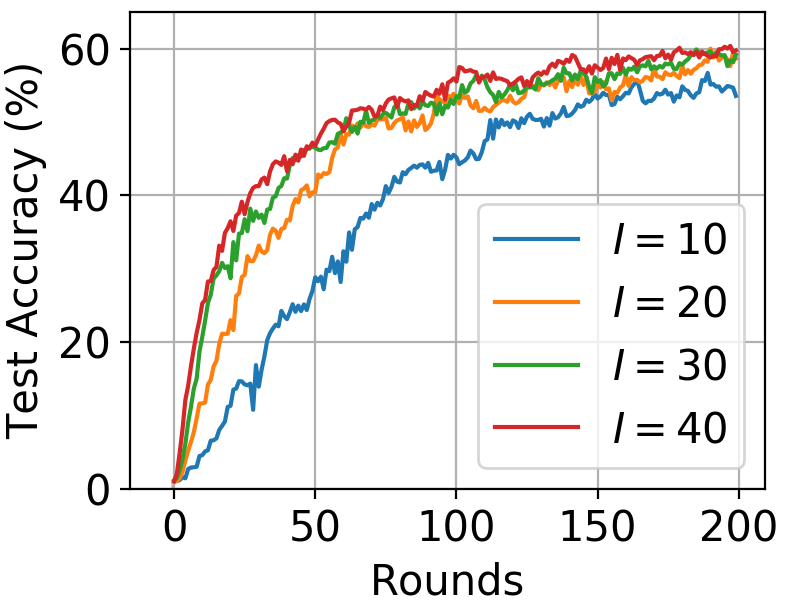}
  \caption{Test Accuracy.}
  \label{fig:vgg16-50-test}
  \end{subfigure}
  \caption{Results with CIFAR-100 dataset. The model is VGG-16. The percentage of heterogeneous data is $50\%$. The learning rates are chosen as $\eta=2$ and $\gamma= 0.02$.}
  \label{fig:vgg16-50}
\end{figure*}
\begin{figure*}[h!]
  \centering
  \begin{subfigure}{0.4\textwidth}
  \centering
  \includegraphics[width=5cm]{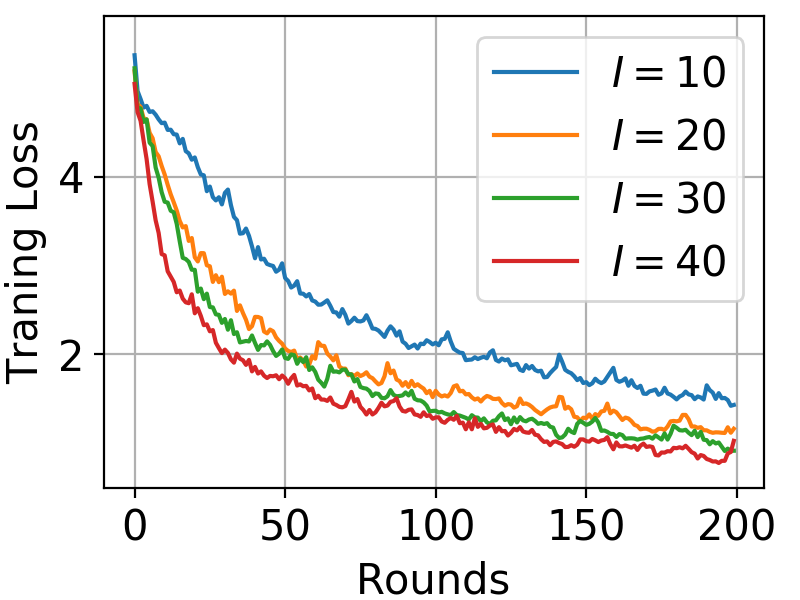}
  \caption{Training Loss.}
  \label{fig:vgg16-75-train}
  \end{subfigure}
  \begin{subfigure}{0.4\textwidth}
  \centering
  \includegraphics[width=5cm]{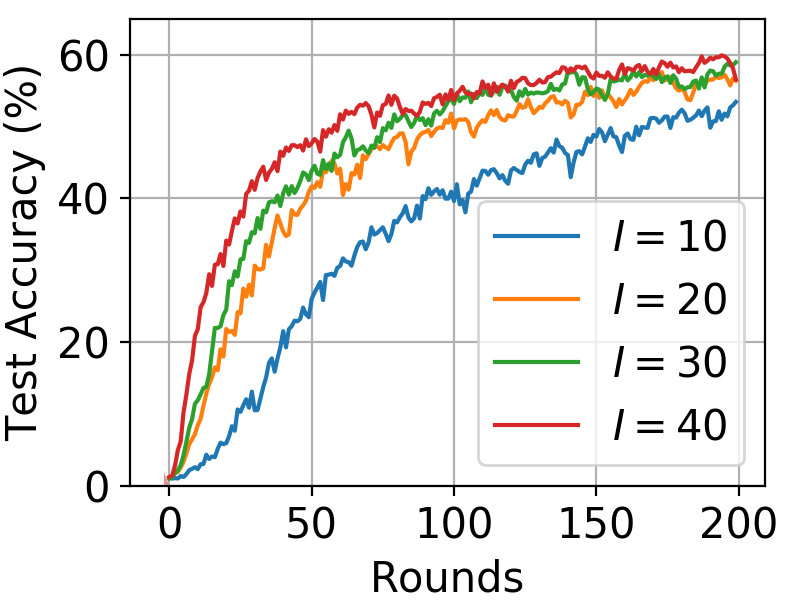}
  \caption{Test Accuracy.}
  \label{fig:vgg16-75-test}
  \end{subfigure}
  \caption{Results with CIFAR-100 dataset. The model is VGG-16. The percentage of heterogeneous data is $75\%$. The learning rates are chosen as $\eta=2$ and $\gamma= 0.02$.}
  \label{fig:vgg16-75}
\end{figure*}

\textbf{Environment.} All of our experiments are implemented in PyTorch and run on a server with four NVIDIA 2080Ti GPUs.  The mini-batch size of SGD for MNIST and CIFAR-10 is $20$. The mini-batch size of SGD for CIFAR-100 is $32$.
We run each experiment 5 times then plot their average and the standard deviation.

\textbf{Model.} For experiments with MLP model, we use a two-layer fully connected neural network, where the width of the networks is $100$.   For experimental results with CIFAR-10 dataset in the main paper, we use a CNN model. The structure of the CNN is $5\times 5\times 32$ Convolutional $\to$ $2 \times 2$ MaxPool $\to$ $5 \times 5 \times 32$ Convolutional  $\to$ $2\times2$ MaxPool $\to$ $4096\times 512$ Dense $\to$ $512 \times 128$ Dense $\to$ $128 \times 10$ Dense $\to$ Softmax. For experimental results with MNIST dataset, we use a two-layer neural network with cross-entropy loss and a linear regression model with MSE loss.  For experimental results with CINIC-10 dataset~\citep{darlow2018cinic10}, we use VGG-16 with the cross-entropy loss.

\textbf{Further explanation of the percentage of heterogeneous data.} For example, the percentage of heterogeneous data is $50\%$ means that $50\%$ of the data on each worker are with the same label, e.g., $50\%$ of the data on worker $1$ are with label $1$. Another $50\%$ of the data are sampled uniformly from the remaining dataset.

\textbf{The estimate of $L_h$.}
Let the global model be $\bar{\x}$ and the local models be $\x_i, i=1,2,\ldots,N$, %
in the beginning of a round, 
then we estimate $L_h$ %
using the following equations. 
\begin{align*}
    &L_h^2 \approx \frac{\normsq{\nabla f(\bar{\x}) - \frac{1}{N}\sum_{i=1}^N \nabla f_i(\x_i)}}{\frac{1}{N}\sum_{i=1}^N \normsq{\x_i- \bar{\x}}}.  
\end{align*}
Starting from a global model that is close to convergence, we perform FedAvg for 10 rounds and estimate $L_h^2$ in each round. 
Then we use the averaged $L_h^2$ over $10$ rounds as the estimate for $L_h^2$. 
The reason for starting from a global model that is close to convergence is that this can make the variance of the estimate smaller. Similarly, the methods of estimating $L_g$ and $\tilde{L}$ are given by
\begin{align*}
&L_g \approx \frac{\left\|\nabla f(\bar{\mathbf{x}})-\nabla f(\bar{\mathbf{y}})\right\|}{\left\|\mathbf{\bar{x}} - \mathbf{\bar{y}}\right\|},\\
&\tilde{L} \approx \max_i \frac{\left\|\nabla F_i(\bar{\mathbf{x}})-\nabla F_i(\mathbf{x}_i)\right\|}{\left\| \bar{\mathbf{x}} - \mathbf{x}_i \right\|}.
\end{align*}

\textbf{Additional Experimental Results.} We partition CIFAR-10, CINIC-10 and CIFAR-100 into $10$ workers. 
During each round, all workers will perform the local updates. 
The results of test accuracy for CNN with CIFAR-10 and MLP with MNIST are provided in Figure~\ref{fig:test-cnn-mlp}. The results of training loss can be found in Figure~\ref{fig:cnn-mlp} of the main paper. 
As shown in Table~\ref{tab:estimate-L} of the main paper, $L_h$ is very small in this case. 
In Corollary~\ref{cor:low-heterogeneity}, with full participation, it is shown that when $L_h$ is small, increasing $I$ can improve the convergence even when data are highly heterogeneous. 
As shown in both Figure~\ref{fig:test-cnn-mlp}, the curve with the largest number of local iterations, converges the fastest and achieves %
the highest accuracy, which validates Theorem~\ref{thm:non-convex} and Corollary~\ref{cor:low-heterogeneity}.

Results for CIFAR-10 with VGG-11 are shown in Figures~\ref{fig:vgg11-50} and \ref{fig:vgg11-75}.
Results for CINIC-10 with VGG-16 are shown in Figures~\ref{fig:cinic-50} and \ref{fig:cinic-75}.
Results for CIFAR-100 with VGG-16 are shown in Figure~\ref{fig:vgg16-50} and \ref{fig:vgg16-75}. 
It can be seen in both results, the curve with the largest number of local iterations converges the fastest and achieves the %
highest accuracy, which is consistent with our theoretical results in Theorem~\ref{thm:non-convex} and Corollary~\ref{cor:low-heterogeneity}.

\end{document}